\documentclass{article}
\usepackage{arxiv}

\usepackage{amsmath,amsfonts,bm}

\def\eqref#1{equation~\ref{#1}}

\def\1{\bm{1}}

\DeclareMathAlphabet{\mathsfit}{\encodingdefault}{\sfdefault}{m}{sl}
\SetMathAlphabet{\mathsfit}{bold}{\encodingdefault}{\sfdefault}{bx}{n}

\usepackage{xcolor}
\definecolor{lightblue}{rgb}{0.6, 0.8, 0.9}
\definecolor{darkblue}{HTML}{0c58ac}
\definecolor{darkgreen}{rgb}{0, 0.55, 0.12}
\definecolor{darkred}{rgb}{0.6,0,0}
\definecolor{yellowbg}{RGB}{255,249,224}
\definecolor{greenbg}{RGB}{232,246,232}
\definecolor{bluebg}{RGB}{226,241,255}
\usepackage{natbib}
\usepackage[
  breaklinks=true,
  colorlinks,
  linkcolor=darkred,
  citecolor=darkgreen,
  bookmarks=false
]{hyperref}
\usepackage{url}

\usepackage{tikz}
\usepackage{graphicx}
\usepackage{wrapfig}
\usepackage{booktabs}
\usepackage{adjustbox}
\usepackage{multirow}
\usepackage{xspace}
\usepackage{array}
\usepackage{bm}
\usepackage{bbm}
\usepackage{xcolor}
\usepackage{amsthm, amsmath, amssymb}
\usepackage{cleveref}
\usepackage{algorithm}
\usepackage{algorithmic}
\newtheorem{definition}{Definition}
\newtheorem{theorem}{Theorem}
\newtheorem{remark}{Remark}

\usepackage{listings}
\usepackage{wrapfig}
\usepackage{caption}
\usepackage{subcaption}
\usepackage{colortbl}
\usepackage{adjustbox}
\usepackage{makecell}
\usepackage{pifont}
\usepackage{tcolorbox}
\usepackage{setspace}
\usepackage{fancybox}
\usepackage{tocloft}
\usepackage{etoc}
\usepackage{enumitem}
\usepackage{pifont}
\usepackage{bm}
\definecolor{darksilver}{RGB}{160,160,160} 
\definecolor{mymauve}{rgb}{0.58,0,0.82}
\newlength{\mysize}

\allowdisplaybreaks

\theoremstyle{definition}

\title{Benign Overfitting under Heterogeneous Input Fusion}

\author{Houzhen Liu \qquad Xiaobo Xia \\
University of Science and Technology of China \\
\tt\small liuhz0919@mail.ustc.edu.cn \qquad xiaoboxia@ustc.edu.cn 
}

\begin{document}

\newcommand{\dataset}{{\cal D}}
\newcommand{\fracpartial}[2]{\frac{\partial #1}{\partial  #2}}
\newcommand{\Breg}{\mathsf{B}^{\mathrm{reg}}}
\newcommand{\Nreg}{\mathsf{N}^{\mathrm{reg}}}
\newcommand{\Bcls}{\mathsf{B}^{\mathrm{cls}}}
\newcommand{\Ncls}{\mathsf{N}^{\mathrm{cls}}}
\newcommand{\BFS}{\mathsf{B}^{\mathrm{FS}}}
\newcommand{\NFS}{\mathsf{N}^{\mathrm{FS}}}
\newtheorem{example}{Example} 
\newtheorem{lemma}[theorem]{Lemma} 
\newtheorem{proposition}[theorem]{Proposition} 
\newtheorem{corollary}[theorem]{Corollary}
\newtheorem{conjecture}[theorem]{Conjecture}
\newtheorem{axiom}[theorem]{Axiom}

\maketitle

\begin{abstract}
Benign overfitting is extensively studied when learning from a single
high-dimensional input, but its behavior under heterogeneous input fusion
remains largely unexplored. We study this question for minimum-norm linear interpolation under a
heterogeneous Gaussian design, comparing two statistically dependent input
blocks with their fusion while holding the underlying population task fixed. For regression, we identify a full-spectrum covariance certificate whose asymptotic status is independent of the cutoff threshold and prove that it is preserved by every positive-semidefinite joint covariance consistent with the two marginals. This protection is sharp, yet it does
not extend to all benign regression problems: outside the certified regime,
two benign marginals can have a harmful fusion. For one-sparse Gaussian
classification, benignity in the regular regime is characterized by the
balance between surviving predictive signal and nuisance contamination.
Fusion can move these two quantities in opposite directions, and within this
model class every marginal-to-joint benign/non-benign pattern is attainable.
We further show that the same fused input can have qualitatively
different effects on regression and classification. These results establish
that benign overfitting under heterogeneous fusion is determined by the joint
signal and spectral geometry created by input interaction, rather than by
marginal benignity alone.
\end{abstract}
\addtocontents{toc}{\protect\setcounter{tocdepth}{-1}}

\section{Introduction}
\label{sec:introduction}

Interpolation is no longer regarded as inherently incompatible with
generalization. Modern overparameterized models can fit every training
observation, including noise, and nevertheless achieve small population error~\citep{zhang2017understanding,belkin2019reconciling}. In high-dimensional linear regression, this phenomenon is commonly known as
benign overfitting~\citep{bartlett2020benign} and has been characterized through the geometry of
the input distribution. Whether the minimum-norm interpolator generalizes
depends not simply on the number of parameters, but on the covariance
spectrum, its effective dimension, and the placement of predictive signal
relative to that spectrum
\citep{bartlett2020benign,muthukumar2020harmless,hastie2022surprises,
tsigler2023benign}. Much of this theory, however, starts from a single input covariance and asks whether interpolation is benign, leaving open what happens when multiple sources of information are combined.

Many learning problems do not come with a single homogeneous input. Different sensors, views, measurements, or feature groups may describe the same observation, often with substantial statistical dependence between them~\citep{andrew2013deep,michaeli2016nonparametric,zhou2023multi,damianou2020multiview}. Learning from multiple sources and views provides prominent examples~\citep{xu2023multimodal,zhao2017multi}, but the issue is more general: one may acquire an additional measurement, append a new group of covariates, or combine two already informative input blocks~\citep{wang2015deepmultiview,baltrusaitis2019multimodal}. In all these cases, a natural operation is to fuse the available inputs. If either input is used alone, existing benign-overfitting theory may characterize whether interpolation generalizes. What is much less clear is whether these marginal conclusions survive fusion. Two individually benign inputs need not obviously remain benign when combined, while two individually unfavorable inputs may conceivably produce a favorable fused input. The central issue is therefore not the marginal inputs themselves, but how their interaction reshapes the behavior of the interpolating estimator.

Understanding this interaction requires more than treating fusion as an increase in ambient dimension. Statistical dependence between input blocks alters the spectrum of the joint covariance and, at the same time, changes how the predictive signal is arranged relative to that spectrum. Consequently, properties that appear harmless from either marginal input can become important after concatenation, while unfavorable marginal structures may offset one another in the joint problem. A theory of benign overfitting under heterogeneous fusion must therefore separate what is inherited from the marginal inputs from what emerges only through their dependence. This leads to the central question of this work: \textit{when can combining multiple sources preserve, destroy, or even create benign overfitting}?

We study this question in an overparameterized linear model with heterogeneous Gaussian inputs. The first input block is Gaussian, while the second contains a component linearly inherited from the first together with an independent innovation. This construction spans independent, redundant, and partially dependent inputs, while making the interaction between the two blocks explicit. We compare learning from either marginal input with learning from their concatenation. Importantly, these are not three separately designed prediction problems. Both regression responses and binary labels are generated from the same latent Gaussian score, so fusion changes only the information available to the learner, not the underlying population task. We consider squared-loss regression with the minimum-Euclidean-norm linear interpolator and binary classification obtained by applying the same minimum-norm least-squares rule to binary labels. This common construction allows us to isolate how input fusion itself changes the conditions under which benign overfitting occurs.

The resulting picture is sharply task-dependent. Specifically, for regression, we identify a protected spectral regime in which benign overfitting is stable under fusion. The regime is characterized by full-spectrum effective-rank conditions, whose asymptotic validity is independent of the fixed threshold used to define the spectral cutoff. More importantly, this regime is closed under fusion at the covariance level: if both marginal covariances satisfy the certificate, then every positive-semidefinite joint covariance consistent with those marginals satisfies it as well. This closure is only a sufficient guarantee, however, not a complete characterization of signal-specific benign regression. Outside the protected regime, benignity can fail to compose. Two marginal problems may both achieve vanishing excess risk while their fusion retains non-vanishing error, and even asymptotically negligible innovation energy can change the outcome when it is distributed unfavorably across the spectrum. Therefore, regression admits a robust spectral law for fusion, but not an unconditional law of marginal benignity. Classification follows a fundamentally different principle. After separating
the predictive direction from the nuisance spectrum, benignity is governed by
a balance between signal survival under minimum-norm interpolation and
nuisance contamination. In a regular one-sparse regime, this yields a sharp
criterion based on the relation between signal survival and contamination.
Fusion can alter these two terms in different directions. It may reduce
contamination while weakening the surviving signal, or strengthen the signal
while increasing contamination. As a result, no analogue of the regression
closure law exists. Indeed, every marginal-to-joint benign/non-benign pattern
can occur. This contrast arises under the same heterogeneous inputs and the
same latent score, showing that the effect of fusion depends fundamentally on
the prediction task rather than on the input distribution alone.

Together, these results show that benign overfitting is not generally compositional across information sources. What determines the effect of fusion is the aspect of joint geometry that governs generalization: the full covariance spectrum in the protected regression regime, and the balance between surviving signal and nuisance contamination in classification. This perspective explains how statistically dependent information can preserve, destroy, or create benign overfitting.

\subsection{Related Work}
\label{sec:related-work}

\paragraph{Benign overfitting in regression.}
Benign overfitting has been studied most extensively in overparameterized
linear regression. Specifically, \citet{bartlett2020benign} characterize when the
minimum-Euclidean norm interpolator can generalize despite exactly fitting
noisy responses, with the relevant conditions expressed through effective
ranks of the population covariance. Related work has clarified how high dimensionality can absorb and distribute
response noise across many directions~\citep{muthukumar2020harmless,chatterji2022foolish}, derived precise
high-dimensional characterizations of ridgeless least squares and its
statistical behavior~\citep{hastie2022surprises,han2026distribution}, and examined how covariance
structure and signal alignment shape transitions between benign and harmful
interpolation~\citep{li2026riskphase}. Analogous phenomena have also been
established for kernel interpolation~\citep{liang2020just}. Further studies have examined norm-based explanations of interpolation~\citep{koehler2021uniform,bartlett2021failures}, robustness of minimum-norm
interpolators~\citep{chinot2022robustness}, and the role of algorithmic and
optimization-induced implicit bias~\citep{donhauser2022fast,chatterji2022interplay,zou2023sgd}.
Benign overfitting has also been studied for regularized estimators such as
ridge regression~\citep{tsigler2023benign}.

Previous work primarily asks whether interpolation is benign for a given
input distribution. Instead, our focus is on what happens when two
statistically dependent input blocks are combined. Importantly, this problem requires
relating the spectral geometry of two marginal inputs to that of their joint
input, rather than analyzing each covariance in isolation. A closely related
setting is considered by \citet{shamir2022implicit}, who studies benign
overfitting when a low-dimensional component is concatenated with a
high-dimensional component. We consider a different regime in which both
input blocks may be high dimensional and statistically dependent, and the
same population task is studied using either input separately and using both
jointly.

\paragraph{Benign overfitting in classification.}
A parallel literature has shown that interpolation can behave differently in
classification. For instance, \citet{muthukumar2021classification} compare regression and
classification in overparameterized Gaussian models and identify regimes in
which the least-squares minimum-norm solution agrees with the hard-margin
classifier. \citet{hsu2021proliferation} study the proliferation of support
vectors in high dimensions, while \citet{chatterji2021finite} establish
finite-sample risk guarantees for overparameterized maximum-margin
classification in the presence of label noise. Benign classification has
also been analyzed under Gaussian and sub-Gaussian mixture models~\citep{cao2021risk,wang2022binary}, in multiclass settings~\citep{wang2021multiclass}, and for nonlinear classifiers in
overparameterized regimes~\citep{frei2022benign,frei2023kkt}.

These prior works largely characterize when a single interpolating classifier
generalizes. We instead ask how classification behavior changes when dependent input blocks are observed separately or jointly. Our classifier is
the least-squares minimum-norm interpolator applied directly to binary labels,
and the marginal and fused problems share the same latent label mechanism.
This setup isolates the effect of input fusion without changing either the
learning rule or the underlying prediction task.

\paragraph{Learning with multiple input sources.}
Learning from multiple statistically dependent sources has been studied
extensively in multi-source and multi-view learning
\citep{wang2015deepmultiview,baltrusaitis2019multimodal}. A recurring theme is
that different inputs may contain both shared and source-specific
information. For example, \citet{shu2022dgcca} study the decomposition of
high-dimensional multi-view data into common and distinctive components. \citet{tosh2021contrastive} relate redundancy across views to downstream
prediction in contrastive learning, while
\citet{zhang2023multimodal} analyze generalization in multi-source contrastive
learning. Existing work also shows that incorporating additional sources need
not uniformly improve learning. \citet{du2023unimodal} identify settings in
which joint multi-source training can impair the learning of useful
single-source features, and \citet{he2024modality} study the utility and
selection of sources when the available inputs are statistically dependent.

Our setting is complementary to this literature. We do not learn the input
features, design a fusion architecture, or select among available sources.
Instead, the input blocks and their joint distribution are given, the learning
rule is fixed, and we ask how observing them jointly changes the
generalization behavior of an interpolating predictor. This question is also
distinct from covariate-shift analyses of minimum-norm interpolation
\citep{mallinar2024minimum}, where the same input variables are observed
under different training and test distributions. Here the population task is
fixed, and what changes is whether the learner observes either input block
alone or both together.

\subsection{Overview of Results}
\label{sec:overview}

Our results identify different composition principles for regression and
classification under heterogeneous input fusion, which are summarized below.

\paragraph{A protected spectral regime for regression.}
We first identify a full-spectrum covariance certificate that captures the effective-rank requirements underlying benign minimum-norm regression. Its asymptotic status is invariant to the fixed threshold used to define the spectral cutoff, which turns it into an intrinsic property of the covariance sequence. Our main regression result shows that this property is closed under fusion. If both marginal covariances satisfy the condition, then every
positive-semidefinite joint covariance with those marginals satisfies it as well. Moreover, this is the only universal restriction at the covariance level. Every other marginal-to-joint spectral pattern is attainable. The scope of this closure is essential. The full-spectrum condition is a sufficient certificate for benign regression, not a characterization of every benign task. Outside the certified regime, marginal benignity need not compose. We construct two marginal regression problems with vanishing excess risk whose fusion retains non-vanishing error. We also show that an innovation with vanishing total energy can change the outcome when its energy is spread across an unfavorable part of the spectrum. Thus regression admits a genuine
spectral protection law, but no unconditional closure law at the level of task-specific benignity.

\paragraph{Signal and nuisance in classification.}
Classification is controlled by a different geometry. After separating the
predictive direction from the nuisance spectrum, we characterize benignity by
the nuisance contamination relative to the amount of predictive signal that
survives minimum-norm interpolation. The resulting condition is sufficient in
general and necessary in the regular one-sparse regime. This characterization also determines how fusion acts. The change from a
marginal input to the joint input separates into two effects, one on nuisance
contamination and one on signal survival. Neither is monotone. Consequently,
marginal benignity alone places no nontrivial restriction on the fused
classifier. All six unordered marginal-to-joint benign/non-benign patterns
are attainable, including both the failure of two benign marginals after
fusion and the recovery of two non-benign marginals through fusion.

\paragraph{Input fusion is task dependent.}
The difference between these two laws persists even when the inputs and the
underlying signal are held fixed. Using a common latent score for regression
and classification, we construct one setting in which both marginal inputs
are non-benign for both tasks, while fusion makes classification benign but
leaves regression non-benign. We construct another in which both marginals
are benign for both tasks, while fusion destroys benign regression and
preserves benign classification. Hence the effect of fusion is not a property
of the input pair alone. It depends on which aspect of the joint geometry
governs generalization for the task.

\subsection{Organization}
The remainder of the paper is organized as follows. Section~\ref{sec:preliminaries} introduces the heterogeneous Gaussian data model, formulates the regression and classification problems under a common latent score, and defines the corresponding minimum-norm interpolators and notions of benign overfitting. Section~\ref{sec:heterogeneous-regression} develops the regression theory, beginning with the full-spectrum covariance certificate and then studying its invariance, fusion properties, and limitations at the task level. Section~\ref{sec:classification-results} turns to classification, where the analysis is organized around signal survival and nuisance contamination and culminates in the complete interaction pattern under fusion. Section~\ref{sec:heterogeneous-separation} directly compares the two tasks under the same input distribution and latent score, establishing that fusion can affect regression and classification differently. Section~\ref{sec:experiments} presents numerical experiments that illustrate the predicted interaction patterns and the mechanisms behind the main constructions. Section~\ref{sec:conclusion} concludes this paper. For better readability, proofs of theoretical results are collected in the Appendix. 

\section{Preliminaries}
\label{sec:preliminaries}
This section introduces the probabilistic framework and learning setup used throughout the paper. We first specify the heterogeneous Gaussian model and the two marginal inputs together with their fusion, and then define the regression and classification tasks through a common latent score. We finally introduce the corresponding minimum-norm interpolators and the asymptotic notions of benign and harmful overfitting.
\subsection{Notation}
\label{subsec:notation}
For a positive integer $m$, let $[m]:=\{1,\ldots,m\}$.  Vectors are column
vectors.  For $v\in\mathbb R^d$, $\|v\|_2$ denotes its Euclidean norm.  For a
matrix $M$, we write $M^\top$ for its transpose, $M^\dagger$ for its
Moore-Penrose pseudoinverse, $\operatorname{tr}(M)$ for its trace,
$\|M\|_{\mathrm{op}}$ for its operator norm, and $\|M\|_{\mathrm F}$ for its
Frobenius norm.  The notation $M\succeq0$ means that $M$ is symmetric positive
semidefinite, and $M\preceq N$ means that $N-M\succeq0$.  We use $I_m$ for the
$m\times m$ identity matrix and $e_j$ for the $j$th standard basis vector.  For
symmetric matrices $M_1$ and $M_2$, $M_1\oplus M_2$ denotes their block direct
sum.  Unless otherwise stated, $\operatorname{sign}(u)=1$ for $u\geq0$ and
$-1$ otherwise; its value at zero is immaterial under the continuous models
considered here.

For nonnegative deterministic sequences $a_n$ and $b_n$, we write
$a_n\lesssim b_n$ if $a_n\leq Cb_n$ for a constant $C$ independent of $n$,
and $a_n\asymp b_n$ if both $a_n\lesssim b_n$ and $b_n\lesssim a_n$.
For arbitrary deterministic sequences $x_n$ and positive sequences $a_n$,
we write $x_n=O(a_n)$ if $|x_n|\leq Ca_n$ eventually for some constant
$C<\infty$, and $x_n=o(a_n)$ if $|x_n|/a_n\to0$.  In particular,
$O(1)$ denotes a bounded deterministic sequence and $o(1)$ a deterministic
sequence converging to zero.  For random sequences,
$Z_n=O_{\mathbb P}(a_n)$ means that $Z_n/a_n$ is bounded in probability,
whereas $Z_n=o_{\mathbb P}(a_n)$ means that $Z_n/a_n\to0$ in probability.
We write $Z_n\asymp_{\mathbb P}a_n$ if there exist
constants $0<c<C<\infty$ such that
\begin{equation}
  \mathbb P\!\left(ca_n\leq Z_n\leq Ca_n\right)\longrightarrow1.
  \label{eq:probabilistic-order-definition}
\end{equation}
All limits are taken as the sample size $n\to\infty$.  Dimensions,
covariances, linear maps, and signal vectors may depend on $n$. This dependence
is suppressed when no ambiguity can arise.

Let $\Sigma\succeq0$ have eigenvalues
$\lambda_1(\Sigma)\geq\lambda_2(\Sigma)\geq\cdots\geq0$.  For any admissible
$k\geq0$, define the tail mass and the two tail effective ranks by
\begin{equation}
  S_k(\Sigma):=\sum_{j>k}\lambda_j(\Sigma),\qquad
  r_k(\Sigma):=\frac{S_k(\Sigma)}{\lambda_{k+1}(\Sigma)},\qquad
  R_k(\Sigma):=\frac{S_k(\Sigma)^2}
  {\sum_{j>k}\lambda_j(\Sigma)^2}.
  \label{eq:spectral-effective-ranks}
\end{equation}
We adopt the usual conventions when the relevant denominator vanishes.  Given
a fixed threshold $b>0$, the cutoff in \citet{bartlett2020benign} is
\begin{equation}
  k_{\mathrm{Bart}}^\star(\Sigma;b)
  :=\min\{k\geq0:r_k(\Sigma)\geq bn\},
  \label{eq:full-spectrum-bartlett-cutoff}
\end{equation}
with value $\infty$ if the set is empty.  This cutoff is always computed from
the \emph{full} covariance spectrum and is used for regression.

For classification, when the signal is an eigendirection $e_t$ of $\Sigma$,
we remove that direction and order the remaining nuisance eigenvalues as
$\mu_1\geq\mu_2\geq\cdots$.  The quantities $S_k^{(-t)},r_k^{(-t)}$, and
$R_k^{(-t)}$ are defined from $\{\mu_j\}$ by the same formulas as in
(\ref{eq:spectral-effective-ranks}), and
\begin{equation}
  k_{-t}^\star(b)
  :=\min\{k\geq0:r_k^{(-t)}\geq bn\},
  \label{eq:nuisance-spectrum-cutoff}
\end{equation}
with value $\infty$ if the set is empty.
Thus $k_{\mathrm{Bart}}^\star$ and $k_{-t}^\star$ are different objects: the
former uses the full spectrum, whereas the latter first deletes the true
classification signal direction.

\subsection{Heterogeneous Gaussian input model}
\label{subsec:heterogeneous-data-model}

We here consider two statistically dependent input blocks generated from a shared Gaussian component and an independent innovation. Let
\begin{equation}
  X_1\sim \mathcal{N}(0,D),\qquad
  U\sim \mathcal{N}(0,\Sigma_U),\qquad U\perp X_1,
  \qquad X_2=AX_1+U,
  \label{eq:heterogeneous-generation-model}
\end{equation}
where $X_1\in\mathbb R^{d_1}$, $X_2,U\in\mathbb R^{d_2}$,
$A\in\mathbb R^{d_2\times d_1}$, $D\succeq0$, and $\Sigma_U\succeq0$.
The term $AX_1$ captures the component of the second input inherited from the first, whereas $U$ is the innovation specific to the second input. This parametrization includes independent inputs ($A=0$), noiseless redundancy
($\Sigma_U=0$), and partially dependent inputs as special cases.

We compare learning from the following three inputs:
\begin{equation}
  X^{(1)}:=X_1,\qquad
  X^{(2)}:=X_2,\qquad
  X^{(J)}:=\begin{pmatrix}X_1\\X_2\end{pmatrix}.
  \label{eq:three-inputs}
\end{equation}
Their covariance matrices are
\begin{align}
  \Sigma_1:=D, \qquad \Sigma_2:=ADA^\top+\Sigma_U,\qquad
  \Sigma_J:=
  \begin{pmatrix}
    D&DA^\top\\
    AD&ADA^\top+\Sigma_U
  \end{pmatrix}\succeq0.
  \label{eq:three-input-covariances}
\end{align}
Here we use $r\in\{1,2,J\}$ as a generic input index and write $d_r$
and $\Sigma_r$ for the corresponding dimension and covariance.  The block
positive-semidefiniteness of $\Sigma_J$ is the only restriction placed on the
cross-modal correlation in our structural fusion results.

\subsection{Regression and classification from a common latent score}
\label{subsec:tasks}

Fix vectors $\beta\in\mathbb R^{d_1}$ and
$\gamma\in\mathbb R^{d_2}$, and let $\xi\sim \mathcal{N}(0,\sigma^2)$ be independent of
$(X_1,U)$.  Both tasks are generated from the common latent Gaussian score
\begin{equation}
  L:=\beta^\top X_1+\gamma^\top U+\xi.
  \label{eq:common-latent-score}
\end{equation}
The regression response and binary classification label are, respectively,
\begin{equation}
  Y^{\mathrm{reg}}:=L,
  \qquad
  Y^{\mathrm{cls}}:=\operatorname{sign}(L)\in\{-1,1\}.
  \label{eq:regression-classification-responses}
\end{equation}
This common-score construction ensures that fusion changes only the observed input, not the underlying population task. Because all variables are jointly Gaussian, for every
$r\in\{1,2,J\}$ the latent score admits an input-specific linear decomposition
\begin{equation}
  L=b_r^\top X^{(r)}+\varepsilon_r,\qquad
  \varepsilon_r\sim \mathcal{N}(0,\tau_r^2),
  \label{eq:input-specific-reduction}
\end{equation}
where $\varepsilon_r$ is centered Gaussian and independent of $X^{(r)}$. Furthermore, let
\(\widetilde b_J=((\beta-A^\top\gamma)^\top,\gamma^\top)^\top\).  The
minimum-Euclidean-norm representatives, valid whether or not the covariances
are singular, are
\begin{align}
  b_1&=D^\dagger D\beta,
  &\tau_1^2&=\gamma^\top\Sigma_U\gamma+\sigma^2,
  \label{eq:source-one-effective-task}\\
  b_2&=\Sigma_2^\dagger(AD\beta+\Sigma_U\gamma),
  &\tau_2^2&=\beta^\top D\beta+\gamma^\top\Sigma_U\gamma+\sigma^2
  -(AD\beta+\Sigma_U\gamma)^\top
  \Sigma_2^\dagger(AD\beta+\Sigma_U\gamma),
  \label{eq:source-two-effective-task}\\
  b_J&=\Sigma_J^\dagger\Sigma_J\widetilde b_J,
  &\tau_J^2&=\sigma^2.
  \label{eq:joint-effective-task}
\end{align}
Here \(M^\dagger\) denotes the Moore--Penrose pseudoinverse. Population
coefficients are identified only modulo the null space of \(\Sigma_r\), and
\(\Sigma_r^\dagger\Sigma_r\) is the orthogonal projector onto
\(\operatorname{range}(\Sigma_r)\).  In particular,
\(D^\dagger D\beta\) and \(\beta\) produce the same score almost surely. The
unprojected vector \(\widetilde b_J\) is always a valid fused coefficient by
the identity below, and its projection in
(\ref{eq:joint-effective-task}) produces the same score almost surely.
(\ref{eq:joint-effective-task}) follows from the identity
\begin{equation}
  \beta^\top X_1+\gamma^\top U
  =(\beta-A^\top\gamma)^\top X_1+\gamma^\top X_2.
  \label{eq:joint-score-reparameterization}
\end{equation}
In particular, both \(\widetilde b_J\) and its uniquely determined
minimum-norm representative \(b_J\) are induced by the original pair
\((\beta,\gamma)\); no coefficient is reselected after observing the fused
covariance.  When the relevant covariance is invertible, these formulas reduce
to \(b_1=\beta\), the usual inverse formula for \(b_2\), and
\(b_J=\widetilde b_J\).

For regression, the population squared-loss risk of $w\in\mathbb R^{d_r}$ is
\begin{equation}
  \mathcal R_r^{\mathrm{reg}}(w)
  :=\mathbb E\bigl[(Y^{\mathrm{reg}}-w^\top X^{(r)})^2\bigr].
  \label{eq:regression-risk}
\end{equation}
The input-specific population oracle is $b_r$. Its Bayes risk is
$\mathcal R_r^{\mathrm{reg},\star}=\tau_r^2$, and
\begin{equation}
  \mathcal E_r^{\mathrm{reg}}(w)
  :=\mathcal R_r^{\mathrm{reg}}(w)-
  \mathcal R_r^{\mathrm{reg},\star}
  =(w-b_r)^\top\Sigma_r(w-b_r).
  \label{eq:regression-excess-risk}
\end{equation}

For classification, the linear classifier induced by $w$ is
$h_w(x):=\operatorname{sign}(w^\top x)$ and its $0$--$1$ risk is
\begin{equation}
  \mathcal R_r^{\mathrm{cls}}(w)
  :=\mathbb P\!\left(h_w(X^{(r)})\neq Y^{\mathrm{cls}}\right).
  \label{eq:classification-risk}
\end{equation}
Whenever $b_r^\top\Sigma_rb_r>0$, Gaussian conditioning shows that the Bayes
classifier based on input $r$ is $h_{b_r}$.  If
$\nu_r^2:=b_r^\top\Sigma_rb_r$, then its risk is
\begin{equation}
  \mathcal R_r^{\mathrm{cls},\star}
  =\frac1\pi\arccos\!\left(
  \frac{\nu_r}{\sqrt{\nu_r^2+\tau_r^2}}
  \right),
  \label{eq:classification-bayes-risk}
\end{equation}
and the classification excess risk is
\begin{equation}
  \mathcal E_r^{\mathrm{cls}}(w)
  :=\mathcal R_r^{\mathrm{cls}}(w)-
  \mathcal R_r^{\mathrm{cls},\star}.
  \label{eq:classification-excess-risk}
\end{equation}
All excess risks are therefore measured relative to the optimal predictor based on the same observed input. The three Bayes risks need not coincide, because a marginal input may not reveal all components of the latent score.

\subsection{Minimum-norm interpolating estimators}
\label{subsec:minnorm-estimators}

Let the training data consist of $n$ independent copies of the population
variables above.  For
$r\in\{1,2,J\}$, let $\mathbf X_r\in\mathbb R^{n\times d_r}$ be the design
matrix whose $i$-th row is $(X_i^{(r)})^\top$, and let
$\mathbf y^{\mathrm{reg}},\mathbf y^{\mathrm{cls}}\in\mathbb R^n$ collect the
corresponding responses. The minimum-Euclidean-norm regression interpolator is
\begin{equation}
  \widehat w_r^{\mathrm{reg}}
  \in\arg\min_{w\in\mathbb R^{d_r}}\|w\|_2
  \quad\text{subject to}\quad
  \mathbf X_rw=\mathbf y^{\mathrm{reg}}.
  \label{eq:minnorm-regression-estimator}
\end{equation}
The classification estimator is the same minimum-norm interpolation rule
applied to the binary labels, i.e.,
\begin{equation}
  \widehat w_r^{\mathrm{cls}}
  \in\arg\min_{w\in\mathbb R^{d_r}}\|w\|_2
  \quad\text{subject to}\quad
  \mathbf X_rw=\mathbf y^{\mathrm{cls}},
  \qquad
  \widehat h_r(x):=\operatorname{sign}
  \bigl((\widehat w_r^{\mathrm{cls}})^\top x\bigr).
  \label{eq:minnorm-classification-estimator}
\end{equation}
Therefore, classification is performed by the ridgeless least-squares interpolator,
not by a hard-margin support-vector machine.  The equality constraint implies
$Y_i^{\mathrm{cls}}(\widehat w_r^{\mathrm{cls}})^\top X_i^{(r)}=1$, so every
training label is classified correctly with unit signed score. Whenever $\mathbf X_r$ has full row rank, both estimators have the closed form
\begin{equation}
  \widehat w_r^{q}
  =\mathbf X_r^\top(\mathbf X_r\mathbf X_r^\top)^{-1}
  \mathbf y^{q},
  \qquad q\in\{\mathrm{reg},\mathrm{cls}\}.
  \label{eq:minnorm-closed-form}
\end{equation}
Without full row rank, the inverse is replaced by a pseudoinverse whenever the
interpolation constraints are feasible.  Throughout the task-level asymptotic
results below, for every input under discussion we restrict to covariance
sequences satisfying
\(\operatorname{rank}(\Sigma_{r,n})\geq n\) eventually.  Under the Gaussian
model, this implies \(\operatorname{rank}(\mathbf X_r)=n\) almost surely, so
the interpolation constraints are feasible and the closed form in
(\ref{eq:minnorm-closed-form}) applies.  Purely spectral covariance statements
do not require this restriction.

\subsection{Benign and non-benign overfitting}
\label{subsec:overfitting-definitions}

Benign overfitting is an asymptotic property of a sequence of learning
problems.  For each $n$, the dimensions $d_{r,n}$, covariance matrices,
dependence map, signal vectors, and noise levels may change, while the sample
size equals $n$.  The estimator always interpolates the training responses. The question is whether its population excess risk nevertheless vanishes.

\begin{definition}[Regression benign and non-benign overfitting]
\label{def:regression-benign-harmful}
Minimum-norm regression is benign for input $r$ if
the interpolation constraints are feasible with probability tending to one
and
\begin{equation}
  \mathcal E_{r,n}^{\mathrm{reg}}
  (\widehat w_{r,n}^{\mathrm{reg}})
  \xrightarrow{\mathbb P}0.
  \label{eq:regression-bo-definition}
\end{equation}
Otherwise the sequence is called regression non-benign.  It is harmful if
there exist constants $c,p>0$ such that
\begin{equation}
  \liminf_{n\to\infty}\mathbb P\!\left(
    \mathcal E_{r,n}^{\mathrm{reg}}
    (\widehat w_{r,n}^{\mathrm{reg}})\geq c
  \right)\geq p.
  \label{eq:regression-harmful-definition}
\end{equation}
\end{definition}

\begin{definition}[Classification benign and non-benign overfitting]
\label{def:classification-benign-harmful}
Minimum-norm classification is benign for input $r$ if the binary
labels are interpolated with probability tending to one and
\begin{equation}
  \mathcal E_{r,n}^{\mathrm{cls}}
  (\widehat w_{r,n}^{\mathrm{cls}})
  \xrightarrow{\mathbb P}0.
  \label{eq:classification-bo-definition}
\end{equation}
Otherwise the sequence is called classification non-benign.  It is harmful if
the analogue of (\ref{eq:regression-harmful-definition}) holds for the
classification excess risk.
\end{definition}

We abbreviate the two benign states by \(\Breg\) and
\(\Bcls\), and their non-benign counterparts by
\(\Nreg\) and \(\Ncls\). These labels are always
input-specific. For example,
\(\Bcls+\Ncls\to\Bcls\) means that the first
marginal classifier is benign, the second is non-benign, and the classifier
trained on the concatenated input is benign, all for the common task
defined in (\ref{eq:common-latent-score}). Under the standing feasibility
condition above, because excess risk is nonnegative, non-benignity is equivalent to
the existence of constants $c,p>0$ and a subsequence along which the
probability of excess risk at least $c$ is at least $p$.  Harmfulness is the
strictly stronger, persistent statement in
(\ref{eq:regression-harmful-definition}): the same lower bound holds in
\(\liminf\) along the full sequence.  We keep this distinction explicit
because a general necessity argument may produce only a bad subsequence,
whereas our explicit harmful constructions have full-sequence lower bounds.

\section{Benign Overfitting in Regression under Heterogeneous Input Fusion}
\label{sec:heterogeneous-regression}

This section distinguishes task-level benign regression, denoted by \(\Breg\), from a covariance-level spectral certificate. Regression benign overfitting means that the excess prediction risk of the minimum-norm interpolator vanishes. The full-spectrum conditions introduced below provide a sufficient covariance-level certificate for this behavior, which we denote by \(\BFS\). Below we first connect this certificate to the classical bounds in \citet{bartlett2020benign}, then establish its threshold invariance and fusion closure, and finally delineate the scope of the resulting protection law.

\subsection{The benchmark and its spectral surrogate}
\label{subsec:regression-bartlett-benchmark}

Consider \(X\sim N(0,\Sigma)\) and
\(Y=w^\top X+\varepsilon\), where
\(\varepsilon\sim N(0,\tau^2)\) is independent of \(X\). Given \(n\)
samples, let \(\widehat w\) be the minimum-Euclidean-norm interpolator.
The intrinsic notion of regression benign overfitting is
\[
  \Breg:
  \qquad
  \mathcal E_n^{\mathrm{reg}}(\widehat w_n)
  =\|\widehat w_n-w_n\|_{\Sigma_n}^2
  \xrightarrow{\mathbb P}0.
\]
This definition contains no spectral cutoff and no threshold parameter. For a fixed \(b>0\), write
\[
  k_b(\Sigma)
  :=k_{\mathrm{Bart}}^\star(\Sigma;b)
  =\min\{k\geq0:r_k(\Sigma)\geq bn\},
  \qquad
  V_b(\Sigma)
  :=\frac{k_b(\Sigma)}{n}
    +\frac{n}{R_{k_b(\Sigma)}(\Sigma)}.
\]
As usual, the minimum of the empty set is \(+\infty\). The following is the
Gaussian specialization of the upper and lower bounds of
\citet[Theorem~4]{bartlett2020benign}.

\begin{theorem}[Regression bounds~\citep{bartlett2020benign}]
\label{thm:bartlett-regression-bounds}
There are universal constants \(b_{\mathrm B},c_{\mathrm B},c_1>1\) such
that the following holds whenever \(\operatorname{rank}(\Sigma)>n\). Put
\(k^\star=k_{b_{\mathrm B}}(\Sigma)\).
If \(k^\star\geq n/c_1\), then
\[
  \mathbb E\mathcal E_n^{\mathrm{reg}}(\widehat w)
  \geq \tau^2/c_{\mathrm B}.
\]
If \(k^\star<n/c_1\), then, whenever
\(\log(1/\delta)<n/c_{\mathrm B}\), with probability at least
\(1-\delta\),
\begin{align}
  \mathcal E_n^{\mathrm{reg}}(\widehat w)
  \leq c_{\mathrm B}\bigg[
  &\|w\|_2^2\|\Sigma\|_{\mathrm{op}}
  \max\!\left\{
    \sqrt{\frac{r_0(\Sigma)}{n}},
    \frac{r_0(\Sigma)}{n},
    \sqrt{\frac{\log(1/\delta)}{n}}
  \right\}
  +\tau^2\log(1/\delta)V_{b_{\mathrm B}}(\Sigma)
  \bigg].
  \label{eq:bartlett-upper}
\end{align}
Moreover,
\begin{equation}
  \mathbb E\mathcal E_n^{\mathrm{reg}}(\widehat w)
  \geq \frac{\tau^2}{c_{\mathrm B}}
  V_{b_{\mathrm B}}(\Sigma).
  \label{eq:bartlett-variance-lower}
\end{equation}
There are also universal \(a_1,a_2>0\) such that, for every \(t\geq0\)
and all sufficiently large \(n\), one can choose
\(\|w\|_2=t\) and \(\tau^2=t^2\|\Sigma\|_{\mathrm{op}}\) for which,
with probability at least \(1/4\),
\begin{equation}
  \mathcal E_n^{\mathrm{reg}}(\widehat w)
  \geq
  \frac{t^2\|\Sigma\|_{\mathrm{op}}}{a_1}
  \mathbf 1\!\left\{
    \frac{r_0(\Sigma)}
         {n\log(1+r_0(\Sigma))}
    \geq a_2
  \right\}.
  \label{eq:bartlett-bias-lower}
\end{equation}
\end{theorem}

\begin{remark}
The bound separates the two ways interpolation can fail: a poorly spread
post-cutoff spectrum produces variance, while an unfavorable signal direction
produces bias. It is the benchmark from which the covariance certificate
below is extracted. 
\end{remark}

\begin{definition}[Threshold-specific full-spectrum certificate]
\label{def:threshold-specific-full-spectrum}
For a fixed \(b>0\), a covariance sequence \(\{\Sigma_n\}\) is in
\(\BFS(b)\) if
\begin{equation}
  \frac{r_0(\Sigma_n)}{n}\to0,
  \qquad
  \frac{k_b(\Sigma_n)}{n}\to0,
  \qquad
  \frac{n}{R_{k_b(\Sigma_n)}(\Sigma_n)}\to0.
  \label{eq:BFS-b-definition}
\end{equation}
Its complement is denoted by \(\NFS(b)\).
A sequence for which \(k_b(\Sigma_n)=+\infty\) infinitely often is, by
convention, not in \(\BFS(b)\).
\end{definition}

The three conditions have different converse statements. Failure of either
variance condition is detected, without choosing an adverse signal
orientation, by the first two lower bounds in
Theorem~\ref{thm:bartlett-regression-bounds}. More precisely, the
fixed-probability version proved in
Lemma~\ref{lem:fixed-probability-variance-lower} (see Appendix~\ref{app:fixed-probability-lower}) shows that, if
\(\inf_n\tau_n^2>0\) and either \(k^\star/n\) or
\(n/R_{k^\star}\) stays bounded away from zero, then the excess risk cannot
converge to zero in probability. By contrast, the available bias converse is
worst-case over the signal and directly detects only the regime in which
\(r_0(\Sigma_n)/[n\log(1+r_0(\Sigma_n))]\) exceeds the universal positive
threshold \(a_2\) in (\ref{eq:bartlett-bias-lower}) along a subsequence. It
does not, by itself, cover every sequence for which this ratio merely fails to
converge to zero.
Thus the sufficient condition \(r_0(\Sigma_n)/n\to0\) and the known
covariance-only necessary condition differ by a logarithmic factor. Consequently, \(\BFS(b_{\mathrm B})\) is a spectral
surrogate, not a definition of \(\Breg\). Suppose that
\[
  \|w_n\|_2^2\|\Sigma_n\|_{\mathrm{op}}=O(1),
  \qquad \tau_n^2=O(1).
\]
The upper bound then implies
\begin{equation}
  \BFS(b_{\mathrm B})
  \quad\Longrightarrow\quad
  \Breg.
  \label{eq:BFS-certifies-Breg}
\end{equation}
The implication follows by taking a sufficiently slowly diverging
\(\log(1/\delta_n)\) (see
Lemma~\ref{lem:BFS-certifies-regression} in Appendix~\ref{app:bartlett-probability}). No converse is asserted.

\subsection{Threshold invariance of the full-spectrum certificate}
\label{subsec:threshold-free-full-spectrum}

Although \(b\) enters the definition of \(k_b\), it does not alter the
asymptotic state in Definition~\ref{def:threshold-specific-full-spectrum}.
The post-cutoff condition \(R_{k_b}/n\to\infty\) supplies exactly the slack
needed to rule out knife-edge behavior at \(r_k=bn\).

\begin{lemma}[Threshold invariance]
\label{lem:regression-threshold-invariance}
For every pair of fixed constants \(b_1,b_2>0\),
\[
  \BFS(b_1)
  \quad\Longleftrightarrow\quad
  \BFS(b_2).
\]
\end{lemma}

\begin{remark}
The numerical constant used to define the spectral cutoff is asymptotically
irrelevant. Thus \(\BFS\) is an intrinsic covariance label rather than an
artifact of a particular threshold choice. Proof is provided in
Appendix~\ref{app:threshold-invariance}.
\end{remark}

We may therefore make the threshold-free definition
\begin{equation}
  \BFS
  :=\BFS(b_{\mathrm B}),
  \qquad
  \NFS
  :=\bigl(\BFS\bigr)^{\mathrm c}.
  \label{eq:threshold-free-BFS}
\end{equation}
Lemma~\ref{lem:regression-threshold-invariance} implies that any other fixed
positive threshold gives the same two asymptotic states. In what follows we
analyze this threshold-invariant covariance proxy. Task-level conclusions are
stated separately through (\ref{eq:BFS-certifies-Breg}) or an explicit lower
bound.

\subsection{Fusion closure and its sharpness}
\label{subsec:regression-fusion-closure}

Let
\begin{equation}
  \Sigma_{J,n}
  =
  \begin{pmatrix}
    \Sigma_{1,n}&C_n\\
    C_n^\top&\Sigma_{2,n}
  \end{pmatrix}
  \succeq0
  \label{eq:general-joint-covariance}
\end{equation}
be an arbitrary joint covariance. The next theorem does not require the triangular parametrization \(X_2=AX_1+U\).

\begin{theorem}[Quantitative fusion closure]
\label{thm:regression-fusion-closure}
Fix \(b>0\) and suppose
\(\Sigma_{i,n}\in\BFS(b)\) for \(i=1,2\). Then, for
every fixed \(0<b'<b/4\), the joint covariance sequence in
(\ref{eq:general-joint-covariance}) belongs to
\(\BFS(b')\).
\end{theorem}

\begin{remark}
Any positive-semidefinite fusion of two full-spectrum-certified inputs is
again full-spectrum certified, even without assuming the triangular
input-generation model. Proof is provided in
Appendix~\ref{app:fusion-closure-proof}.
\end{remark}

\begin{corollary}[Threshold-free closure]
\label{cor:regression-threshold-free-closure}
At the level of the full-spectrum surrogate,
\[
  \BFS+\BFS
  \longrightarrow \BFS.
\]
Hence the transition
\(\BFS+\BFS
\to\NFS\) is impossible for every
positive-semidefinite fusion.
\end{corollary}

\begin{remark}
This is the only universal Boolean restriction at the level of \(\BFS\):
all five other unordered marginal-to-joint spectral patterns can occur.
The same conclusion extends to any fixed number of input blocks. Proof is
provided in Appendix~\ref{app:fusion-closure-proof}.
\end{remark}

\begin{proposition}[Sharpness of fusion closure]
\label{prop:regression-transition-sharpness}
Each attainable row in Table~\ref{tab:regression-spectral-transitions} is
realized by a Gaussian sequence of the form \(X_2=A_nX_1+U_n\), where
\(A_n\) is diagonal, has no zero diagonal entries, and satisfies
\(\|A_n\|_{\mathrm F}=1+o(1)\).  Moreover, every attainable spectral row
admits a single regression task whose marginal and fused prediction risks
have the corresponding \(\Breg/\Nreg\) states.
\end{proposition}

\begin{remark}
The closure law is sharp: apart from
\(\BFS+\BFS\to\NFS\), no spectral transition is ruled out by the two
marginal certificate labels alone. Proof is provided in
Appendix~\ref{app:transition-witnesses}.
\end{remark}

For a compact statement of the witnesses, fix \(c>2b_{\mathrm B}\) and
write \(p=n^3\), \(m=\lfloor cn\rfloor\), \(\epsilon=n^{-3}\),
\(\eta=n^{-4}\), and \(\rho=n^{-8}\).  The notation \(a^{[q]}\)
denotes \(q\) repeated eigenvalues equal to \(a\), and \(\uplus\) denotes
multiset union.  Multiplicative \(1+o(1)\) perturbations, which do not
affect any full-spectrum state, are suppressed in the table.  Cases 1--3
and 5 use the common dense diagonal coupling given in
Appendix~\ref{app:transition-witnesses}. Case 4 uses a signal-aligned
diagonal coupling.  Every covariance eigendirection has a nonzero
cross-modal coupling. In this table, \(\BFS\)
and \(\NFS\) denote the full-spectrum certificate states, respectively,
and \(k_i:=k_{b_{\mathrm B}}(\Sigma_i)\) for \(i\in\{1,2,J\}\).

\begin{table}[t]
  \centering
  \scriptsize
  \caption{Complete spectral realization of the interaction pattern for
  the threshold-free full-spectrum surrogate\protect\footnotemark[1]. The two marginal states are
  interchangeable.  The impossible row is excluded by
  Corollary~\ref{cor:regression-threshold-free-closure}. Every other row is
  witnessed by the displayed spectra and the nondegenerate couplings
  constructed and verified in
  Appendix~\ref{app:transition-witnesses}.}
  \label{tab:regression-spectral-transitions}
  \begin{tabular}{p{0.18\textwidth}p{0.22\textwidth}p{0.25\textwidth}p{0.24\textwidth}}
    \toprule
    Pattern&\(\operatorname{spec}(\Sigma_1)\)&
    \(\operatorname{spec}(\Sigma_2)\)&Spectral status\\
    \midrule
    \(\BFS\!+\!\BFS\!\to\!\BFS\)&
      \(1\uplus\epsilon^{[p]}\)&\(1\uplus\epsilon^{[p]}\)&
      1: cert.; 2: cert.; \(J\): cert.\\
    \(\BFS\!+\!\BFS\!\to\!\NFS\)&
      arbitrary \(\BFS\)&
      arbitrary \(\BFS\)&
      Impossible by Corollary~\ref{cor:regression-threshold-free-closure}\\
    \(\BFS\!+\!\NFS\!\to\!\BFS\)&
      \(1\uplus\epsilon^{[p]}\)&
      \(1\uplus(n^{-2})^{[m]}\uplus\rho^{[p-m]}\)&
      1: cert.; 2: (iii); \(J\): cert.\\
    \(\BFS\!+\!\NFS\!\to\!\NFS\)&
      \(1\uplus\eta^{[p]}\)&
      \(1\uplus(n^{-2})^{[m]}\uplus\eta^{[p-m]}\)&
      1: cert.; 2: (iii); \(J\): (iii).\\
    \(\NFS\!+\!\NFS\!\to\!\BFS\)&
      \(1\uplus\epsilon^{[m]}\uplus\rho^{[p-m]}\)&
      \(\epsilon^{[p+1]}\)&
      1: (iii); 2: (i); \(J\): cert.\\
    \(\NFS\!+\!\NFS\!\to\!\NFS\)&
      \(1\uplus(n^{-2})^{[m]}\uplus\rho^{[p-m]}\)&
      \(1\uplus(n^{-2})^{[m]}\uplus\rho^{[p-m]}\)&
      1: (iii); 2: (iii); \(J\): (iii).\\
    \bottomrule
  \end{tabular}
\end{table}
\footnotetext[1]{The labels \(1,2,J\) refer to
  \(\Sigma_1,\Sigma_2,\Sigma_J\), respectively; \emph{cert.} means
  membership in \(\BFS\), while (i) or (iii) identifies the corresponding
  violated condition in (\ref{eq:BFS-b-definition}). These are
  covariance-level indicators only, not task-level benign/non-benign
  conclusions.}

Appendix~\ref{app:transition-witnesses} gives explicit diagonal witnesses and
checks the stated spectral diagnostics.  The table itself establishes only
the \(\BFS/\NFS\) labels; it does not infer task-level failure from
\(\NFS\). The appendix separately realizes the same five attainable rows at
the task level. In Cases 1-3 and 5, one may take the common response
to be independent Gaussian noise: the benign assertions follow from
Lemma~\ref{lem:BFS-certifies-regression} (see Appendix~\ref{app:bartlett-probability}), while every non-benign assertion is
forced by the fixed-probability variance lower bound in
Lemma~\ref{lem:fixed-probability-variance-lower} (see Appendix~\ref{app:fixed-probability-lower}).  Case 4 requires a
different argument because the second marginal fails only the bias-side
condition \(r_0/n\to0\).  There we specify a common signal aligned with the
leading fused eigendirection and prove directly that its second-marginal
null-space bias stays bounded away from zero. The examples show why no Boolean rule based only on the two marginal states
can govern the other four cases: fusion changes the complete ordered
spectrum. The transition is controlled by the joint values of \(r_0\),
\(k_b/n\), and \(n/R_{k_b}\), not by total energy or either marginal
effective rank in isolation.

\subsection{Equivalence under well-conditioned forward and reverse maps}
\label{subsec:regression-operator-equivalence}

The closure theorem is one-sided because an arbitrary fusion map may discard
or severely attenuate directions. Uniformly well-conditioned maps give a
stronger conclusion: the fused and marginal full-spectrum states are
equivalent. Continue under the heterogeneous Gaussian model in
(\ref{eq:heterogeneous-generation-model}).

\begin{theorem}[Forward--reverse spectral equivalence]
\label{thm:regression-forward-reverse-equivalence}
\mbox{}\par
\noindent
\emph{Forward map.}
Assume, for constants \(c_A,C_A>0\) independent of \(n\), that
  \begin{equation}
    c_AI_{d_1}\preceq A_n^\top A_n\preceq C_AI_{d_1}.
    \label{eq:forward-map-conditioning}
  \end{equation}
\par\noindent
Then
  \begin{equation}
    \Sigma_{J,n}\in\BFS
    \quad\Longleftrightarrow\quad
    \Sigma_{2,n}\in\BFS.
    \label{eq:forward-map-BFS-equivalence}
  \end{equation}

\par\smallskip
\noindent\emph{Reverse map.}
Suppose that \(\Sigma_{2,n}\) is invertible and define the reverse
  conditional-mean operator
  \begin{equation}
    B_n
    :=\operatorname{Cov}(X_1,X_2)\Sigma_{2,n}^{-1}
    =D_nA_n^\top
      (A_nD_nA_n^\top+\Sigma_{U,n})^{-1}.
    \label{eq:reverse-conditional-mean-map}
  \end{equation}
  If, for constants \(c_B,C_B>0\) independent of \(n\),
  \begin{equation}
    c_BI_{d_2}\preceq B_n^\top B_n\preceq C_BI_{d_2},
    \label{eq:reverse-map-conditioning}
  \end{equation}
\par\noindent
Then
  \begin{equation}
    \Sigma_{J,n}\in\BFS
    \quad\Longleftrightarrow\quad
    \Sigma_{1,n}\in\BFS.
    \label{eq:reverse-map-BFS-equivalence}
  \end{equation}
Consequently, if both conditioning assumptions hold, then
\begin{equation}
  \Sigma_{1,n}\in\BFS
  \quad\Longleftrightarrow\quad
  \Sigma_{2,n}\in\BFS
  \quad\Longleftrightarrow\quad
  \Sigma_{J,n}\in\BFS.
  \label{eq:three-way-BFS-equivalence}
\end{equation}
\end{theorem}

\begin{remark}
When the forward conditional-mean map $A_n$ is uniformly well-conditioned, modality 2 and the fused setting are equivalent in terms of $\BFS$; 
likewise, when the reverse conditional-mean map $B_n$ is uniformly well-conditioned, modality 1 and the fused setting are equivalent in terms of $\BFS$.
Proof is provided in Appendix~\ref{app:operator-equivalence-proof}.
\end{remark}

The closure and equivalence theorems delimit the protected covariance
regime. We next use two complementary case studies to show what these
structural statements do, and do not, control. The first remains entirely at
the level of \(\BFS\). The second exploits the known
logarithmic gap to leave that spectral regime and examine actual regression
risk.

\subsection{Case study I: vanishing innovation energy can still be harmful}
\label{subsec:regression-small-energy-case-study}

The following construction makes the preceding point quantitative. Fix
\(c>2b_{\mathrm B}\), let \(p_n=n^3\) and
\(m_n=\lfloor cn\rfloor\), and set
\begin{align}
  D_n&=\operatorname{diag}(1,n^{-4}I_{p_n}),\qquad A_n=I, \notag\\
  \Sigma_{U,n}
  &=\operatorname{diag}\!\left(
      n^{-8},\,n^{-2}I_{m_n},\,n^{-8}I_{p_n-m_n}
    \right).
  \label{eq:small-energy-construction}
\end{align}
Thus \(X_2=X_1+U\),
\(\Sigma_{1,n}=D_n\), \(\Sigma_{2,n}=D_n+\Sigma_{U,n}\), and the joint
covariance is given by (\ref{eq:three-input-covariances}). The innovation is small
in both operator norm and total energy:
\[
  \|\Sigma_{U,n}\|_{\mathrm{op}}=n^{-2},
  \qquad
  \operatorname{tr}(\Sigma_{U,n})=c/n+o(n^{-1}).
\]
However,
\begin{equation}
  \Sigma_{1,n}\in\BFS,
  \qquad
  \Sigma_{2,n},\Sigma_{J,n}\in\NFS.
  \label{eq:small-energy-spectral-states}
\end{equation}

For \(\Sigma_{1,n}\), the Bartlett cutoff is \(1\) and
\(R_1(\Sigma_{1,n})=p_n=n^3\). For \(\Sigma_{2,n}\), the cutoff is again
\(1\), but its post-cutoff spectrum contains \(m_n\) eigenvalues of order
\(n^{-2}\); direct calculation gives
\[
  \frac{n}{R_1(\Sigma_{2,n})}
  \longrightarrow \frac{c}{(c+1)^2}.
\]
The joint covariance decomposes into coordinatewise \(2\times2\) blocks.
After its leading eigenvalue is removed,
\[
  S_1(\Sigma_{J,n})=\frac{c+2}{n}(1+o(1)),
  \qquad
  \sum_{j>1}\lambda_j(\Sigma_{J,n})^2
  =\frac{c}{n^3}(1+o(1)),
\]
and therefore
\begin{equation}
  \frac{n}{R_1(\Sigma_{J,n})}
  \longrightarrow\frac{c}{(c+2)^2}>0.
  \label{eq:small-energy-joint-obstruction}
\end{equation}
All cutoff and \(r_0\) checks, including the block-eigenvalue expansions, are
given in Appendix~\ref{app:small-energy-proof}.

This covariance transition also produces a task-level transition without any
signal-orientation issue. Take the common regression task \(Y=\varepsilon\)
with fixed \(\tau^2>0\). ~(\ref{eq:BFS-certifies-Breg}) yields
\(\Breg\) for \(X_1\), whereas the variance lower bound and
(\ref{eq:small-energy-joint-obstruction}) give non-benign regression for both
\(X_2\) and the fused input. Thus even
\(\mathbb E\|X_2-X_1\|_2^2=\operatorname{tr}(\Sigma_{U,n})\to0\) does not
protect benign overfitting. A vanishing amount of innovation energy can
matter when it is arranged as an \(O(n)\)-wide spectral shoulder.

The first case study concerns the covariance proxy itself: small innovation
energy does not control how that energy is distributed across eigendirections.
The next case study concerns the boundary of the proxy: closure of
\(\BFS\) does not imply unrestricted closure of actual
regression benign overfitting for tasks that lie in Bartlett's logarithmic
bias gap.

\subsection{Case study II: task-level non-closure inside the logarithmic gap}
\label{subsec:regression-log-gap-counterexample}

We now cross the boundary of the covariance proxy. Although
\(\BFS\) is closed under fusion, actual regression benign
overfitting can occur outside this certified region. The following
counterexample uses precisely the gap between the sufficient and necessary
bias conditions in Theorem~\ref{thm:bartlett-regression-bounds} and shows
that unrestricted task-level closure is false.

Let \(h_n=\sqrt{\log n}\), \(p_n=n^3\), and define
\[
  \varepsilon_n=h_n^{-2},\qquad
  \eta_n=h_n/n^2,\qquad
  \rho_n=n^{-10}.
\]
In dimension \(d_n=p_n+1\), take
\begin{equation}
  D_n=\operatorname{diag}(1,\eta_nI_{p_n}),\qquad
  A_n=I_{d_n},\qquad
  \Sigma_{U,n}
  =\operatorname{diag}(\varepsilon_n,\rho_nI_{p_n}).
  \label{eq:log-gap-covariance-construction}
\end{equation}
The common regression task is induced by
\begin{equation}
  \beta_n=\sqrt{\varepsilon_n}e_1,\qquad
  \gamma_n=-\varepsilon_n^{-1/2}e_1,\qquad
  \xi=0.
  \label{eq:log-gap-label-construction}
\end{equation}
Equivalently,
\(Y=\sqrt{\varepsilon_n}X_{1,1}
-\varepsilon_n^{-1/2}U_1\), whose variance converges to one.

\begin{theorem}[Regression non-closure]
\label{thm:regression-task-level-nonclosure}
For this sequence, the three minimum norm regression risks satisfy
\begin{equation}
  \mathcal E_{1,n}^{\mathrm{reg}}
  \xrightarrow{\mathbb P}0,
  \qquad
  \mathcal E_{2,n}^{\mathrm{reg}}
  \xrightarrow{\mathbb P}0,
  \qquad
  \mathbb P\!\left\{
    \mathcal E_{J,n}^{\mathrm{reg}}\geq\frac12
  \right\}\longrightarrow1.
  \label{eq:task-level-nonclosure-risks}
\end{equation}
Thus actual regression benign overfitting is not closed under fusion:
\[
  \Breg+\Breg
  \longrightarrow \Nreg.
\]
This remains true even though \(A_nA_n^\top=I\) and the reverse
conditional-mean operators \(B_n\) satisfy uniform two-sided bounds.
\end{theorem}

\begin{remark}
The full-spectrum closure law does not extend to arbitrary regression tasks:
two benign marginal tasks can have a non-benign fusion outside the certified
region. This example lies in the logarithmic gap between the known sufficient
and necessary spectral conditions, so it does not contradict
\(\BFS+\BFS\to\BFS\). Proof is provided in
Appendix~\ref{app:log-gap-counterexample-proof}.
\end{remark}

\subsection{Scope of the closure law}
\label{subsec:regression-closure-scope}

Corollary~\ref{cor:regression-threshold-free-closure} is a closure theorem for
the spectral certificate \(\BFS\). Combined with
(\ref{eq:BFS-certifies-Breg}), it proves closure for regression tasks whose
two marginal benign states are certified by the full-spectrum
conditions~\citep{bartlett2020benign} and whose joint signal and noise scales satisfy the stated
boundedness assumptions. Theorem~\ref{thm:regression-task-level-nonclosure}
shows that this qualification is essential: without a uniform restriction on
the task geometry, the unrestricted law
\(\Breg+\Breg\to\Breg\) fails even under
well-conditioned forward and reverse maps.

Accordingly, the universal statement established in this section is the
covariance-level closure
\(\BFS+\BFS
\to\BFS\). It yields task-level closure only when the
full-spectrum certificate and its signal-scale assumptions are available.
Actual benign regression is signal-specific and can occur outside that
protected region. The logarithmic-gap construction then shows that fusion can
convert a marginally negligible or conditionally invisible signal into a
jointly visible low-variance direction with nonvanishing interpolation bias.

\section{Benign Overfitting in Classification under Heterogeneous Input Fusion}
\label{sec:classification-results}

The regression closure law in Section \ref{sec:heterogeneous-regression} is a full-spectrum covariance statement. Classification is governed by a different geometry: the nuisance spectrum controls the contamination introduced by minimum-norm interpolation, while its scale relative to the predictive direction determines how much signal survives. We first characterize this balance for a single input, then determine how fusion changes its two components, and finally derive the complete classification interaction pattern.

\subsection{A spectral characterization through survival and contamination}
\label{subsec:classification-main-characterization}

We begin with a single input and suppress its index. Suppose that its population coefficient is aligned with an eigendirection $e_t$ of \(\Sigma\). After rescaling the latent score, the induced classification problem can be written as
\begin{equation}
  \Sigma=\operatorname{diag}
  (\lambda_t,\mu_1,\ldots,\mu_{d-1}),
  \qquad
  Y=\operatorname{sign}(Z_t+\tau G),
  \label{eq:classification-one-sparse-model}
\end{equation}
where \(Z_t,G\sim N(0,1)\) are independent, \(\lambda_t>0\), and the true
coefficient is \(e_t/\sqrt{\lambda_t}\).  The nuisance eigenvalues are
ordered as \(\mu_1\geq\cdots\geq\mu_{d-1}\).  All spectral quantities below
are computed from this signal-deleted spectrum, never from the full spectrum.

Fix universal constants \(b_{\mathrm C},c_{\mathrm C}>1\) large enough for
Lemma~\ref{lem:nuisance-gram-variance} (see Appendix~\ref{app:classification-delta-equivalence}), and let
\(k_{-t}^\star=k_{-t}^\star(b_{\mathrm C})\) be the cutoff in
(\ref{eq:nuisance-spectrum-cutoff}).  When
\(k_{-t}^\star<+\infty\), put
\begin{equation}
  S_n:=S_{k_{-t}^\star}^{(-t)},\qquad
  R_n:=R_{k_{-t}^\star}^{(-t)},\qquad
  V_n:=\frac{k_{-t}^\star}{n}+\frac{n}{R_n},
  \qquad
  s_n:=\frac{n\lambda_t}{S_n+n\lambda_t}.
  \label{eq:classification-spectral-scales}
\end{equation}
Define the classification balance parameter
\begin{equation}
  \Delta_n:=\frac{V_n}{s_n^2}
  =\left(1+\frac{S_n}{n\lambda_t}\right)^2
   \left(\frac{k_{-t}^\star}{n}+\frac{n}{R_n}\right).
  \label{eq:classification-interaction-quantity}
\end{equation}
When \(k_{-t}^\star=+\infty\), set
\(V_n=\Delta_n=+\infty\); the quantities \(S_n,R_n,s_n\) will not be used.
The term \(V_n\) is the contamination scale; \(s_n\) is the survival scale.

For a fitted vector \(w\), define
\[
  \mathrm{SU}(w):=\sqrt{\lambda_t}\,w_t,
  \qquad
  \mathrm{CN}(w)^2:=\sum_{j\ne t}\lambda_jw_j^2,
  \qquad
  \Gamma(w):=\frac{\mathrm{CN}(w)}{\mathrm{SU}(w)}
\]
whenever \(\mathrm{SU}(w)>0\).  Let
\(\rho_\tau=(1+\tau^2)^{-1/2}\).  The bivariate Gaussian orthant identity
\citep[Chapter~2]{anderson2003multivariate} gives
\begin{equation}
  \mathcal E_n^{\mathrm{cls}}(w)
  =\frac1\pi\left[
    \arccos\!\left(
      \frac{\rho_\tau}{\sqrt{1+\Gamma(w)^2}}
    \right)-\arccos(\rho_\tau)
  \right].
  \label{eq:classification-risk-gamma}
\end{equation}
Thus, for uniformly bounded \(\tau\), classification excess risk vanishes
if and only if \(\Gamma\) vanishes.

We call a sequence \emph{regular} if
\begin{equation}
  k_{-t}^\star<\frac{n}{c_{\mathrm C}}
  \label{eq:classification-regular-regime}
\end{equation}
eventually.  This is the range in which the matching Gaussian variance-trace
bounds of \citet[Lemmas~11, 16, and 17]{bartlett2020benign} apply. Every sequence with \(k_{-t}^\star/n\to0\) is
eventually regular.

\begin{theorem}[Classification criterion]
\label{thm:classification-main-characterization}
Consider the one-sparse model above, and suppose that
\(0\leq\tau_n\leq\bar\tau<\infty\).
\begin{enumerate}
  \item For an arbitrary spectral sequence, \(\Delta_n\to0\) implies
  \(\mathcal E_n^{\mathrm{cls}}(\widehat w_n)\to_{\mathbb P}0\).
  \item In the regular regime (\ref{eq:classification-regular-regime}),
  \begin{equation}
    \mathcal E_n^{\mathrm{cls}}(\widehat w_n)
    \xrightarrow{\mathbb P}0
    \quad\Longleftrightarrow\quad
    \Delta_n\longrightarrow0.
    \label{eq:classification-regular-iff}
  \end{equation}
\end{enumerate}
\end{theorem}

\begin{remark}
Classification is benign exactly when nuisance contamination is small
relative to the surviving predictive signal. In the regular regime this is
fully captured by the single scalar \(\Delta_n\); outside that regime,
\(\Delta_n\to0\) remains sufficient. A positive limiting obstruction implies non-benignity in the regular regime,
while a uniformly positive obstruction implies harmfulness.
Proof is provided in Appendix~\ref{app:classification-delta-equivalence}.
\end{remark}

\subsection{A survival-contamination decomposition for heterogeneous fusion}
\label{subsec:classification-fusion-balance}

Return to the three inputs \(r\in\{1,2,J\}\) generated by the
common score in (\ref{eq:common-latent-score}).  Assume that each
input-specific reduction has a one-sparse eigendirection and a
uniformly bounded normalized probit noise level.  Delete the signal direction associated with each input and define \(S_{r,n},R_{r,n},V_{r,n},s_{r,n}\), and
\(\Delta_{r,n}\) by (\ref{eq:classification-spectral-scales})-%
(\ref{eq:classification-interaction-quantity}).  These objects generally use
three different nuisance spectra and three different signal eigenvalues.

The following identity is an algebraic decomposition, not by itself a
structural fusion theorem.  Its role is to organize the two spectral changes
that the explicit heterogeneous constructions below must control. Whenever
the signal-deleted cutoffs of inputs \(i\) and \(J\) are finite and
\(V_{i,n}>0\), define
\begin{equation}
  \mathcal T_{i\to J,n}^{\mathrm{CN}}
  :=\frac{V_{J,n}}{V_{i,n}},
  \qquad
  \mathcal T_{i\to J,n}^{\mathrm{SU}}
  :=\frac{s_{i,n}^2}{s_{J,n}^2}
  =\left(
    \frac{1+S_{J,n}/(n\lambda_{t,J})}
         {1+S_{i,n}/(n\lambda_{t,i})}
   \right)^2.
  \label{eq:classification-fusion-transfer-factors}
\end{equation}

\begin{proposition}[Survival-contamination transfer identity]
\label{prop:classification-fusion-balance}
Suppose inputs \(i\in\{1,2\}\) and \(J\) lie in the regular
one-sparse regime and \(V_{i,n}>0\). Then
\begin{equation}
  \Delta_{J,n}
  =\Delta_{i,n}
   \mathcal T_{i\to J,n}^{\mathrm{CN}}
   \mathcal T_{i\to J,n}^{\mathrm{SU}}.
  \label{eq:classification-fusion-balance-identity}
\end{equation}
Moreover, their minimum-norm interpolators obey
\[
  \mathrm{SU}(\widehat w_r)\asymp_{\mathbb P}s_{r,n},\qquad
  \mathrm{CN}(\widehat w_r)^2=O_{\mathbb P}(V_{r,n}),\qquad r\in\{i,J\},
\]
with matching fixed-positive-probability lower bounds for the contamination.
Consequently, the fused classifier is benign if and only if the right-hand
side of (\ref{eq:classification-fusion-balance-identity}) tends to zero.
\end{proposition}

\begin{remark}
Fusion changes classification through two separate effects: it changes
nuisance contamination and it changes signal survival. Neither effect is
monotone, so marginal benign/non-benign labels alone cannot predict the
fused classifier.
Proof is provided in Appendix~\ref{app:classification-fusion-mixing}.
\end{remark}

\subsection{The complete six-pattern classification law}
\label{subsec:classification-six-pattern-law}

We now turn the balance law into a sharp structural statement.  The two
marginal columns are regarded as unordered, leaving six Boolean patterns.

\begin{theorem}[Complete classification interaction pattern]
\label{thm:classification-six-patterns}
All six unordered classification states
\[
  \begin{aligned}
    \Bcls+\Bcls&\to\Bcls,
    &\qquad \Bcls+\Bcls&\to\Ncls,\\
    \Bcls+\Ncls&\to\Bcls,
    &\qquad \Bcls+\Ncls&\to\Ncls,\\
    \Ncls+\Ncls&\to\Bcls,
    &\qquad \Ncls+\Ncls&\to\Ncls
  \end{aligned}
\]
are realized by the heterogeneous Gaussian sequences in
Table~\ref{tab:classification-witness-scales}. In every realization:
\begin{enumerate}
  \item the three classifiers use the same latent label mechanism;
  \item each input has a one-sparse signal eigendirection;
  \item all three signal-deleted cutoffs are \(o(n)\), so the conclusions use
  both directions of Theorem~\ref{thm:classification-main-characterization}.
  \item the fusion map acts nontrivially on every covariance eigendirection.
  In the five common dense-lift rows its nuisance restriction has squared
  Frobenius norm one; in the \(\Ncls+\Ncls\to\Bcls\) row every nuisance coefficient is
  \(n^{-2}\), so its squared Frobenius norm is asymptotic to \(n^{-2}\).
  The shared signal has nonvanishing cross-covariance in every row.
\end{enumerate}
Hence classification admits no nontrivial Boolean closure law based only on
the two marginal benign/non-benign states.
\end{theorem}

\begin{remark}
Unlike the full-spectrum regression certificate, classification has no
universal closure rule: every benign/non-benign pattern of two marginal
inputs and their fusion can occur. Proof is provided in
Appendix~\ref{app:classification-transition-witnesses}.
\end{remark}

The six constructions are organized and verified in
Appendix~\ref{app:classification-transition-witnesses}.  Five rows use one
common dense weak-coupling lift of the target nuisance spectra. The
\(\Ncls+\Ncls\to\Bcls\) row uses a scale-dependent shared signal and the uniform weak
nuisance coupling used in the experiments.  In both templates every
covariance eigendirection is coupled, and the common label is chosen so that
the induced coefficient is a signal eigenvector in all three inputs.

Table~\ref{tab:classification-witness-scales} records only the target
nuisance spectra and the resulting interaction scales.  Here
\(a^{[m]}\) denotes \(m\) repeated eigenvalues equal to \(a\),
\(q_n=\lfloor\sqrt n\rfloor\), \(p_n=\lfloor\rho n\rfloor\), and
\(\rho>b_{\mathrm C}\) is fixed.
For the \(\Ncls+\Ncls\to\Bcls\) row, let
\(a_n:=1+n^{\ell-9/2}\), \(b_n:=n^{-M}+n^{\ell-9/2}\),
\(d_n:=n^{2(r+\ell)}\), and \(j_n:=n^{2\ell-1}\);
it assumes \(0<r<1/2\), \(0<\ell<1/2-r\), and \(M>2\).

  \begin{table}[!t]
  \centering
  \caption{Complete spectral realization of the six classification patterns
  in Theorem~\ref{thm:classification-six-patterns}.  The two middle columns
  specify the target marginal nuisance spectra.  The entrywise nonzero
  coupling maps, common-label constructions, signal checks, and all
  \(S_k,R_k,\Delta\) calculations are given in
  Appendix~\ref{app:classification-transition-witnesses}.  The two marginal
  states are interchangeable, and only orders relevant to \(\Delta\) are
  displayed.}
  \label{tab:classification-witness-scales}
  \begin{tabular}{p{0.21\textwidth}p{0.24\textwidth}p{0.24\textwidth}p{0.20\textwidth}}
  \toprule
  Pattern&Nuisance spectrum 1&Nuisance spectrum 2&
  \((\Delta_1,\Delta_2,\Delta_J)\)\\
  \midrule
  \(\Bcls+\Bcls\to\Bcls\)&\((n^{-1})^{[n^2]}\)&\((n^{-1})^{[n^2]}\)&
  \((n^{-1},n^{-1},n^{-1})\)\\
  \(\Bcls+\Ncls\to\Bcls\)&\((n^{-1})^{[n^2]}\)&
  \((n^{-1})^{[p_n]}\)&
  \((n^{-1},\rho^{-1},n^{-1})\)\\
  \(\Bcls+\Ncls\to\Ncls\)&\((n^{-2})^{[n^2]}\)&\(1^{[n^2]}\)&
  \((n^{-1},n,n)\)\\
  \(\Ncls+\Ncls\to\Ncls\)&\((n^{-1})^{[p_n]}\)&
  \((n^{-1})^{[p_n]}\)&
  \((\rho^{-1},\rho^{-1},(2\rho)^{-1})\)\\
  \(\Bcls+\Bcls\to\Ncls\)&\((n^{-2})^{[n^4]}\)&
  \((n^2)^{[q_n]}\uplus(n^{-2})^{[n^2]}\)&
  \((n^{-1},n^{-1/2},n^{3/2})\)\\
  \(\Ncls+\Ncls\to\Bcls\)&\((n^{\ell-1/2})^{[n^2]}\)&
  \(a_n^{[p_n]}\uplus b_n^{[n^2-p_n]}\)&
  \((d_n,\rho^{-1},j_n)\)\\
  \bottomrule
  \end{tabular}
  \end{table}

\begin{remark}
The transition \(\Bcls+\Bcls\to\Ncls\) is the clearest contrast with
regression spectral closure: individually benign inputs need not remain
benign after fusion. This is why classification cannot be summarized by a
Boolean rule on marginal states alone.
\end{remark}

\subsection{Discussion: a signal-to-noise interpretation}
The survival-contamination criterion also admits a signal-to-noise
interpretation. For a fitted classifier $w$, consider
$\mathrm{SNR}_{\mathrm{eff}}(w)
:=
\frac{\mathrm{SU}(w)^2}{\mathrm{CN}(w)^2}$.
Since $\Gamma(w)^2=\mathrm{SNR}_{\mathrm{eff}}(w)^{-1}$, the excess-risk
identity in (\ref{eq:classification-risk-gamma}) shows that, under the bounded-noise condition of
Theorem~\ref{thm:classification-main-characterization}, benign classification corresponds to a diverging effective
signal-to-noise ratio. At the deterministic spectral scale,
$\Delta_n=V_n/s_n^2$ plays the role of its inverse: $s_n$ measures the
predictive signal retained by minimum-norm interpolation, whereas $V_n$
measures the nuisance contamination introduced by interpolation. Therefore, the
relevant signal-to-noise ratio is not a property of the population problem
alone, but of how the interpolating estimator transforms signal and nuisance
through the covariance geometry.

This distinction is important under fusion. A population-level
signal-to-noise ratio, such as the ratio between predictable and residual
variance, does not by itself determine whether interpolation is benign.
Fusion may improve population-level predictability while simultaneously weakening
signal survival or increasing nuisance contamination. Consequently, its
effective signal-to-noise ratio induced by interpolation can decrease even when
the population-level prediction problem becomes easier. A related distinction
appears in regression, where benign overfitting depends on the full covariance
spectrum and the alignment of the predictive signal with spectral directions,
rather than on any single scalar signal-to-noise measure. 

\section{Cross-Task Separation under Heterogeneous Fusion}
\label{sec:heterogeneous-separation}
Sections 3 and 4 establish different fusion laws for regression and classification, but treat the two tasks separately. We now hold both the heterogeneous inputs and the latent score fixed and compare the two learning problems directly. This isolates task dependence from any change in the input distribution or target-generating mechanism.
\begin{theorem}[Cross-task separation under fusion]
\label{thm:cross-task-separation}
Within the heterogeneous Gaussian model
\(X_2=A_nX_1+U_n\), there exist two sequences of common latent scores
\(L_n\) for which the regression response and classification label are
\(Y_n^{\mathrm{reg}}=L_n\) and
\(Y_n^{\mathrm{cls}}=\operatorname{sign}(L_n)\), respectively, and the
following two transitions occur.
\begin{enumerate}
  \item \emph{Fusion rescue from two non-benign inputs.}
  Each marginal input is non-benign for both tasks, whereas the
  fused input is benign only for classification:
  \[
    \begin{array}{c|cc}
      &\text{minimum-norm regression}
      &\text{minimum-norm classification}\\
      \hline
      X_1&\Nreg&\Ncls\\
      X_2&\Nreg&\Ncls\\
      (X_1,X_2)&\Nreg&\Bcls
    \end{array}
  \]
  \item \emph{Task separation from two benign inputs.}
  Each marginal input is benign for both tasks, while fusion
  preserves classification benignity but destroys regression benignity:
  \[
    \begin{array}{c|cc}
      &\text{minimum-norm regression}
      &\text{minimum-norm classification}\\
      \hline
      X_1&\Breg&\Bcls\\
      X_2&\Breg&\Bcls\\
      (X_1,X_2)&\Nreg&\Bcls
    \end{array}
  \]
\end{enumerate}
Both sequences use nonzero cross-block coupling and a single latent score for all three inputs; the target-generating mechanism is not changed across inputs.
\end{theorem}

\begin{remark}
The above theorem states that the effect of fusion is task-dependent even when the
inputs and latent score are held fixed: the same fusion can help
classification while harming regression. The second construction gives a
formal classification separation, since its marginal Bayes advantages vanish
asymptotically. Proof is provided in
Appendix~\ref{app:cross-task-separation-proofs}.
\end{remark}

  \section{Experiments}
  \label{sec:experiments}

  The experiments serve three purposes.  First, they illustrate finite-sample
  minimum-norm risks that are consistent with the interaction states proved in
  Sections~\ref{sec:heterogeneous-regression} and~\ref{sec:classification-results}.
  Second, they illustrate the mechanisms behind the two regression case studies:
    an innovation with vanishing total energy can create a deleterious spectral
    shoulder, and a task inside the logarithmic bias gap of \citet{bartlett2020benign} can be benign in
    both marginals but non-benign after fusion.
  Third, they hold the input distribution and latent score fixed to display the
  cross-task separation of Theorem~\ref{thm:cross-task-separation}.  The experiments use exactly the
  spectral sequences and common-score constructions stated in the main text and
  appendices. They are not separately tuned finite-dimensional examples.
  Finally, we include an out-of-model multi-view sanity check on a real data
  set, as a limited practical illustration rather than a validation of the
  Gaussian theory.
  \subsection{Experimental protocol}
  \label{subsec:experimental-protocol}

  For every construction, we generate the triangular Gaussian model
  \(X_2=AX_1+U\), with \(X_1\perp U\), and fit the minimum-Euclidean-norm
  interpolant separately using \(X_1\), \(X_2\), and the concatenated
  input \(X_J=(X_1^\top,X_2^\top)^\top\).  The synthetic Gaussian experiments
  cover sample sizes from \(n=16\) to \(n=15000\), with \(50\) independent
  repetitions at each displayed point.  We use the fixed numerical cutoff
  \(b=2\) when reporting the Bartlett quantities.  For readability, the
  overview figures display the approximately log-spaced
  subset \(n\in\{16,50,100,200,400,800,1600,3000,5000,10000,15000\}\). For each construction, the Monte Carlo samples for \(X_1\), \(X_2\), and
  \(X_J\) are drawn independently from their respective exact marginal laws.
  Thus, the figures compare the risk distributions of the three estimators,
  rather than paired realization-by-realization fusion effects. Several witnesses have ambient dimension \(n^3\) or \(n^4\).  To make these
  simulations tractable, we use exact block-Wishart compression: each repeated
  \(\lambda^{[m]}\) block is represented through \(\lambda ZZ^\top\), while
  finite-rank signal blocks are added via Woodbury updates
  \citep[Section~0.7.4]{hornjohnson2013matrix}.  This preserves the Gaussian
  Gram law of the full design. Regression excess risk is evaluated as the exact population quadratic risk of
  the fitted vector.  Classification risk uses the exact Gaussian sign formula,
  and is reported relative to the Bayes risk of the corresponding
  input.  Thus no finite test set is used.  Each curve shows the
  unsmoothed across-repetition median at every displayed sample size, and the translucent
  band represents the empirical interquartile range (25th-75th percentiles)
  across repetitions.  Deterministic population spectral diagnostics are shown
  without uncertainty bands. 

  \subsection{The five attainable regression interactions}
  \label{subsec:experimental-regression-interactions}

  Figure~\ref{fig:regression-five-experiments} reports the five attainable rows
  of Table~\ref{tab:regression-spectral-transitions}.  The omitted Boolean row
  \(\BFS+\BFS\to\NFS\) is the transition ruled out for the threshold-free
  full-spectrum surrogate by
  Corollary~\ref{cor:regression-threshold-free-closure}.  In the four
  variance-driven witnesses, harmfulness is certified by the matching Bartlett
  variance lower bound. The \(\NFS+\NFS\to\BFS\) witness uses the explicit
  signal-aligned bias construction of
  Appendix~\ref{app:transition-witnesses}.
  \begin{figure}[!t]
    \centering
    \includegraphics[width=\textwidth]
    {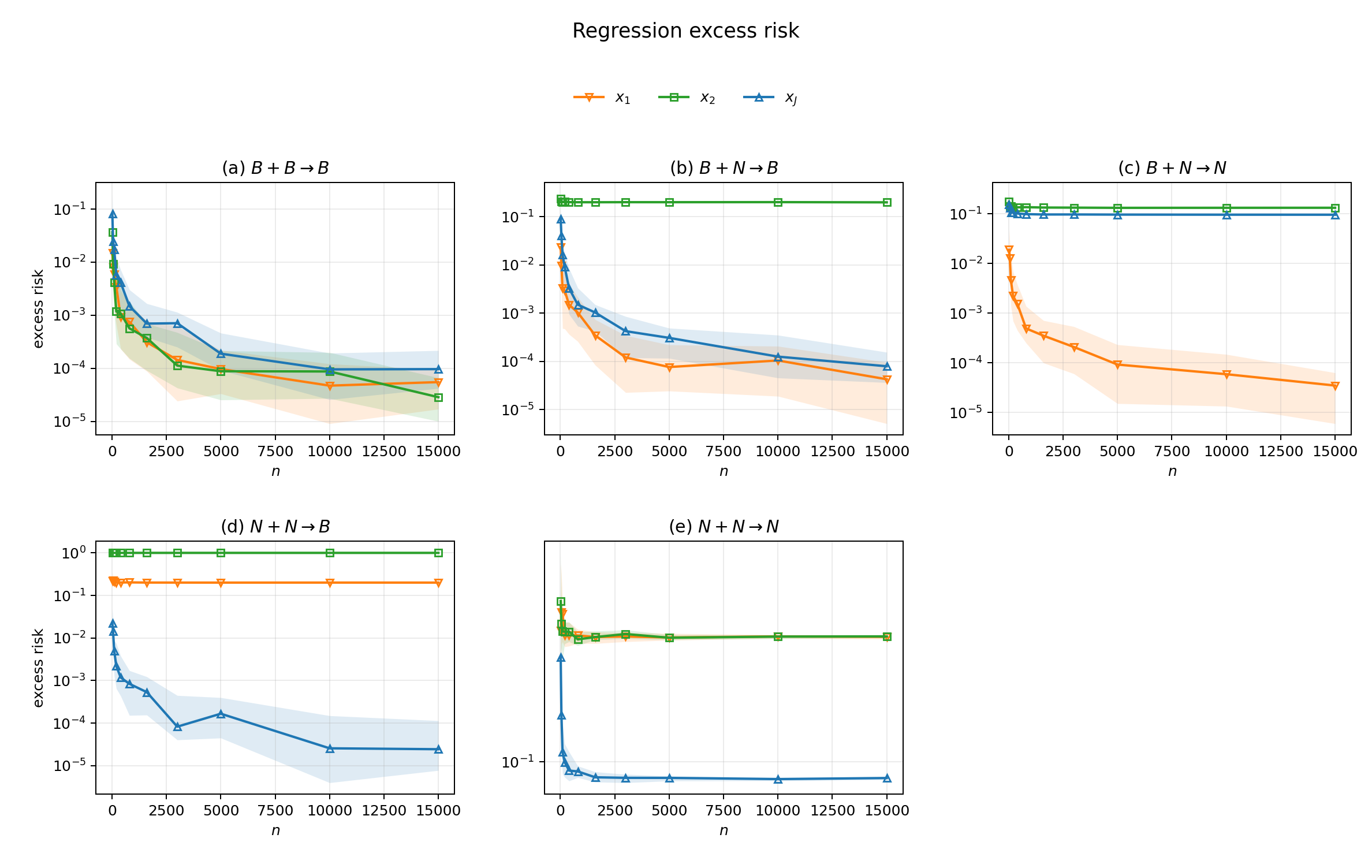}
    \caption{Population regression excess risk for the five attainable
    interaction patterns.  The panels follow the ordering in
    Table~\ref{tab:regression-spectral-transitions}; \(\Breg\) and \(\Nreg\)
    denote task-level benign and non-benign regression behavior for the displayed
    response.  Curves show unsmoothed Monte Carlo medians, and shaded bands
    show the interquartile range.  The corresponding repetition counts are specified in
    Section~\ref{subsec:experimental-protocol}.}
    \label{fig:regression-five-experiments}
  \end{figure} 
  The five panels illustrate behavior consistent with the predicted interaction
  law.  In panel (a), all
  three risks decrease toward zero.  In panel (b), the second-source risk
  stabilizes near \(0.20\), whereas the first and joint risks are much smaller
  at the largest displayed sample size, a pattern consistent with
  \(\Breg+\Nreg\to\Breg\). Panel (c) keeps
  both the second and joint risks bounded away from zero while the first risk
  vanishes.  The signal-aligned construction in panel (d) is particularly
  pronounced: the two marginal risks remain near \(0.20\) and \(1\), while the
  joint risk falls below \(10^{-4}\).  Finally, all three
  risks remain separated from zero in panel (e).  Taken together, the panels
  provide finite-sample illustrations consistent with the covariance-level
  closure statement and its corresponding task-level risk behavior.

  \subsection{Regression case studies}
  \label{subsec:experimental-regression-case-studies}

  Figure~\ref{fig:regression-case-study-experiments} isolates the two boundary
  phenomena from
  Sections~\ref{subsec:regression-small-energy-case-study}
  and~\ref{subsec:regression-log-gap-counterexample}.  The top row concerns the
  small-innovation construction.  The panels report excess risk and the
  corresponding scaled variance quantity \(n/R_{k^\star}\) across sample sizes.
  Although \(\|\Sigma_U\|_{\mathrm{op}}=n^{-2}\) and
  \(\operatorname{tr}(\Sigma_U)=c/n+o(n^{-1})\), panels (a)--(b) illustrate
  that only \(X_1\) has vanishing risk and variance scale, whereas \(X_2\) and
  the fused representation \(X_J\) do not exhibit this benign behavior.
  \begin{figure}[!t]
    \centering
    \includegraphics[width=\textwidth]
    {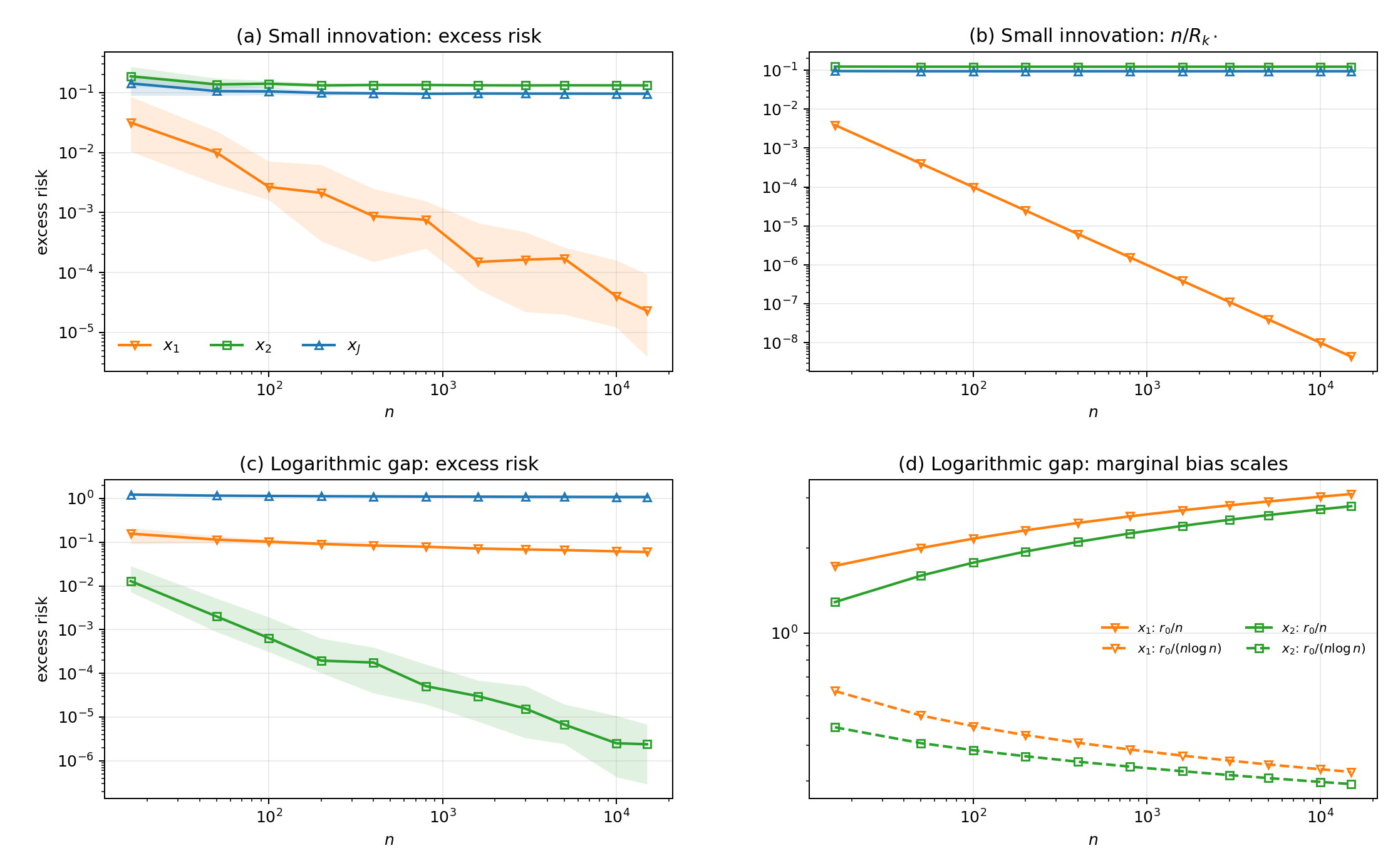}
    \caption{The two regression case studies.  Panels (a)-(b) illustrate how
    vanishing innovation energy can be associated with a nonvanishing variance obstruction.
    Panels (c)-(d) probe Bartlett's logarithmic bias gap: the two marginal
    risks decrease while the fused risk stays separated, and the marginal
    quantities \(r_0/n\) and \(r_0/(n\log n)\) move in opposite directions.
    The logarithmic denominator in panel (d) is the finite-sample proxy for
    \(\log(1+r_0)\), which is asymptotically equivalent in this construction.
    Risk curves are unsmoothed medians with interquartile bands.}
    \label{fig:regression-case-study-experiments}
  \end{figure}
 The bottom row illustrates the logarithmic-gap construction.  Panels (c)--(d)
  display small marginal excess risks while the fused
  representation remains substantially worse over the displayed sample sizes.
  The marginal tasks lie in the logarithmic gap between the known sufficient and
  necessary conditions, producing a slow asymptotic regime.  The persistent
  separation between the marginal and fused curves is a finite-sample
  illustration consistent with the non-closure mechanism in
  Theorem~\ref{thm:regression-task-level-nonclosure}.

  \subsection{The complete classification interaction law}
  \label{subsec:experimental-classification-interactions}

  Figure~\ref{fig:classification-six-experiments} provides finite-sample
  illustrations for all six constructions in
  Theorem~\ref{thm:classification-six-patterns}.  For every input, the
  population quantity \(\Delta\) is computed after deleting that
  input's own signal eigendirection, exactly as required by
  Theorem~\ref{thm:classification-main-characterization}.  The plotted risk is
  then obtained from the fitted survival and contamination, rather than by
  treating \(\Delta\) itself as an empirical loss.

  \begin{figure}[!t]
    \centering
    \includegraphics[width=\textwidth]
    {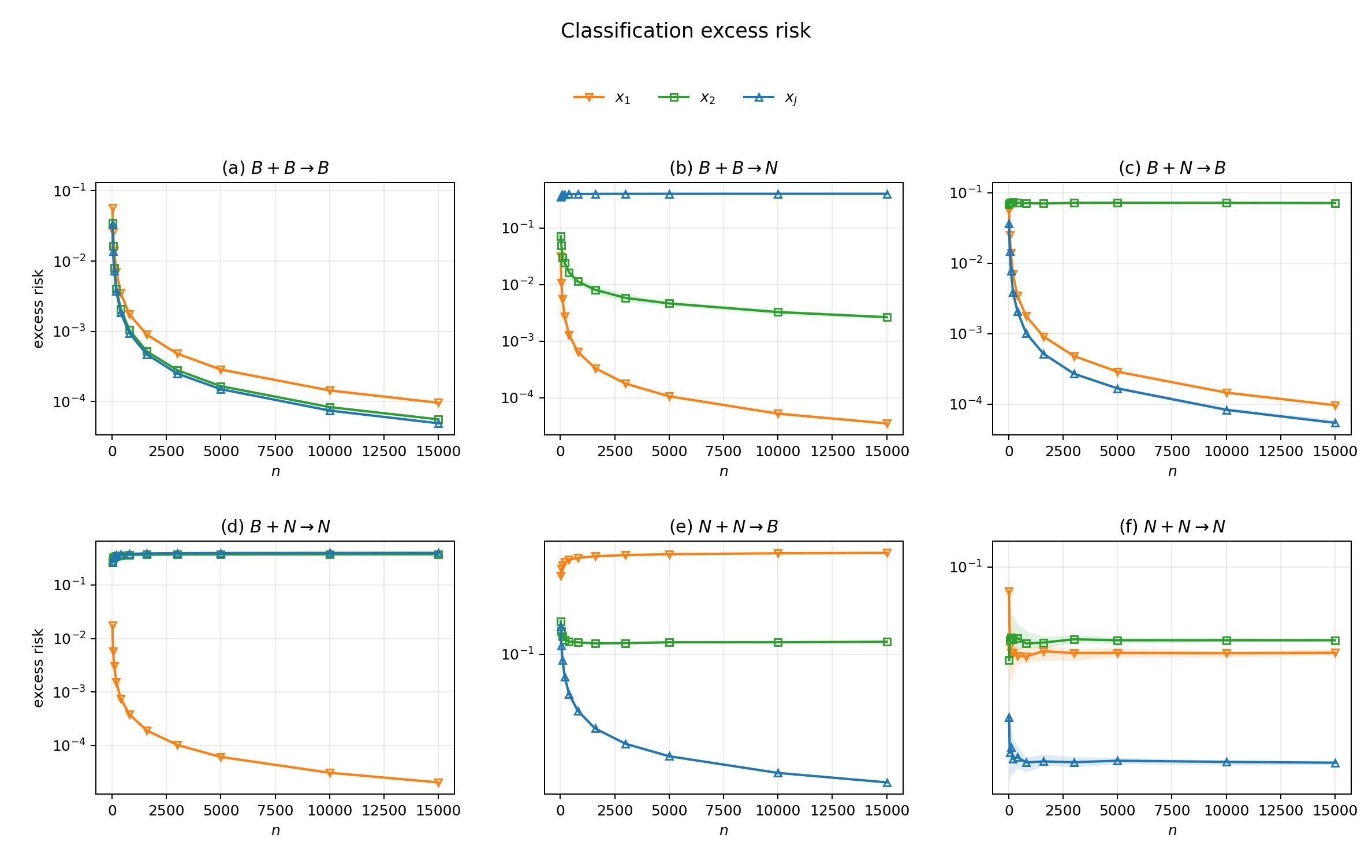}
    \caption{Population classification excess risk for all six unordered
    interaction patterns.  Each construction uses one common latent label
  mechanism and input-specific one-sparse signal directions.  The
    six panels illustrate the absence of a nontrivial Boolean closure law based
    only on the two marginal benign/non-benign states.
    Curves are unsmoothed medians with interquartile bands.}
    \label{fig:classification-six-experiments}
  \end{figure}

  The classification panels illustrate all six possibilities.  Panel (a) has
  three decreasing risks.  In panel (b), both marginal risks decrease while the
  joint risk remains near \(0.40\), consistent with the classification transition
  that is impossible for the regression full-spectrum surrogate.  Panels (c) and
  (d) distinguish the two \(\Bcls+\Ncls\) outcomes: fusion removes the harmful marginal
  effect in (c), but retains it in (d).  Panel (e) illustrates correlated
  rescue from two non-benign sources.  The joint risk decreases more slowly
  than in the other benign examples, consistent with the theoretical
  balance \(\Delta_J\asymp n^{2\ell-1}=n^{-0.6}\) for the experimental choice
  \(\ell=0.2\).  Panel (f) leaves all three risks bounded away from zero.
  Together with the regression panels, these finite-sample patterns are
  consistent with the paper's main distinction: regression is protected by a
  full-spectrum covariance surrogate, whereas classification is governed by the
  input-specific ratio of contamination to surviving signal.

  \subsection{Cross-task separation under the same fusion}
  \label{subsec:experimental-cross-task-separation}

  Figure~\ref{fig:cross-task-separation-experiments} illustrates the two
  constructions of Theorem~\ref{thm:cross-task-separation}.  Within each row,
  the Gaussian inputs and the common latent score \(L_n\) are unchanged: the
  regression response is \(L_n\), while the classification response is
  \(\operatorname{sign}(L_n)\).  In particular, these panels do not combine
  separately tuned regression and classification witnesses.  We use the exact
  Appendix~\ref{app:cross-task-separation-proofs} parameters and normalize its
  designated joint signal direction to unit variance.
  \begin{figure}[!t]
    \centering
    \includegraphics[width=\textwidth]
    {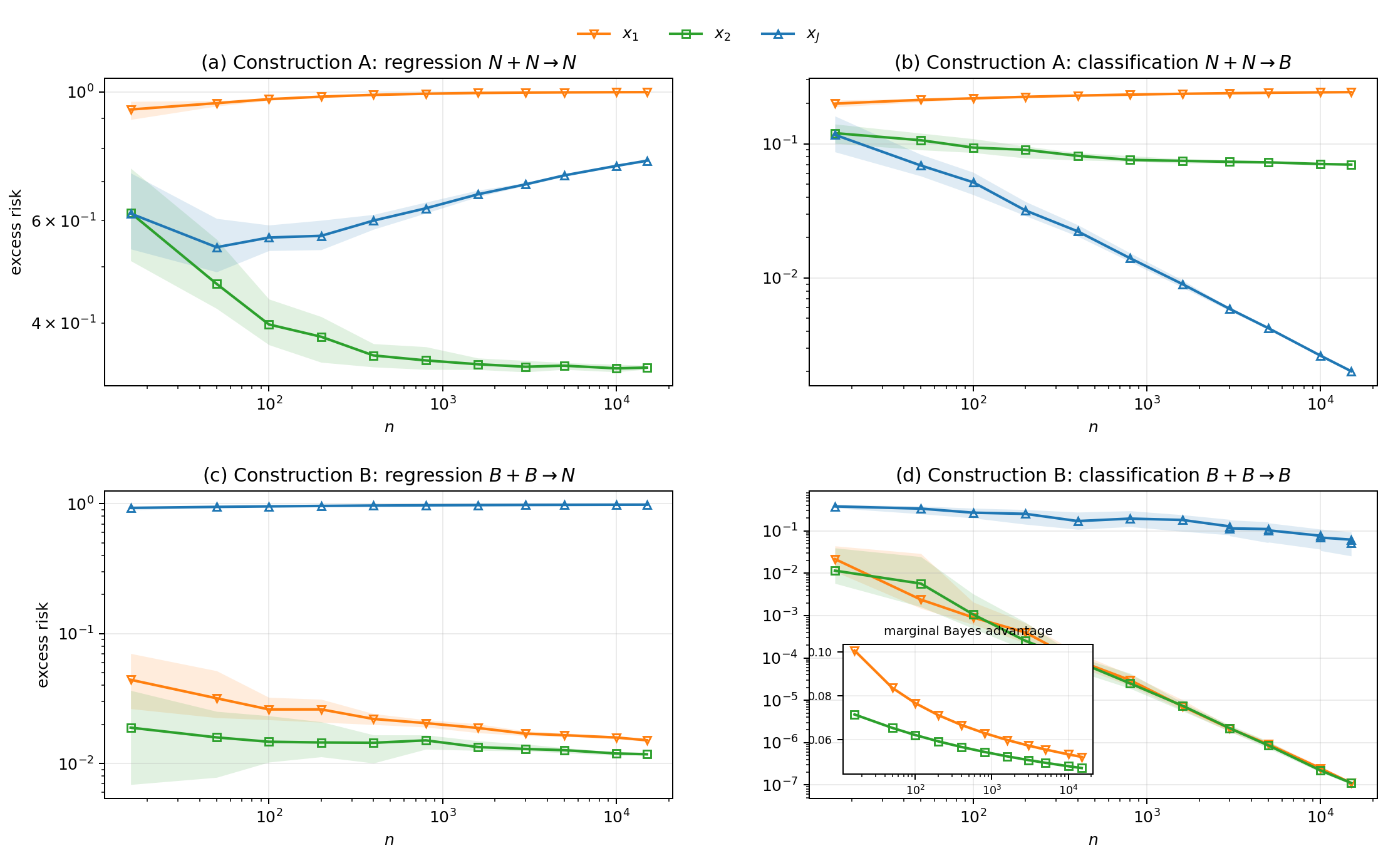}
    \caption{Cross-task separation under the same heterogeneous fusion and
    common latent score.  Panels (a)--(b) use the first construction of
    Theorem~\ref{thm:cross-task-separation}: the displayed finite-sample trends
    are consistent with regression non-benignity and classification rescue.
    Panels (c)--(d) use the second construction: the displayed trends are
    consistent with benign marginal inputs for both tasks, regression degradation
    after fusion, and preserved classification benignity.  The inset reports the vanishing marginal Bayes
    advantages in panel (d).  All synthetic panels extend through
    \(n=15000\).  All points use \(50\) repetitions.  Curves are unsmoothed
    medians with interquartile bands.}
    \label{fig:cross-task-separation-experiments}
  \end{figure}
  The first row illustrates a task-dependent rescue: the regression risks remain
  separated from zero, but the fused classification excess risk decreases while
  the two marginal classification risks remain non-benign.  The second row
  illustrates the complementary separation.  Its displayed risks are consistent
  with persistent regression bias despite decreasing marginal regression risks,
  while all three classification excess risks decrease.  The inset makes explicit that the
  second construction's marginal classification statement concerns excess risk:
  the marginal Bayes advantages themselves vanish, so their raw risks approach
  chance level.  Thus the figure provides a finite-sample view of task dependence
  without changing the input distribution, latent score, or fusion operation.
\begin{figure}[!t]
    \centering
    \includegraphics[width=\textwidth]
    {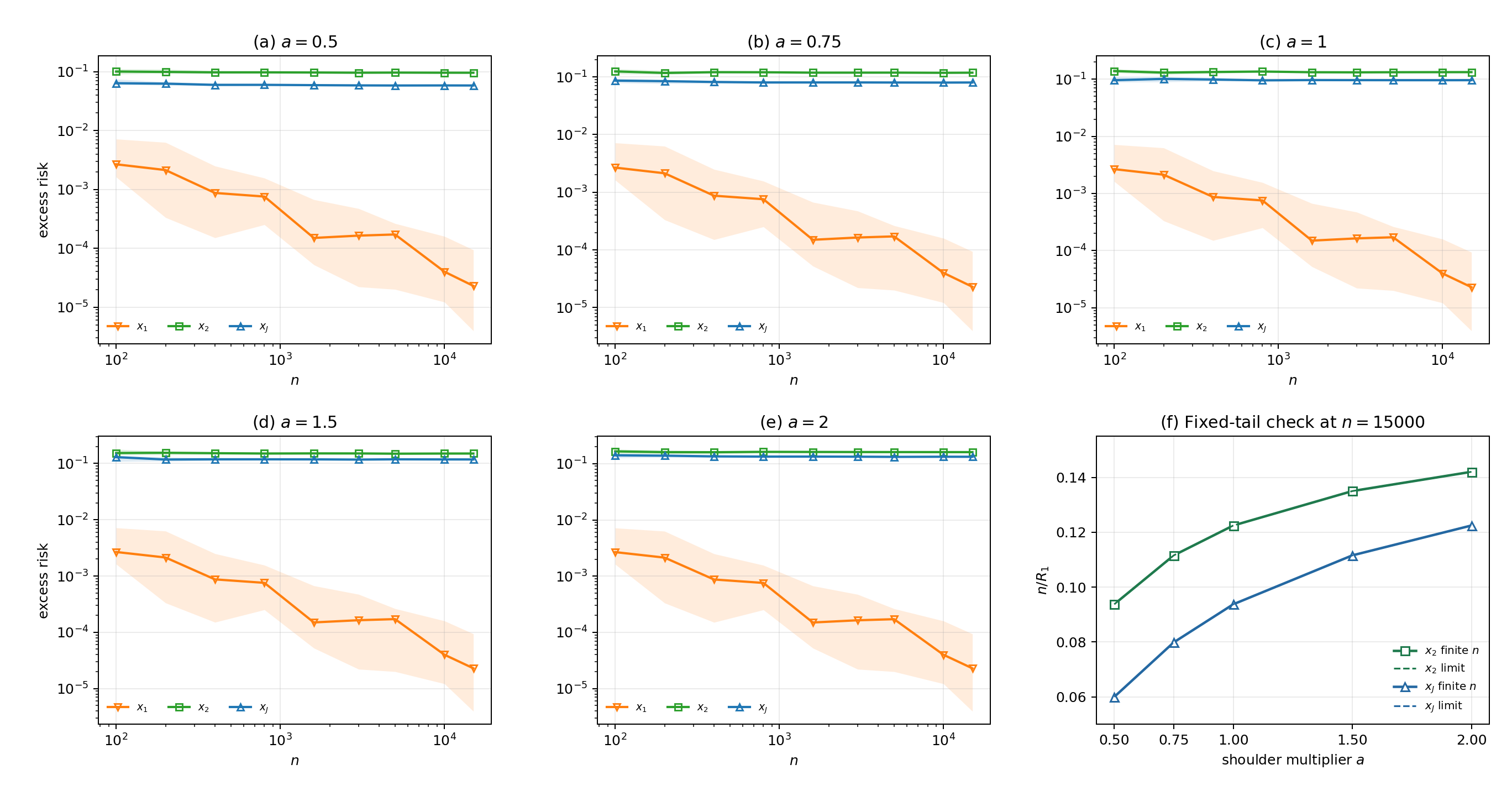}
    \caption{Constant-factor robustness of the small-innovation mechanism.
    Panels (a)--(e) report population regression excess risk for the five fixed
    shoulder multipliers \(a\), with one problem instance per panel; the
    \(X_1\) curve is repeated as a common reference because its law does not
    depend on \(a\).  Panel (f) compares the finite-sample fixed-tail
    quantity at \(n=15000\) against the analytic limits \(6a^2/(1+6a)^2\) for
    \(X_2\) and \(6a^2/(2+6a)^2\) for \(X_J\); dashed curves are analytic, not
    fitted.  Every shaded band is
    the Monte Carlo interquartile range for its own fixed problem instance;
    no band aggregates variation across multipliers.}
    \label{fig:small-energy-robustness-experiments}
  \end{figure}
  \subsection{Robustness of the small-innovation mechanism}
  \label{subsec:experimental-small-innovation-robustness}

  We next test whether the small-innovation obstruction in
  Section~\ref{subsec:regression-small-energy-case-study} depends on its exact
  constant.  Keep \(p=n^3\), \(m=6n\),
  \(D=\operatorname{diag}(1,n^{-4}I_p)\), and \(A=I\), and replace the
  innovation covariance by $\Sigma_U(a)=\operatorname{diag}
    \bigl(n^{-8},a n^{-2}I_m,n^{-8}I_{p-m}\bigr),
    \qquad a\in\{0.5,0.75,1,1.5,2\}$.
  For every fixed \(a>0\), its total energy still vanishes since
  \(\operatorname{tr}(\Sigma_U(a))=6a/n+o(n^{-1})\).  On the other hand,
  the fixed-tail variance quantities have the analytic limits
  \[
    \frac{n}{R_1(\Sigma_2)}\longrightarrow
    \frac{6a^2}{(1+6a)^2},
    \qquad
    \frac{n}{R_1(\Sigma_J)}\longrightarrow
    \frac{6a^2}{(2+6a)^2}.
  \]
  Figure~\ref{fig:small-energy-robustness-experiments} treats each fixed
  multiplier as a separate finite-sample problem instance.  Within each panel,
  the shaded bands quantify only Monte Carlo variation for that instance.
  The panels illustrate persistent second and joint risks throughout the
  displayed constant-factor sweep, while the first-input reference risk
  decreases.  The deterministic spectral curves are close to their corresponding
  analytic limits; at \(a=1\), these reduce to
  \(6/49\) and \(6/64\), respectively.  This is a finite-sample robustness
  illustration consistent with the proved mechanism, not a uniform-in-\(a\)
  extension of the theorem.

\begin{figure}[!tp]
    \centering
    \includegraphics[width=\textwidth]
    {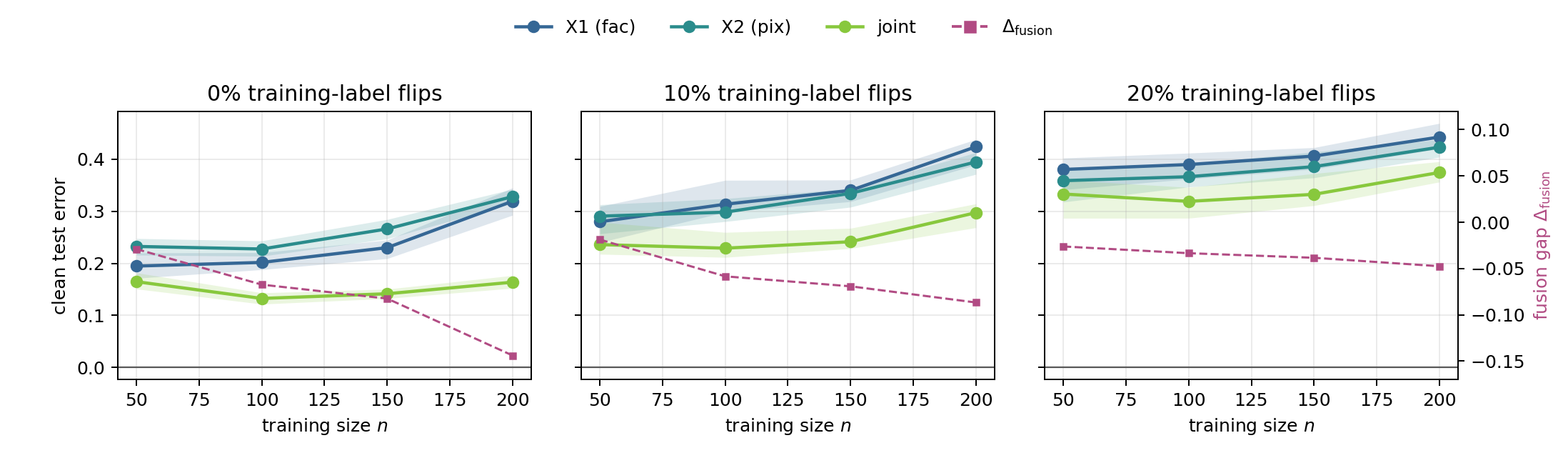}
    \caption{Out-of-model sanity check on the UCI Multiple Features data.
    \(X_1\) is the profile-correlation view, \(X_2\) the pixel-average view,
    and \(X_J\) their concatenation.  Solid curves are median clean-test
    errors and shaded bands are interquartile ranges over 30 full-row-rank
    stratified splits. Dashed magenta curves use the right axis and report
    \(\Delta_{\mathrm{fusion}}\).  Only training labels are flipped.}
    \label{fig:uci-sanity-check}
  \end{figure}

  \subsection{Out-of-model multi-view sanity check}
  \label{subsec:experimental-uci-sanity-check}

  As a deliberately limited practical complement, we also evaluate two aligned
  views from the UCI Multiple Features data set~\citep{duin1998multiple}:
  the 216-dimensional profile-correlation view (\texttt{mfeat-fac}) and the
  240-dimensional pixel-average view (\texttt{mfeat-pix}).  The data contain
  2,000 handwritten digits with corresponding rows across views.  We form the
  balanced binary task digits 0-4 versus 5-9.  For each of 30 stratified
  splits and each \(n\in\{50,100,150,200\}\), every view is centered using
  only its training mean and rescaled by one training-set scalar (no featurewise
  whitening). A common intercept is included in each minimum-norm fit.  We
  retain only full-row-rank splits, so all three representations interpolate the
  possibly flipped training labels to numerical precision.  Test labels are
  always clean. Figure~\ref{fig:uci-sanity-check} reports clean test error together with the
  fusion gap $\Delta_{\mathrm{fusion}}
    =\operatorname{Err}(X_J)-
    \min\{\operatorname{Err}(X_1),\operatorname{Err}(X_2)\}$. We use 0\%, 10\%, and 20\% symmetric flips of training labels.  Fusion has a
  negative median gap throughout this small protocol.  This is an out-of-model
  sanity check: the data are neither Gaussian nor a one-sparse eigensignal
  model, and the result is not presented as validation of a theorem.

\section{Conclusion}
\label{sec:conclusion}
We studied benign overfitting under heterogeneous input fusion while keeping the underlying population task fixed. For regression, we identified a threshold-invariant full-spectrum covariance certificate that is preserved under arbitrary positive-semidefinite joint fusion, while showing that task-level benign regression need not be closed outside this certified regime. For one-sparse Gaussian classification, benignity in the regular regime is governed by the balance between surviving predictive signal and nuisance contamination, and within this model class all marginal-to-joint benign/non-benign patterns are attainable. These results show that benign overfitting under input fusion is fundamentally task-dependent: the same fused input can affect regression and classification differently. Extending these conclusions beyond the Gaussian setting and identifying stronger conditions for task-level regression closure remain important directions for future work.

\clearpage
\onecolumn

\renewcommand{\cftsecfont}{\normalsize} 
\renewcommand{\cftsubsecfont}{\normalsize} 
\renewcommand{\cftbeforesecskip}{12pt}      
\renewcommand{\cftbeforesubsecskip}{12pt}

\renewcommand{\contentsname}{Appendix}
\addtocontents{toc}{\protect\setcounter{tocdepth}{2}}

  \appendix
  \tableofcontents

    \clearpage
\clearpage
\appendix

\clearpage
\section{Proofs for the Regression Results}
\label{app:regression-proofs}

\subsection[Certificates and probabilistic consequences]{Certificates and probabilistic consequences in \citet{bartlett2020benign}}
\label{app:bartlett-probability}

\begin{lemma}[The full-spectrum certificate implies benign regression]
\label{lem:BFS-certifies-regression}
Assume
\[
  \Sigma_n\in\BFS(b_{\mathrm B}),
\]
together with
\[
  \|w_n\|_2^2\|\Sigma_n\|_{\mathrm{op}}=O(1),
  \qquad \tau_n^2=O(1).
\]
Then
\(\mathcal E_n^{\mathrm{reg}}(\widehat w_n)\to_{\mathbb P}0\).
\end{lemma}

\begin{proof}
Let \(a_n=r_0(\Sigma_n)/n\) and
\(v_n=V_{b_{\mathrm B}}(\Sigma_n)\). By assumption,
\(a_n\vee v_n\to0\), and \(k_{b_{\mathrm B}}(\Sigma_n)<n/c_1\) eventually.
Choose a deterministic sequence \(L_n\to\infty\) so slowly that
\[
  L_nv_n\to0,\qquad L_n/n\to0.
\]
For example, after replacing \(v_n\) by \(v_n+n^{-1}\), one may take
\(L_n=\min\{v_n^{-1/2},n^{1/2}\}\). Apply
Theorem~\ref{thm:bartlett-regression-bounds} with
\(\delta_n=e^{-L_n}\). The bias term in (\ref{eq:bartlett-upper}) tends to
zero because \(a_n\to0\) and \(L_n/n\to0\), and its variance term tends to
zero because \(L_nv_n\to0\). Since \(\delta_n\to0\), the claimed convergence
in probability follows.
\end{proof}

\subsubsection{A fixed-probability form of the variance lower bound}
\label{app:fixed-probability-lower}

\begin{lemma}[Fixed-probability variance lower bound]
\label{lem:fixed-probability-variance-lower}
There are universal constants \(c_0,p_0>0\) for which the following holds
under the Gaussian regression model whenever
\(\operatorname{rank}(\Sigma)>n\).  In particular, the Gaussian design has
full row rank almost surely and every inverse below is well defined.
\begin{equation}
  \mathbb P\!\left\{
    \mathcal E_n^{\mathrm{reg}}(\widehat w)
    \geq c_0\tau^2\underline V_n(\Sigma)
  \right\}
  \geq p_0-o(1),
  \label{eq:fixed-probability-variance-lower}
\end{equation}
where
\[
  \underline V_n(\Sigma)
  :=
  \begin{cases}
    1,&k_{b_{\mathrm B}}(\Sigma)\geq n/c_1,\\
    V_{b_{\mathrm B}}(\Sigma),
      &k_{b_{\mathrm B}}(\Sigma)<n/c_1.
  \end{cases}
\]
Consequently, if \(\inf_n\tau_n^2>0\) and
\(\limsup_n\underline V_n(\Sigma_n)>0\), then regression benign overfitting
fails along a subsequence.  If the latter condition is strengthened to
\(\liminf_n\underline V_n(\Sigma_n)>0\), the lower bound holds persistently
along the full sequence and the regression problem is harmful according to
Definition~\ref{def:regression-benign-harmful}.
\end{lemma}

\begin{proof}
Let \(\mathbf X\) be the \(n\times d\) design matrix and put
\[
  G_{\mathbf X}
  :=\Sigma^{1/2}\mathbf X^\top
    (\mathbf X\mathbf X^\top)^{-1}\varepsilon.
\]
Conditional on \(\mathbf X\), this is a centered Gaussian vector. Its squared
Euclidean norm is the prediction error contributed by the observation noise,
and
\begin{equation}
  \mathbb E\!\left[\|G_{\mathbf X}\|_2^2\mid\mathbf X\right]
  =\tau^2\operatorname{tr}\!\left[
    (\mathbf X\mathbf X^\top)^{-1}
    \mathbf X\Sigma\mathbf X^\top
    (\mathbf X\mathbf X^\top)^{-1}
  \right].
  \label{eq:conditional-noise-trace}
\end{equation}
For every centered Gaussian vector \(G\), diagonalizing its conditional
covariance and using the Gaussian fourth moment gives
\(\mathbb E\|G\|_2^4\leq3(\mathbb E\|G\|_2^2)^2\).  The
Paley--Zygmund inequality \citep[Lemma~19]{paley1932some}, applied to
\(W=\|G_{\mathbf X}\|_2^2\) with threshold parameter \(1/2\), therefore
gives
\[
  \mathbb P\!\left\{
    \|G_{\mathbf X}\|_2^2
    \geq \tfrac12
    \mathbb E[\|G_{\mathbf X}\|_2^2\mid\mathbf X]
    \,\middle|\,\mathbf X
  \right\}\geq\frac1{12}.
\]
Conditional on the design, the full prediction error has the form
\(\|h_{\mathbf X}+G_{\mathbf X}\|_2^2\), where
\(h_{\mathbf X}\) is deterministic. Anderson's inequality for centered
Gaussian measures \citep{anderson1955integral} implies that translating a centered Euclidean ball cannot
increase its probability. Therefore, the same lower-tail estimate, with the
same threshold, holds for the full prediction error.

The design-dependent trace in (\ref{eq:conditional-noise-trace}) is exactly
the variance trace in \citet{bartlett2020benign}. To make the two regimes in
their lower bound explicit, write \(k^\star=k_{b_{\mathrm B}}(\Sigma)\) and
enlarge \(c_1\), if necessary, so that it dominates the constants in their
Lemmas~16--17. If \(k^\star<n/c_1\), apply part 2 of their Lemma~16 at
\(k=k^\star\); their Lemma~17 identifies its minimum over \(l\leq k^\star\)
with a constant multiple of
\(k^\star/n+n/R_{k^\star}(\Sigma)\). Thus the trace is at least
\(cV_{b_{\mathrm B}}(\Sigma)\) with probability \(1-o(1)\).

If \(k^\star\geq n/c_1\), take
\(k_0=\lfloor n/(2c_1)\rfloor<k^\star\). By minimality of \(k^\star\),
\(r_{k_0}(\Sigma)<b_{\mathrm B}n\). Part 1 of their Lemma~16 then gives
\[
  \operatorname{tr}\!\left[
    (\mathbf X\mathbf X^\top)^{-1}
    \mathbf X\Sigma\mathbf X^\top
    (\mathbf X\mathbf X^\top)^{-1}
  \right]
  \geq \frac{k_0+1}{c b_{\mathrm B}^2n}
  \geq c'>0
\]
with probability \(1-o(1)\). This proves the design-side lower bound in both
branches. Integrating the conditional probability \(1/12\) over that event
proves (\ref{eq:fixed-probability-variance-lower}).
\end{proof}

\subsection{Threshold invariance of the full-spectrum state}
\label{app:threshold-invariance}

\begin{proof}[Proof of Lemma~\ref{lem:regression-threshold-invariance}]
The condition \(r_0(\Sigma_n)/n\to0\) is independent of the threshold. We
therefore transfer only the cutoff and post-cutoff effective-rank conditions.
Suppress the argument \(\Sigma_n\), write \(k_i=k_{b_i}(\Sigma_n)\), and
define
\[
  S_i=S_{k_i}(\Sigma_n),
  \qquad
  Q_i=\sum_{j>k_i}\lambda_j(\Sigma_n)^2.
\]
Assume first that \(\Sigma_n\in\BFS(b_1)\).

If \(b_2<b_1\), then \(k_2\leq k_1=o(n)\). Since
\(r_{k_2}\geq b_2n\),
\(\lambda_{k_2+1}\leq S_2/(b_2n)\). Consequently,
\[
  Q_2
  \leq Q_1+(k_1-k_2)\frac{S_2^2}{b_2^2n^2}.
\]
Now \(Q_1=o(S_1^2/n)\), \(S_1\leq S_2\), and \(k_1-k_2=o(n)\), so the
right-hand side is \(o(S_2^2/n)\). Hence \(R_{k_2}=S_2^2/Q_2\gg n\), and
\(\Sigma_n\in\BFS(b_2)\).

Suppose next that \(b_2>b_1\). At \(k=k_1\), abbreviate
\[
  S=S_k,\qquad Q=\sum_{j>k}\lambda_j^2,\qquad R=S^2/Q.
\]
We have \(k=o(n)\) and \(R/n\to\infty\). Let
\(a=S/(2b_2n)\), and let \(m\) be the number of tail eigenvalues strictly
larger than \(a\). Then
\[
  m a^2\leq Q,
  \qquad
  m\leq\frac{4b_2^2n^2}{R}=o(n).
\]
By Cauchy--Schwarz, the total mass of those \(m\) eigenvalues is at most
\[
  \sqrt{mQ}
  \leq \frac{2b_2nS}{R}
  =o(S).
\]
Set \(\ell=k+m=o(n)\). The remaining tail has
\(S_\ell=S(1-o(1))\), its largest eigenvalue is at most \(a\), and hence
\[
  r_\ell\geq2b_2n(1-o(1))\geq b_2n,
  \qquad
  R_\ell\geq \frac{S_\ell^2}{Q}\gg n.
\]
It follows that \(k_2\leq\ell=o(n)\). To transfer the effective-rank bound
from \(\ell\) back to the first qualifying index \(k_2\), use
\(\lambda_{k_2+1}\leq S_{k_2}/(b_2n)\):
\[
  \sum_{j>k_2}\lambda_j^2
  \leq
  \sum_{j>\ell}\lambda_j^2
  +(\ell-k_2)\frac{S_{k_2}^2}{b_2^2n^2}
  =o(S_{k_2}^2/n).
\]
Thus \(R_{k_2}\gg n\), proving
\(\BFS(b_1)\Rightarrow
\BFS(b_2)\). Interchanging \(b_1\) and \(b_2\) gives the
reverse implication.
\end{proof}

\subsection{Fusion closure}
\label{app:fusion-closure-proof}

We first record the block-matrix fact used in the proof.

\begin{lemma}[Positive-semidefinite block control]
\label{lem:psd-block-control}
If
\(
M=\left(\begin{smallmatrix}A&C\\C^\top&B\end{smallmatrix}\right)\succeq0
\),
then
\[
  \|M\|_{\mathrm{op}}
  \leq\|A\|_{\mathrm{op}}+\|B\|_{\mathrm{op}},
  \qquad
  \|M\|_{\mathrm F}^2
  \leq2\bigl(\|A\|_{\mathrm F}^2+\|B\|_{\mathrm F}^2\bigr).
\]
\end{lemma}

\begin{proof}
The factorization theorem for positive-semidefinite block matrices
\citep[Chapter~7]{hornjohnson2013matrix} gives, for
some contraction \(K\),
\[
  C=A^{1/2}KB^{1/2}.
\]
Therefore \(\|C\|_{\mathrm F}^2
\leq\|A\|_{\mathrm F}\|B\|_{\mathrm F}\), and
\[
  \|M\|_{\mathrm F}^2
  =\|A\|_{\mathrm F}^2+\|B\|_{\mathrm F}^2+2\|C\|_{\mathrm F}^2
  \leq2(\|A\|_{\mathrm F}^2+\|B\|_{\mathrm F}^2).
\]
For \(x,y\) of compatible dimensions, positivity and Cauchy--Schwarz for the
semi-inner product induced by \(M\) imply
\[
  |x^\top Cy|
  \leq\sqrt{x^\top Ax}\sqrt{y^\top By}.
\]
Thus the quadratic form of \(M\) is at most
\[
  \left(
    \sqrt{\|A\|_{\mathrm{op}}}\|x\|
    +\sqrt{\|B\|_{\mathrm{op}}}\|y\|
  \right)^2
  \leq
  \bigl(\|A\|_{\mathrm{op}}+\|B\|_{\mathrm{op}}\bigr)
  (\|x\|^2+\|y\|^2),
\]
which proves the operator-norm bound.
\end{proof}

\begin{proof}[Proof of Theorem~\ref{thm:regression-fusion-closure}]
For \(i\in\{1,2\}\), let \(k_i=k_b(\Sigma_{i,n})=o(n)\), and let \(P_i\)
project onto the marginal eigenspace after the first \(k_i\) directions.
Define the marginal tails
\[
  T_i=P_i\Sigma_{i,n}P_i,\qquad
  S_i=\operatorname{tr}(T_i),\qquad
  \mu_i=\|T_i\|_{\mathrm{op}}.
\]
The marginal certificates give
\begin{equation}
  S_i/\mu_i\geq bn,
  \qquad
  \|T_i\|_{\mathrm F}^2=o(S_i^2/n).
  \label{eq:appendix-marginal-tail}
\end{equation}
Compress \(\Sigma_{J,n}\) to
\(\operatorname{range}(P_1)\oplus\operatorname{range}(P_2)\):
\[
  T=
  \begin{pmatrix}
    T_1&\widetilde C\\
    \widetilde C^\top&T_2
  \end{pmatrix}\succeq0.
\]
By Lemma~\ref{lem:psd-block-control},
\begin{align}
  \operatorname{tr}(T)&=S_1+S_2,\notag\\
  \|T\|_{\mathrm{op}}&\leq\mu_1+\mu_2,\notag\\
  \|T\|_{\mathrm F}^2
  &\leq2\sum_{i=1}^2\|T_i\|_{\mathrm F}^2
  =o\!\left((S_1+S_2)^2/n\right).
  \label{eq:appendix-compressed-tail}
\end{align}
In particular, \(r_0(T)\geq bn\) and \(R_0(T)/n\to\infty\).

Let \(m=k_1+k_2=o(n)\), and write the decreasing eigenvalues of
\(\Sigma_{J,n}\) and \(T\) as \(\lambda_j\) and \(\alpha_j\). Since
\(R_0(T)\leq\operatorname{rank}(T)\), the rank of \(T\) is larger than \(m\)
eventually. Poincar\'e separation for a compression of codimension \(m\)
\citep[Section~4.3]{hornjohnson2013matrix}
yields
\begin{equation}
  \lambda_j\geq\alpha_j\geq\lambda_{j+m}.
  \label{eq:poincare-separation-fusion}
\end{equation}
It follows that
\begin{align}
  S_m(\Sigma_{J,n})
  &\geq\operatorname{tr}(T)-m\alpha_1,\notag\\
  \sum_{j>m}\lambda_j^2
  &\leq\sum_j\alpha_j^2=\|T\|_{\mathrm F}^2,\notag\\
  \lambda_{m+1}&\leq\alpha_1.
  \label{eq:appendix-tail-transfer}
\end{align}
Since \(m=o(n)\) and
\(\operatorname{tr}(T)/\alpha_1\geq bn\),
\(m\alpha_1=o(\operatorname{tr}(T))\). Hence
\[
  r_m(\Sigma_{J,n})\geq bn(1-o(1)),
  \qquad
  R_m(\Sigma_{J,n})/n\to\infty.
\]
For any fixed \(b'<b/4\), this shows that
\(\ell:=k_{b'}(\Sigma_{J,n})\leq m=o(n)\) eventually. Moreover,
\(\lambda_{\ell+1}\leq S_\ell/(b'n)\), so
\[
  \sum_{j>\ell}\lambda_j^2
  \leq
  \sum_{j>m}\lambda_j^2
  +(m-\ell)\frac{S_\ell^2}{b'^2n^2}
  =o(S_\ell^2/n).
\]
For the first term, (\ref{eq:appendix-compressed-tail}) and
(\ref{eq:appendix-tail-transfer}) give
\[
  \sum_{j>m}\lambda_j^2
  =o\!\left(\frac{\operatorname{tr}(T)^2}{n}\right)
  =o\!\left(\frac{S_\ell^2}{n}\right),
\]
because \(S_\ell\geq S_m\geq
(1-o(1))\operatorname{tr}(T)\). The added finite-block term is also
\(o(S_\ell^2/n)\) because \(m-\ell=o(n)\). Thus
\(R_\ell/n\to\infty\).

It remains to verify the uncut effective rank. Since each marginal is a
principal submatrix of the joint covariance,
\(\|\Sigma_{J,n}\|_{\mathrm{op}}\geq
\max_i\|\Sigma_{i,n}\|_{\mathrm{op}}\), while
\(\operatorname{tr}(\Sigma_{J,n})
=\operatorname{tr}(\Sigma_{1,n})+\operatorname{tr}(\Sigma_{2,n})\).
Therefore
\[
  r_0(\Sigma_{J,n})
  \leq r_0(\Sigma_{1,n})+r_0(\Sigma_{2,n})
  =o(n).
\]
All three conditions defining \(\BFS(b')\) follow.
\end{proof}

\medskip
\noindent\textbf{Proof of
Corollary~\ref{cor:regression-threshold-free-closure}.}\par\smallskip
\begin{proof}
Choose any fixed \(b>4b_{\mathrm B}\).  By
Lemma~\ref{lem:regression-threshold-invariance}, each marginal
\(\BFS\) state also belongs to
\(\BFS(b)\).  Apply
Theorem~\ref{thm:regression-fusion-closure} with
\(b'=b_{\mathrm B}<b/4\).  The joint covariance then belongs to
\(\BFS(b_{\mathrm B})
=\BFS\), proving threshold-free closure.  In particular,
\(\BFS+\BFS\to
\NFS\) is impossible.  Repeating the same two-source
argument proves closure for any fixed number of sources.
\end{proof}

\subsection{Five attainable regression interaction patterns}
\label{app:transition-witnesses}

\begin{proof}
Fix \(c>2b_{\mathrm B}\), let
\[
  p=p_n=n^3,\qquad m=m_n=\lfloor cn\rfloor,\qquad
  \epsilon=n^{-3},\qquad \rho=n^{-8}.
\]
All asymptotic comparisons below are unchanged by integer rounding.

\medskip
\noindent\textbf{Step 1: spectral design.}\par\smallskip
There are six unordered Boolean patterns.  The transition
\(\BFS+\BFS\to
\NFS\) is ruled out by
Corollary~\ref{cor:regression-threshold-free-closure}. Therefore, only five
witnesses must be constructed.  The common strategy is:
\[
  \begin{aligned}
    (D,Q)&:\ \text{choose the target marginal spectra};\\
    D\oplus Q&:\ \text{compute the desired joint state};\\
    \Sigma_J&:\ \text{introduce dense dependence while preserving that state}.
  \end{aligned}
\]
The first stage uses three reusable spectral components:
\begin{itemize}
  \item \(p=n^3\) comparable tail eigenvalues give
  \(R=\Theta(n^3)\), and hence a wide benign tail;
  \item \(m=\Theta(n)\) comparable shoulder eigenvalues give only
  \(R=\Theta(n)\), which violates the third full-spectrum condition;
  \item a flat \(p\)-dimensional spectrum gives \(r_0=p=\Theta(n^3)\),
  which violates the first condition; this is only a bias-side obstruction
  and therefore requires an explicit signal before it implies task-level
  non-benignity.
\end{itemize}
Indeed, after removing a fixed spectral head, a tail
\(a^{[M]}\uplus z^{[P-M]}\) has
\begin{equation}
  S=Ma+(P-M)z,\qquad
  \mathcal Q=Ma^2+(P-M)z^2,\qquad
  R=\frac{S^2}{\mathcal Q}.
  \label{eq:regression-witness-master-calculation}
\end{equation}
For a direct sum, the tail masses \(S\) and squared masses
\(\mathcal Q\) add.  Thus the two rescuing cases add a sufficiently wide
tail to repair the defective marginal condition, while the harmful cases
preserve either the \(O(n)\)-rank shoulder or the \(r_0\)-obstruction.
Notice that two variance-side shoulders cannot rescue one another in this
direct-sum-comparable family. For arbitrary positive tails,
\[
  R(D\oplus Q)
  =\frac{(S_D+S_Q)^2}{\mathcal Q_D+\mathcal Q_Q}
  \leq R(D)+R(Q).
\]
Consequently, if both marginal tail effective ranks are \(O(n)\), the joint
tail effective rank is still \(O(n)\).  This is why the
\(\NFS+\NFS\to\BFS\) witness combines one variance-side defect with one explicit
bias-side defect instead of pretending that failure of \(r_0/n\to0\)
universally implies non-benign risk.

\medskip
\noindent\textbf{Step 2: nondegenerate correlated lifts.}\par\smallskip
Cases 1--3 and 5 use the following common dense diagonal lift.  Case 4 uses
a signal-aligned diagonal lift, specified in that case, because its
\(r_0\)-obstructed marginal must be paired with an explicit hard signal.
Let \(D=\operatorname{diag}(d_1,\ldots,d_P)\) and
\(Q=\operatorname{diag}(q_1,\ldots,q_P)\) be the target marginal
covariances, with all \(d_j,q_j>0\), and define
\begin{equation}
  L:=\sum_{j=1}^P\frac{q_j}{d_j},\qquad
  \chi:=L^{-1/2},\qquad
  a_j:=\chi\sqrt{\frac{q_j}{d_j}}.
  \label{eq:regression-dense-coupling}
\end{equation}
For all the spectra below, \(L\to\infty\). Hence \(0<\chi<1\) eventually,
and we may set
\begin{equation}
  A:=\operatorname{diag}(a_1,\ldots,a_P),
  \qquad
  \Sigma_U:=(1-\chi^2)Q.
  \label{eq:regression-dense-innovation}
\end{equation}
Every diagonal entry of \(A\) is strictly positive and
\(\|A\|_{\mathrm F}^2=\sum_j a_j^2=1\). Moreover,
\(ADA^\top+\Sigma_U=Q\), so the marginal covariances are exactly \(D\) and
\(Q\). After a simultaneous permutation of coordinates, the joint
covariance is the direct sum of the blocks
\[
  M_j=
  \begin{pmatrix}
    d_j&\chi\sqrt{d_jq_j}\\
    \chi\sqrt{d_jq_j}&q_j
  \end{pmatrix}.
\]
Congruencing each block by
\(\operatorname{diag}(d_j^{-1/2},q_j^{-1/2})\) gives the correlation block
\(\left(\begin{smallmatrix}1&\chi\\ \chi&1\end{smallmatrix}\right)\), and
hence
\begin{equation}
  (1-\chi)(D\oplus Q)
  \preceq \Sigma_J
  \preceq (1+\chi)(D\oplus Q).
  \label{eq:regression-dense-joint-comparison}
\end{equation}
Lemma~\ref{lem:uniform-spectral-comparison} therefore transfers the
full-spectrum state of the direct sum to this genuinely correlated joint
covariance. Since \(\chi=o(1)\) in every row, the relevant orders of
\(r_0\), the cutoff, and \(R_k\) are preserved as well.

\medskip
\noindent\textbf{Step 3: five explicit substitutions.}\par\smallskip
We now verify the five attainable rows.  In Cases 1-3 and 5, the marginal
state is read directly from
(\ref{eq:regression-witness-master-calculation}), the joint state is first
computed for \(D\oplus Q\), and
(\ref{eq:regression-dense-joint-comparison}) transfers that state to the
dense correlated covariance \(\Sigma_J\).  We also give a common response
for every row and verify the resulting task-level regression states.

\medskip
\noindent\textbf{Case 1:
\(\BFS+\BFS
\to\BFS\).}\par\smallskip
\noindent\emph{Target spectra.}
Take \(D=Q=\operatorname{diag}(1,\epsilon I_p)\).

\noindent\emph{Marginal checks.}
For each marginal, \(r_0=2\), \(k_b=1=o(n)\), and
\(R_1=p=n^3\), so both marginals are in \(\BFS\).

\noindent\emph{Joint check.}
The direct sum has two leading unit eigenvalues and
\(k_b(D\oplus Q)=2\), while \(R_2(D\oplus Q)=2p\).

\noindent\emph{Conclusion.}
The dense lift satisfies
\(\Sigma_J\in\BFS\).  Taking the common response to be
\(Y=\xi\), where \(\xi\sim N(0,\tau^2)\) and \(0<\tau^2<\infty\),
Lemma~\ref{lem:BFS-certifies-regression} gives benign regression for both
marginals and the joint input.  Hence this is also a task-level
\(\Breg+\Breg\to\Breg\) witness.

\medskip
\noindent\textbf{Case 2:
\(\BFS+\NFS
\to\BFS\).}\par\smallskip
\noindent\emph{Target spectra.}
Take
\[
  D=\operatorname{diag}(1,\epsilon I_p),\qquad
  Q=\operatorname{diag}(1,n^{-2}I_m,\rho I_{p-m}).
\]

\noindent\emph{Marginal checks.}
As in Case 1, \(D\in\BFS\).  For \(Q\),
\(r_0(Q)=1+o(1)\), \(k_b(Q)=1\), and
\[
  S_1(Q)=c/n+o(n^{-1}),\qquad R_1(Q)=cn(1+o(1)),
\]
so the third condition fails and \(Q\in\NFS\).

\noindent\emph{Joint check.}
After removing the two leading unit eigenvalues,
\[
  S_2(D\oplus Q)=1+o(1),\qquad
  \mathcal Q_2(D\oplus Q)=(c+1)n^{-3}(1+o(1)).
\]
Hence \(k_b(D\oplus Q)=2=o(n)\) and
\(R_2(D\oplus Q)=\Theta(n^3)\).
The comparison (\ref{eq:regression-dense-joint-comparison}) is strong enough
here to locate, rather than merely certify, the joint cutoff: the first two
eigenvalues remain \(\Theta(1)\), so \(r_0,r_1=O(1)\), whereas after them
the largest tail eigenvalue is \(\Theta(n^{-2})\) and the tail mass is
\(\Theta(1)\), so \(r_2=\Theta(n^2)\). Hence
\(k_{b_{\mathrm B}}(\Sigma_J)=2\) and
\(V_{b_{\mathrm B}}(\Sigma_J)=O(n^{-1})\).

\noindent\emph{Conclusion.}
The wide tail of \(D\) repairs the \(R=\Theta(n)\) shoulder of \(Q\), and
\(\Sigma_J\in\BFS\).  With the same pure-noise response
\(Y=\xi\), Lemma~\ref{lem:BFS-certifies-regression} gives benignity for
\(D\) and \(\Sigma_J\).  Since
\(V_{b_{\mathrm B}}(Q)=n/R_1(Q)+o(1)\to1/c\),
Lemma~\ref{lem:fixed-probability-variance-lower} rules out benignity for
\(Q\).  This realizes
\(\Breg+\Nreg\to\Breg\).

\medskip
\noindent\textbf{Case 3:
\(\BFS+\NFS
\to\NFS\).}\par\smallskip
\noindent\emph{Target spectra.}
Set \(\eta=n^{-4}\) and take
\[
  D=\operatorname{diag}(1,\eta I_p),\qquad
  Q=\operatorname{diag}(1,n^{-2}I_m,\eta I_{p-m}).
\]

\noindent\emph{Marginal checks.}
For \(D\), \(k_b(D)=1\) and \(R_1(D)=p=n^3\), so
\(D\in\BFS\).  For \(Q\),
\[
  S_1(Q)=\frac{c+1}{n}(1+o(1)),\qquad
  \mathcal Q_1(Q)=\frac{c}{n^3}(1+o(1)),
\]
and therefore \(R_1(Q)=\Theta(n)\); the third condition fails.

\noindent\emph{Joint check.}
After removing the two leading unit eigenvalues,
\[
  S_2(D\oplus Q)=\frac{c+2}{n}(1+o(1)),\qquad
  \mathcal Q_2(D\oplus Q)=\frac{c}{n^3}(1+o(1)),
\]
so \(R_2(D\oplus Q)=\Theta(n)\).
Under (\ref{eq:regression-dense-joint-comparison}), the two leading
eigenvalues remain \(\Theta(1)\), giving \(r_0,r_1=O(1)\), while at
\(k=2\) the tail mass is \(\Theta(n^{-1})\), its largest eigenvalue is
\(\Theta(n^{-2})\), and
\(r_2=(c+2+o(1))n>b_{\mathrm B}n\). Thus the actual joint cutoff is two and
\(V_{b_{\mathrm B}}(\Sigma_J)=\Theta(1)\), not merely an unspecified
quantity associated with the same \(\NFS\) state.

\noindent\emph{Conclusion.}
The \(D\)-tail now has too little mass to repair the shoulder; the third
condition remains false and
\(\Sigma_J\in\NFS\).  For \(Y=\xi\), the first marginal
is benign by Lemma~\ref{lem:BFS-certifies-regression}.  Both \(Q\) and
\(\Sigma_J\) have \(V_{b_{\mathrm B}}=\Theta(1)\), so
Lemma~\ref{lem:fixed-probability-variance-lower} makes their prediction
risks non-benign.  Thus the same construction realizes
\(\Breg+\Nreg\to\Nreg\).

\medskip
\noindent\textbf{Case 4:
\(\NFS+\NFS
\to\BFS\).}\par\smallskip
\noindent\emph{Target spectra.}
Take
\[
  D=\operatorname{diag}(1,\epsilon I_m,\rho I_{p-m}),
  \qquad Q=\epsilon I_{p+1}.
\]

For this case set
\begin{equation}
  A=\operatorname{diag}\!\left(
    \sqrt{\epsilon-\rho},\sqrt\epsilon I_p
  \right),
  \qquad
  \Sigma_U=\operatorname{diag}\!\left(
    \rho,(\epsilon-\epsilon^2)I_m,
    (\epsilon-\epsilon\rho)I_{p-m}
  \right).
  \label{eq:regression-bias-rescue-coupling}
\end{equation}
All diagonal entries of \(A\) are positive,
\(\|A\|_{\mathrm F}^2=1+\epsilon-\rho=1+o(1)\), and
\(ADA^\top+\Sigma_U=Q\).

\noindent\emph{Marginal checks.}
For \(D\), \(r_0(D)=1+o(1)\), \(k_b(D)=1\), and
\(R_1(D)=cn(1+o(1))\), so the third condition fails.  For \(Q\),
\(r_0(Q)=p=n^3\), so the first condition fails.

\noindent\emph{Joint check.}
The signal-coordinate block of \(\Sigma_J\) is
\begin{equation}
  G_n=
  \begin{pmatrix}
    1&\sqrt{\epsilon-\rho}\\
    \sqrt{\epsilon-\rho}&\epsilon
  \end{pmatrix}.
  \label{eq:regression-bias-rescue-signal-block}
\end{equation}
Its eigenvalues satisfy \(\lambda_+(G_n)=1+o(1)\) and
\(\lambda_-(G_n)=\rho(1+o(1))\).  Each shoulder block has eigenvalues
\(\epsilon(1\pm\sqrt\epsilon)\), while every remaining nuisance block has
one eigenvalue \(\epsilon(1+o(1))\) and one eigenvalue
\(\rho(1+o(1))\).  Hence, after deleting \(\lambda_+(G_n)\),
\[
  S_1(\Sigma_J)=1+o(1),\qquad
  \sum_{j>1}\lambda_j(\Sigma_J)^2=n^{-3}(1+o(1)).
\]
It follows that \(k_{b_{\mathrm B}}(\Sigma_J)=1\) and
\(R_1(\Sigma_J)=n^3(1+o(1))\), so
\(\Sigma_J\in\BFS\).

\noindent\emph{A common hard signal for the two marginals.}
Let \(v_n=(v_{1,n},v_{2,n})\) be a unit leading eigenvector of \(G_n\),
chosen with positive entries, and put
\begin{equation}
  \gamma_n=v_{2,n}e_1,\qquad
  \beta_n=
  \bigl(v_{1,n}+\sqrt{\epsilon-\rho}\,v_{2,n}\bigr)e_1.
  \label{eq:regression-bias-rescue-signal}
\end{equation}
For the heterogeneous model \(X_2=AX_1+U\), take the single response
\[
  Y=\beta_n^\top X_1+\gamma_n^\top U+\xi,
  \qquad \xi\sim N(0,\tau^2),\quad 0<\tau^2<\infty.
\]
Its coefficient in the fused coordinates is
\((\beta_n-A^\top\gamma_n,\gamma_n)=(v_{1,n}e_1,v_{2,n}e_1)\).
It is therefore a unit vector in the leading eigendirection of
\(\Sigma_J\).  Lemma~\ref{lem:BFS-certifies-regression} applies and yields
benign regression for the fused input.

For the first marginal, the conditional noise variance is at least
\(\tau^2\), while \(R_1(D)=cn(1+o(1))\).  Thus
Lemma~\ref{lem:fixed-probability-variance-lower} shows that its prediction
risk does not converge to zero.

For the second marginal, \(Q=\epsilon I_P\) with \(P=p+1\).  The Gaussian
conditional-mean formula \citep[Chapter~2]{anderson2003multivariate} gives
the population coefficient
\[
  w_{2,n}=\epsilon^{-1}\lambda_+(G_n)v_{2,n}e_1,
  \qquad
  \epsilon\|w_{2,n}\|_2^2\longrightarrow1.
\]
Let \(\Pi_{2,n}\) be the Euclidean projection onto the row space of the
second-marginal training design.  Because \(Q\) is isotropic, the null-space
bias and the fitted-noise component are orthogonal also in prediction norm;
hence, deterministically conditional on the design,
\begin{equation}
  \mathcal E_n^{\mathrm{reg}}(\widehat w_{2,n})
  \geq
  \epsilon\|(I-\Pi_{2,n})w_{2,n}\|_2^2.
  \label{eq:regression-bias-rescue-nullspace-lower}
\end{equation}
To justify the random-subspace step explicitly, write the isotropic design
as \(\sqrt\epsilon Z\), where \(Z\) has iid standard Gaussian entries.
Right-orthogonal invariance of the multivariate normal law
\citep[Chapter~2]{anderson2003multivariate} makes its row space uniform on
the Grassmannian.  Equivalently,
\(U^\top\mathbb E[\Pi_{2,n}]U=\mathbb E[\Pi_{2,n}]\) for every orthogonal
\(U\), so \(\mathbb E[\Pi_{2,n}]=aI_P\).  Taking traces gives
\(a=n/P\), and hence
\(\mathbb E\|\Pi_{2,n}w_{2,n}\|_2^2/
\|w_{2,n}\|_2^2=n/P\to0\).  Markov's inequality and
(\ref{eq:regression-bias-rescue-nullspace-lower}) therefore imply, for
every fixed \(a<1\),
\[
  \mathbb P\!\left\{
    \mathcal E_n^{\mathrm{reg}}(\widehat w_{2,n})\geq a
  \right\}\longrightarrow1.
\]
Both marginals are non-benign whereas their fusion is
benign, establishing the task-level transition
\(\Nreg+\Nreg\to\Breg\).

\medskip
\noindent\textbf{Case 5:
\(\NFS+\NFS
\to\NFS\).}\par\smallskip
\noindent\emph{Target spectra.}
Take
\[
  D=Q=\operatorname{diag}(1,n^{-2}I_m,\rho I_{p-m}).
\]

\noindent\emph{Marginal checks.}
For either marginal, \(r_0=1+o(1)\), \(k_b=1\), and
\[
  S_1=c/n+o(n^{-1}),\qquad R_1=cn(1+o(1)).
\]
Thus both marginals violate the third condition.

\noindent\emph{Joint check.}
For the comparison direct sum, \(k_b(D\oplus Q)=2\) and
\(R_2(D\oplus Q)=2cn(1+o(1))\).  The dense comparison
(\ref{eq:regression-dense-joint-comparison}) transfers this
\(\Theta(n)\) post-cutoff effective rank to \(\Sigma_J\). More explicitly,
the two leading eigenvalues stay \(\Theta(1)\), so \(r_0,r_1=O(1)\), while
the remaining tail has mass \(\Theta(n^{-1})\), largest eigenvalue
\(\Theta(n^{-2})\), and \(r_2=(2c+o(1))n>b_{\mathrm B}n\). Hence the actual
joint cutoff is two and \(V_{b_{\mathrm B}}(\Sigma_J)=\Theta(1)\).

\noindent\emph{Conclusion.}
The third-condition obstruction survives the dense coupling, and
\(\Sigma_J\in\NFS\).  With \(Y=\xi\) and fixed
\(\tau^2>0\), Lemma~\ref{lem:fixed-probability-variance-lower} applies to
both marginals and the joint input.  All three tasks are therefore
non-benign, realizing
\(\Nreg+\Nreg\to\Nreg\). This completes all attainable
spectral and task-level rows.
\end{proof}

\subsection{Forward-reverse spectral equivalence}
\label{app:operator-equivalence-proof}

We begin with three deterministic lemmas. Eigenvalue sequences below are
arranged in decreasing order and padded with zeros when the matrix dimensions
differ.

\begin{lemma}[Candidate-cutoff characterization]
\label{lem:candidate-cutoff-characterization}
Fix \(b>0\). A covariance sequence \(M_n\) belongs to
\(\BFS\) if and only if \(r_0(M_n)=o(n)\) and there is a
sequence \(\ell_n=o(n)\) such that
\begin{equation}
  r_{\ell_n}(M_n)\geq bn,
  \qquad
  R_{\ell_n}(M_n)/n\to\infty.
  \label{eq:candidate-cutoff-characterization}
\end{equation}
\end{lemma}

\begin{proof}
Necessity follows by taking \(\ell_n=k_b(M_n)\). Conversely, let
\(k_n=k_b(M_n)\). Then \(k_n\leq\ell_n=o(n)\). Write
\(S_k=S_k(M_n)\) and \(Q_k=\sum_{j>k}\lambda_j(M_n)^2\). Since
\(\lambda_{k_n+1}\leq S_{k_n}/(bn)\),
\[
  Q_{k_n}
  \leq Q_{\ell_n}
  +(\ell_n-k_n)\frac{S_{k_n}^2}{b^2n^2}
  =o(S_{k_n}^2/n).
\]
Here \(Q_{\ell_n}=o(S_{\ell_n}^2/n)\) and
\(S_{\ell_n}\leq S_{k_n}\). Thus \(R_{k_n}/n\to\infty\), proving all three
conditions in Definition~\ref{def:threshold-specific-full-spectrum}.
\end{proof}

\begin{lemma}[Uniform spectral comparison]
\label{lem:uniform-spectral-comparison}
Suppose that, for fixed \(0<c\leq C<\infty\),
\begin{equation}
  c\,\lambda_j(M_n)
  \leq\lambda_j(N_n)
  \leq C\,\lambda_j(M_n)
  \quad\text{for every }j,n.
  \label{eq:uniform-ordered-spectrum-comparison}
\end{equation}
Then \(M_n\in\BFS\) if and only if
\(N_n\in\BFS\).
\end{lemma}

\begin{proof}
The uncut effective ranks are comparable within the factor \(C/c\).
Assume \(M_n\in\BFS\). By threshold invariance, choose a
candidate cutoff \(\ell_n=o(n)\) satisfying
\(r_{\ell_n}(M_n)\geq (C/c)bn\) and
\(R_{\ell_n}(M_n)\gg n\). Then
\[
  r_{\ell_n}(N_n)\geq(c/C)r_{\ell_n}(M_n)\geq bn,
  \qquad
  R_{\ell_n}(N_n)\geq(c/C)^2R_{\ell_n}(M_n)\gg n.
\]
Lemma~\ref{lem:candidate-cutoff-characterization} gives
\(N_n\in\BFS\). The reverse implication is symmetric.
\end{proof}

\begin{lemma}[Sum--direct-sum equivalence]
\label{lem:sum-direct-sum-equivalence}
For arbitrary positive-semidefinite matrices \(P_n,Q_n\) of the same
dimension,
\begin{equation}
  P_n+Q_n\in\BFS
  \quad\Longleftrightarrow\quad
  P_n\oplus Q_n\in\BFS.
  \label{eq:sum-direct-sum-equivalence}
\end{equation}
\end{lemma}

\begin{proof}
Write \(S=P+Q\) and \(G=P\oplus Q\), suppressing \(n\).
Their traces coincide, and
\begin{equation}
  \max\{\|P\|_{\mathrm{op}},\|Q\|_{\mathrm{op}}\}
  \leq\|S\|_{\mathrm{op}}
  \leq2\max\{\|P\|_{\mathrm{op}},\|Q\|_{\mathrm{op}}\}.
  \label{eq:sum-direct-sum-r0-comparison}
\end{equation}
Hence \(r_0(S)=o(n)\) if and only if \(r_0(G)=o(n)\).

First suppose \(G\in\BFS\). By threshold invariance, take
its cutoff \(k=o(n)\) at threshold \(4b\), and split the spectral head and
tail blockwise as \(P=P_H+P_T\) and \(Q=Q_H+Q_T\). The total head rank is at
most \(k\). For \(T=P_T+Q_T\),
\begin{align}
  \operatorname{tr}(T)&=S_k(G),&
  \|T\|_{\mathrm{op}}&\leq2\lambda_{k+1}(G),\notag\\
  \|T\|_{\mathrm F}^2
  &\leq2\sum_{j>k}\lambda_j(G)^2.
  \label{eq:direct-sum-to-sum-tail}
\end{align}
Thus \(r_0(T)\geq2bn\) and \(R_0(T)\gg n\).
Adding the positive-semidefinite head \(H=P_H+Q_H\) of rank at most \(k\)
preserves a wide tail after \(k\) directions.  Indeed, Weyl's inequalities
and Ky Fan's variational principle
\citep[Sections~4.3 and~4.4]{hornjohnson2013matrix} give
\begin{align}
  S_k(H+T)&\geq\operatorname{tr}(T)-k\|T\|_{\mathrm{op}},\notag\\
  \lambda_{k+1}(H+T)&\leq\|T\|_{\mathrm{op}},&
  \sum_{j>k}\lambda_j(H+T)^2&\leq\|T\|_{\mathrm F}^2.
  \label{eq:low-rank-head-tail-transfer}
\end{align}
For clarity, the first inequality follows because the sum of the largest
\(k\) eigenvalues of \(H+T\) is at most
\(\operatorname{tr}(H)+k\|T\|_{\mathrm{op}}\).  Weyl's rank-\(k\)
interlacing gives
\(\lambda_{k+j}(H+T)\leq\lambda_j(T)\) for every \(j\geq1\), which yields
the other two inequalities after taking \(j=1\) and summing squares.
Because \(k=o(n)\) and \(r_0(T)\gtrsim n\), the first right-hand side is
\(\operatorname{tr}(T)(1-o(1))\). Therefore \(k\) is a valid candidate
cutoff for \(S\), and
Lemma~\ref{lem:candidate-cutoff-characterization} yields
\(S\in\BFS\).

Conversely, suppose \(S\in\BFS\), and take its cutoff
\(k=o(n)\) at threshold \(2b\). Let \(\Pi\) project onto the span of the
eigenvectors of \(S\) after its first \(k\) directions. Define
\[
  P_T=P^{1/2}\Pi P^{1/2},
  \qquad
  Q_T=Q^{1/2}\Pi Q^{1/2}.
\]
The complementary heads have total rank at most \(2k\). The nonzero spectra
of \(P_T\) and \(\Pi P\Pi\) coincide, and similarly for \(Q_T\). Positivity
therefore gives
\begin{align}
  \operatorname{tr}(P_T)+\operatorname{tr}(Q_T)&=S_k(S),\notag\\
  \max\{\|P_T\|_{\mathrm{op}},\|Q_T\|_{\mathrm{op}}\}
  &\leq\lambda_{k+1}(S),\notag\\
  \|P_T\|_{\mathrm F}^2+\|Q_T\|_{\mathrm F}^2
  &\leq\|\Pi S\Pi\|_{\mathrm F}^2
   =\sum_{j>k}\lambda_j(S)^2.
  \label{eq:sum-to-direct-sum-tail}
\end{align}
Hence \(P_T\oplus Q_T\) has effective ranks at least those of the tail of
\(S\).  Moreover,
\[
  G=
  \bigl(P^{1/2}(I-\Pi)P^{1/2}\bigr)
  \oplus
  \bigl(Q^{1/2}(I-\Pi)Q^{1/2}\bigr)
  +(P_T\oplus Q_T),
\]
and the first summand has rank at most \(2k\).  Applying
(\ref{eq:low-rank-head-tail-transfer}) with this head rank shows that
\(2k=o(n)\) is a candidate cutoff for \(G\). The
candidate-cutoff lemma completes the proof.
\end{proof}

\begin{proof}[Proof of Theorem~\ref{thm:regression-forward-reverse-equivalence}]
Let \(G_n=D_n\oplus\Sigma_{U,n}\). The triangular factorization
\begin{equation}
  \Sigma_{J,n}
  =L_{A,n}G_nL_{A,n}^\top,
  \qquad
  L_{A,n}=
  \begin{pmatrix}I&0\\A_n&I\end{pmatrix},
  \qquad
  L_{A,n}^{-1}=
  \begin{pmatrix}I&0\\-A_n&I\end{pmatrix}
  \label{eq:joint-triangular-factorization}
\end{equation}
shows that \(\Sigma_{J,n}\) and \(G_n\) have uniformly comparable ordered
eigenvalues whenever \(\|A_n\|_{\mathrm{op}}=O(1)\).  More explicitly, the
Courant--Fischer principle \citep[Section~4.2]{hornjohnson2013matrix} gives
\[
  \|L_{A,n}^{-1}\|_{\mathrm{op}}^{-2}\lambda_j(G_n)
  \leq \lambda_j(\Sigma_{J,n})
  \leq \|L_{A,n}\|_{\mathrm{op}}^2\lambda_j(G_n),
\]
and both norms are uniformly bounded. Lemma~\ref{lem:uniform-spectral-comparison}
therefore gives
\begin{equation}
  \Sigma_{J,n}\in\BFS
  \quad\Longleftrightarrow\quad
  D_n\oplus\Sigma_{U,n}\in\BFS.
  \label{eq:joint-direct-sum-equivalence}
\end{equation}

Under (\ref{eq:forward-map-conditioning}), the nonzero eigenvalues of
\(A_nD_nA_n^\top\) are those of
\(D_n^{1/2}A_n^\top A_nD_n^{1/2}\). The latter matrix lies between
\(c_AD_n\) and \(C_AD_n\) in Loewner order. After padding by zeros,
Lemma~\ref{lem:uniform-spectral-comparison} gives
\[
  D_n\oplus\Sigma_{U,n}\in\BFS
  \quad\Longleftrightarrow\quad
  (A_nD_nA_n^\top)\oplus\Sigma_{U,n}
  \in\BFS.
\]
Finally, Lemma~\ref{lem:sum-direct-sum-equivalence} identifies the last state
with that of
\(A_nD_nA_n^\top+\Sigma_{U,n}=\Sigma_{2,n}\). This proves the first
equivalence.

For the reverse direction, define \(B_n\) by
(\ref{eq:reverse-conditional-mean-map}) and let
\[
  V_n:=X_1-B_nX_2.
\]
Joint Gaussianity and the conditional-mean formula
\citep[Chapter~2]{anderson2003multivariate} give \(V_n\perp X_2\) and
\[
  X_1=B_nX_2+V_n,\qquad
  \operatorname{Cov}(V_n)
  =D_n-B_n\Sigma_{2,n}B_n^\top\succeq0.
\]
After permuting the two joint coordinate blocks, this is the same triangular
model with forward map \(B_n\). Condition
(\ref{eq:reverse-map-conditioning}) therefore allows the first part of the
proof to be applied with \(X_2\) as the shared variable and \(X_1\) as the
second source. It yields
\(\Sigma_{J,n}\in\BFS\) if and only if
\(\Sigma_{1,n}\in\BFS\). Combining the two conclusions
proves the three-way equivalence.
\end{proof}

\subsection{Case study I: vanishing innovation energy}
\label{app:small-energy-proof}

\begin{proof}[Verification of (\ref{eq:small-energy-spectral-states})]
Let \(p=n^3\), \(m=\lfloor cn\rfloor\), and use
(\ref{eq:small-energy-construction}). For the first marginal,
\[
  r_0(\Sigma_1)=1+n^{-1},\qquad
  S_1(\Sigma_1)=n^{-1},\qquad
  r_1(\Sigma_1)=R_1(\Sigma_1)=p=n^3.
\]
Thus \(k_{b_{\mathrm B}}(\Sigma_1)=1\) eventually and
\(\Sigma_1\in\BFS\).

The second marginal is diagonal. Its eigenvalues consist of
\(1+n^{-8}\), followed by \(m\) copies of \(n^{-2}+n^{-4}\), and
\(p-m\) copies of \(n^{-4}+n^{-8}\). Hence
\begin{align}
  r_0(\Sigma_2)&=1+o(1),\notag\\
  S_1(\Sigma_2)&=(c+1)n^{-1}(1+o(1)),\notag\\
  \sum_{j>1}\lambda_j(\Sigma_2)^2
  &=cn^{-3}(1+o(1)).
  \label{eq:small-energy-second-marginal}
\end{align}
Since \(r_1(\Sigma_2)=(c+1)n(1+o(1))>b_{\mathrm B}n\),
the cutoff equals \(1\), and
\[
  R_1(\Sigma_2)
  =\frac{(c+1)^2}{c}\,n(1+o(1)).
\]
Thus \(\Sigma_2\in\NFS\).

For the joint covariance, a coordinate with
\(\operatorname{Var}(X_{1,j})=s\) and
\(\operatorname{Var}(U_j)=u\) contributes the block
\[
  M(s,u)=
  \begin{pmatrix}s&s\\s&s+u\end{pmatrix}.
\]
Its two eigenvalues are
\begin{equation}
  \lambda_\pm(s,u)
  =\frac{2s+u\pm\sqrt{4s^2+u^2}}{2}.
  \label{eq:small-energy-block-eigenvalues}
\end{equation}
The leading block \((s,u)=(1,n^{-8})\) has
\(\lambda_+=2+o(n^{-1})\) and
\(\lambda_-=\tfrac12n^{-8}(1+o(1))\). For the \(m\) shoulder blocks,
\((s,u)=(n^{-4},n^{-2})\), so
\[
  \lambda_+=n^{-2}(1+o(1)),
  \qquad
  \lambda_-=n^{-4}(1+o(1)).
\]
For the remaining \(p-m\) blocks,
\((s,u)=(n^{-4},n^{-8})\), and
\[
  \lambda_+=2n^{-4}(1+o(1)),
  \qquad
  \lambda_-=\tfrac12n^{-8}(1+o(1)).
\]
It follows that \(r_0(\Sigma_J)=1+o(1)\) and, after the unique leading
eigenvalue is removed,
\begin{align}
  S_1(\Sigma_J)&=(c+2)n^{-1}(1+o(1)),\notag\\
  \sum_{j>1}\lambda_j(\Sigma_J)^2
  &=cn^{-3}(1+o(1)).
  \label{eq:small-energy-joint-tail}
\end{align}
The largest remaining eigenvalue is \(n^{-2}(1+o(1))\), so
\(r_1(\Sigma_J)=(c+2)n(1+o(1))>b_{\mathrm B}n\) and the joint cutoff is
\(1\). Finally,
\[
  R_1(\Sigma_J)
  =\frac{(c+2)^2}{c}\,n(1+o(1)),
\]
which proves \(\Sigma_J\in\NFS\) and
(\ref{eq:small-energy-joint-obstruction}). The norm and trace assertions for
\(\Sigma_U\) follow directly from its diagonal definition.
\end{proof}

\subsection{Case study II: task-level non-closure in the logarithmic gap}
\label{app:log-gap-counterexample-proof}

\begin{proof}[Proof of Theorem~\ref{thm:regression-task-level-nonclosure}]
Write \(h=h_n\), \(p=p_n\), \(\varepsilon=\varepsilon_n\),
\(\eta=\eta_n\), and \(\rho=\rho_n\). We first verify benign
overfitting for the two marginal tasks and then establish a high-probability
lower bound for the fused task.

\paragraph{The two marginal tasks.}
Their covariances are
\[
  \Sigma_{1,n}=\operatorname{diag}(1,\eta I_p),\qquad
  \Sigma_{2,n}=\operatorname{diag}
  (1+\varepsilon,(\eta+\rho)I_p).
\]
Since \(p\eta=nh\), \(p\eta^2=h^2/n=o(1)\), and
\(p\rho=o(1)\), both spectra satisfy
\begin{equation}
  \frac{r_0(\Sigma_{r,n})}{n}=h(1+o(1)),\qquad
  k_{b_{\mathrm B}}(\Sigma_{r,n})=0,\qquad
  R_0(\Sigma_{r,n})=n^2h^2(1+o(1)),
  \quad r\in\{1,2\}.
  \label{eq:log-gap-marginal-spectral-calculation}
\end{equation}
Indeed, the leading eigenvalue is \(1+o(1)\), the spectral sum is
\(nh(1+o(1))\), and the squared spectral sum is \(1+o(1)\).
In particular, \(r_0/(n\log(1+r_0))\to0\), although
\(r_0/n\to\infty\).

For source \(1\), the marginal regression input is
\[
  Y=\sqrt\varepsilon\,X_{1,1}+\zeta_1,
  \qquad
  \zeta_1=-\varepsilon^{-1/2}U_1
  \sim\mathcal N(0,1),
\]
with \(\zeta_1\perp X_1\). Thus
\(w_{1,n}=\sqrt\varepsilon e_1\) and \(\tau_{1,n}^2=1\).
Apply Theorem~\ref{thm:bartlett-regression-bounds} with
\(\delta_n=n^{-2}\). By
(\ref{eq:log-gap-marginal-spectral-calculation}), its bias and variance
terms are respectively
\[
  O(\varepsilon h)=O(h^{-1}),
  \qquad
  O\!\left((\log n)\frac{n}{R_0(\Sigma_{1,n})}\right)
  =O(n^{-1}).
\]
Consequently,
\(\mathcal E_{1,n}^{\mathrm{reg}}\to_{\mathbb P}0\).

For source \(2\),
\[
  \operatorname{Cov}(X_2,Y)
  =D_n\beta_n+\Sigma_{U,n}\gamma_n=0.
\]
The variables are jointly Gaussian, so zero cross-covariance implies
\(Y\perp X_2\) \citep[Chapter~2]{anderson2003multivariate}. Its marginal
coefficient is therefore zero and its marginal noise variance is
\(\operatorname{Var}(Y)=1+\varepsilon\). The same Bartlett bound now has
zero bias term and variance term \(O(n^{-1})\), proving
\(\mathcal E_{2,n}^{\mathrm{reg}}\to_{\mathbb P}0\).

For completeness, the reverse conditional-mean map in this construction is
\[
  B_n=D_n(D_n+\Sigma_{U,n})^{-1}
  =\operatorname{diag}\!\left(
    \frac1{1+\varepsilon},
    \frac{\eta}{\eta+\rho}I_p
  \right).
\]
Since \(\rho/\eta\to0\), eventually
\(\frac12I\preceq B_n\preceq I\); hence both \(B_n^\top B_n\) and
\(A_n^\top A_n=I\) satisfy the uniform two-sided bounds of
Theorem~\ref{thm:regression-forward-reverse-equivalence}.

\paragraph{The fused task.}
The covariance of the first coordinate pair \((X_{1,1},X_{2,1})\) is
\[
  M_\varepsilon=
  \begin{pmatrix}1&1\\1&1+\varepsilon\end{pmatrix}.
\]
Let \((\lambda_-,v_-)\) and \((\lambda_+,v_+)\) be its ordered normalized
eigenpairs. Direct calculation gives
\begin{equation}
  \lambda_\pm
  =\frac{2+\varepsilon\pm\sqrt{4+\varepsilon^2}}2,
  \qquad
  \lambda_-=\frac\varepsilon2+O(\varepsilon^2),
  \qquad
  \lambda_+=2+O(\varepsilon).
  \label{eq:log-gap-signal-block-eigenvalues}
\end{equation}
In the corresponding standardized principal coordinates, write the two
training columns as \(z_-,z_+\sim\mathcal N(0,I_n)\). Because
\(\operatorname{Cov}((X_{1,1},X_{2,1})^\top,Y)
=(\sqrt\varepsilon,0)^\top\), the response admits the exact expansion
\begin{equation}
  y=a_-z_-+a_+z_+,
  \qquad
  a_\pm=\frac{\sqrt\varepsilon\,(v_\pm)_1}{\sqrt{\lambda_\pm}}.
  \label{eq:log-gap-label-principal-coordinates}
\end{equation}
There is no residual noise in the joint input. Hence
\(a_-^2+a_+^2=\operatorname{Var}(Y)=1+\varepsilon\). Moreover,
\(a_+^2\leq\varepsilon/\lambda_+=O(\varepsilon)\), and therefore
\begin{equation}
  a_-^2\longrightarrow1.
  \label{eq:log-gap-difference-label-energy}
\end{equation}

Each of the remaining \(p\) coordinate pairs has covariance block
\(\bigl(\begin{smallmatrix}\eta&\eta\\\eta&\eta+\rho\end{smallmatrix}\bigr)\).
Its two eigenvalues are
\[
  \mu_\pm
  =\frac{2\eta+\rho\pm\sqrt{4\eta^2+\rho^2}}2,
  \qquad
  \mu_+=2\eta(1+o(1)),
  \quad
  \mu_-=\frac\rho2(1+o(1)).
\]
Let \(Z_+\in\mathbb R^{n\times p}\) contain the standardized Gaussian
columns associated with the \(p\) copies of \(\mu_+\), and use
\(Z_-\) for the \(\mu_-\) columns. All these standardized principal
coordinates are mutually independent. The joint Gram matrix is
\[
  K=\lambda_-z_-z_-^\top+\lambda_+z_+z_+^\top
    +\mu_+Z_+Z_+^\top+\mu_-Z_-Z_-^\top.
\]
Put \(A_-=K-\lambda_-z_-z_-^\top\).  The two-sided singular-value bound for
rectangular sub-Gaussian matrices
\citep[Theorem~4.6.1]{vershynin2018highdimensional}, applied to
\(Z_+^\top\), and Gaussian norm concentration
\citep[Theorem~3.1.1]{vershynin2018highdimensional} imply, because
\(p/n\to\infty\),
\[
  \lambda_{\min}(Z_+Z_+^\top)\geq p/2,
  \qquad
  \|z_-\|_2^2\vee\|z_+\|_2^2\leq2n
\]
with probability tending to one. On this event,
\begin{equation}
  A_-\succeq\mu_+Z_+Z_+^\top\succeq cnhI_n,
  \qquad
  \|A_-^{-1}\|_{\mathrm{op}}\leq\frac{C}{nh}.
  \label{eq:log-gap-nuisance-gram-lower}
\end{equation}

Define
\[
  u=\lambda_-z_-^\top A_-^{-1}z_-,
  \qquad
  c_\times=\lambda_-a_+z_-^\top A_-^{-1}z_+.
\]
(\ref{eq:log-gap-signal-block-eigenvalues}) and
(\ref{eq:log-gap-nuisance-gram-lower}), together with
\(|a_+|=O(\sqrt\varepsilon)\), yield
\begin{equation}
  u=O_{\mathbb P}\!\left(\frac\varepsilon h\right)
    =O_{\mathbb P}(h^{-3}),
  \qquad
  c_\times=O_{\mathbb P}\!\left(
    \frac{\varepsilon^{3/2}}h
  \right)=O_{\mathbb P}(h^{-4}).
  \label{eq:log-gap-survival-quantities}
\end{equation}

Let \(\widehat a_-\) be the standardized population amplitude of the
minimum-norm estimator in the \(v_-\) direction. Since the corresponding
design column is \(\sqrt{\lambda_-}z_-\),
\(\widehat a_-=\lambda_-z_-^\top K^{-1}y\). Applying the
Sherman--Morrison identity \citep[Section~0.7.4]{hornjohnson2013matrix} to
\(K=A_-+\lambda_-z_-z_-^\top\) and using
(\ref{eq:log-gap-label-principal-coordinates}) gives the exact identity
\begin{equation}
  \widehat a_-=\frac{a_-u+c_\times}{1+u},
  \qquad
  a_- -\widehat a_-=\frac{a_- -c_\times}{1+u}.
  \label{eq:log-gap-exact-difference-survival}
\end{equation}
By (\ref{eq:log-gap-difference-label-energy}) and
(\ref{eq:log-gap-survival-quantities}),
\((a_- -\widehat a_-)^2\to_{\mathbb P}1\).

Finally, in standardized principal coordinates the noiseless population
excess risk is the sum of the squared amplitude errors over all eigendirections.
It is therefore bounded below by the contribution of \(v_-\):
\[
  \mathcal E_{J,n}^{\mathrm{reg}}
  \geq(a_- -\widehat a_-)^2.
\]
It follows that
\(\mathbb P\{\mathcal E_{J,n}^{\mathrm{reg}}\geq1/2\}\to1\), completing
the proof.
\end{proof}

\section{Proofs for the Classification Results}
\label{app:classification-proofs}

\subsection{The \texorpdfstring{\(\Delta\)}{Delta}-criterion and classification benignity}
\label{app:classification-delta-equivalence}
\label{app:classification-su-cn-scales}

\begin{lemma}[Nuisance Gram and variance trace]
\label{lem:nuisance-gram-variance}
Let \(z_j\stackrel{\mathrm{iid}}\sim N(0,I_n)\), let
\(\mu_1\geq\mu_2\geq\cdots\geq0\), and define
\[
  A_-:=\sum_j\mu_jz_jz_j^\top,
  \qquad H:=\sum_j\mu_j^2z_jz_j^\top,
  \qquad B:=A_-^{-1},\qquad C:=BHB.
\]
There are universal constants \(b_{\mathrm C},c_{\mathrm C},C>1\) and
\(0<c<1\) such
that the following holds. Let \(k=k_{-t}^\star(b_{\mathrm C})\), put
\(S=S_k^{(-t)}\) and \(V=k/n+n/R_k^{(-t)}\), and suppose
\(k<n/c_{\mathrm C}\). Then, with probability at least
\(1-Ce^{-n/C}\), \(A_-\) is invertible and
\begin{equation}
  c\frac{n}{S}\leq\operatorname{tr}(B)\leq C\frac{n}{S},
  \qquad
  \|B\|_{\mathrm F}\leq C\frac{\sqrt n}{S},
  \qquad
  cV\leq\operatorname{tr}(C)\leq CV.
  \label{eq:nuisance-gram-variance-bounds}
\end{equation}
All constants are uniform over the nuisance spectrum.
\end{lemma}

\begin{proof}
The matrix \(C\) is exactly the variance-trace matrix studied by
\citet{bartlett2020benign}, with their covariance spectrum replaced by the
signal-deleted nuisance spectrum. We spell out the reduction so that no
additional random-matrix estimate is implicit.

Let \(T_k=\sum_{j>k}\mu_jz_jz_j^\top\). By definition of the cutoff,
\(r_k^{(-t)}=S/\mu_{k+1}\geq b_{\mathrm C}n\). Choose
\(b_{\mathrm C}\) at least as large as the universal threshold in
\citet[Lemma~10]{bartlett2020benign}. Part 3 of that lemma gives, on an
event of probability at least \(1-2e^{-n/C}\),
\begin{equation}
  cS I_n\preceq T_k\preceq CS I_n.
  \label{eq:nuisance-tail-gram-sandwich}
\end{equation}
Since \(A_-\succeq T_k\), the lower inequality implies
\(\|B\|_{\mathrm{op}}\leq C/S\), and hence
\(\operatorname{tr}(B)\leq Cn/S\) and
\(\|B\|_{\mathrm F}\leq\sqrt n\|B\|_{\mathrm{op}}
\leq C\sqrt n/S\).
For the reverse trace bound, the head \(A_--T_k\) has rank at most \(k\).
On its null space, of dimension at least \(n-k\), the Rayleigh quotient of
\(A_-\) is at most \(CS\) by
(\ref{eq:nuisance-tail-gram-sandwich}). The min--max principle
\citep[Section~4.2]{hornjohnson2013matrix} therefore shows that at least
\(n-k\) eigenvalues of \(A_-\) are at most \(CS\). Because
\(k<n/c_{\mathrm C}\),
\(\operatorname{tr}(B)\geq(n-k)/(CS)\geq cn/S\).

It remains to control \(C\). In the notation of
\citet{bartlett2020benign},
\[
  C=(X_-X_-^\top)^{-1}X_-\Sigma_-X_-^\top
    (X_-X_-^\top)^{-1},
\]
because \(X_-X_-^\top=A_-\) and
\(X_-\Sigma_-X_-^\top=H\). Apply their Lemma~11 with
\(l=k=k^\star\). Since \(\mu_{k+1}r_k^{(-t)}=S\), it yields
\[
  \operatorname{tr}(C)
  \leq C\left(\frac{k}{n}
    +\frac{n\sum_{j>k}\mu_j^2}{S^2}\right)
  =CV.
\]
For the lower bound, part 2 of their Lemma~16 at \(k=k^\star\), followed by
their Lemma~17, identifies the same minimum over \(l\leq k^\star\) with a
constant multiple of \(k/n+n/R_k^{(-t)}=V\). Hence
\(\operatorname{tr}(C)\geq cV\). Their three events have exponentially
small failure probability; a union bound proves
(\ref{eq:nuisance-gram-variance-bounds}). Finally,
\(r_k^{(-t)}\geq b_{\mathrm C}n>n\) forces the positive tail rank to exceed
\(n\), so \(A_-\) is invertible almost surely (and also on the displayed
high-probability event).
\end{proof}

\begin{lemma}[Regular-regime SU and CN scales]
\label{lem:classification-su-cn-scales}
Under the assumptions of
Theorem~\ref{thm:classification-main-characterization}, suppose that
\(k_{-t}^\star<n/c_{\mathrm C}\). With \(s_n,V_n\)
defined in (\ref{eq:classification-spectral-scales}),
\begin{equation}
  \mathrm{SU}(\widehat w)\asymp_{\mathbb P}s_n,
  \qquad
  \mathrm{CN}(\widehat w)^2=O_{\mathbb P}(V_n).
  \label{eq:app-classification-su-cn-upper}
\end{equation}
More precisely, there are \(0<c_s<C_s<\infty\) such that
\[
  \mathbb P\!\left\{
    c_ss_n\leq\mathrm{SU}(\widehat w)\leq C_ss_n
  \right\}\longrightarrow1.
\]
There are constants \(c,p>0\), depending only on
\(b_{\mathrm C},c_{\mathrm C}\), and \(\bar\tau\), such that
\begin{equation}
  \mathbb P\!\left\{
    \mathrm{CN}(\widehat w)^2\geq cV_n
  \right\}\geq p-o(1).
  \label{eq:app-classification-cn-lower}
\end{equation}
\end{lemma}

\begin{proof}
Write \(z=z_t\), and let \(z_j\sim N(0,I_n)\) denote the standardized
training column in nuisance direction \(j\).  Define
\[
  A_-:=\sum_j\mu_jz_jz_j^\top,
  \quad B:=A_-^{-1},
  \quad H:=\sum_j\mu_j^2z_jz_j^\top,
  \quad C:=BHB.
\]
The nuisance variables, and hence \((B,C)\), are independent of \((z,y)\).
The full Gram matrix is \(A=A_-+\lambda_tzz^\top\).  The
Sherman--Morrison rank-one inverse formula
\citep[Section~0.7.4]{hornjohnson2013matrix} therefore gives the exact
identities
\begin{align}
  \mathrm{SU}(\widehat w)
  &=\frac{\lambda_tz^\top By}
          {1+\lambda_tz^\top Bz},
  \label{eq:app-exact-su}\\
  \mathrm{CN}(\widehat w)^2
  &=(y-\mathrm{SU}(\widehat w)z)^\top
    C(y-\mathrm{SU}(\widehat w)z).
  \label{eq:app-exact-cn}
\end{align}

Let \(\kappa_\tau=\sqrt{2/\pi}/\sqrt{1+\tau^2}\).  To verify the scalar
identity used below, set \(W=(Z+\tau G)/\sqrt{1+\tau^2}\).  Gaussian
conditioning \citep[Chapter~2]{anderson2003multivariate} gives
\(\mathbb E[Z\mid W]=W/\sqrt{1+\tau^2}\), and therefore
\[
  \mathbb E[Z\operatorname{sign}(Z+\tau G)]
  =\frac{\mathbb E|W|}{\sqrt{1+\tau^2}}
  =\kappa_\tau.
\]
Independence across samples consequently yields
\begin{equation}
  \mathbb E[z^\top By\mid B]=\kappa_\tau\operatorname{tr}(B),
  \qquad
  \mathbb E[z^\top Bz\mid B]=\operatorname{tr}(B).
  \label{eq:app-su-conditional-means}
\end{equation}
The variables \(z_i\pm y_i\) are independent across \(i\), centered, and
uniformly sub-Gaussian. Polarization followed by the Hanson--Wright inequality
\citep{rudelson2013hansonwright}, together
with the usual Gaussian quadratic-form bound, gives
\begin{align}
  z^\top By
  &=\kappa_\tau\operatorname{tr}(B)+O_{\mathbb P}(\|B\|_{\mathrm F}),
  \notag\\
  z^\top Bz
  &=\operatorname{tr}(B)+O_{\mathbb P}(\|B\|_{\mathrm F}).
  \label{eq:app-su-concentration}
\end{align}

Lemma~\ref{lem:nuisance-gram-variance}, applied to the signal-deleted
nuisance covariance, gives
\begin{equation}
  \operatorname{tr}(B)\asymp_{\mathbb P}\frac{n}{S_n},
  \qquad
  \frac{\|B\|_{\mathrm F}}{\operatorname{tr}(B)}
    =O_{\mathbb P}(n^{-1/2}),
  \qquad
  \operatorname{tr}(C)\asymp_{\mathbb P}V_n.
  \label{eq:app-bartlett-classification-traces}
\end{equation}
The first two relations and
(\ref{eq:app-exact-su})-(\ref{eq:app-su-concentration}) show that, with
\(u=\lambda_t\operatorname{tr}(B)\),
\[
  \mathrm{SU}(\widehat w)
  =\kappa_\tau\frac{u}{1+u}+o_{\mathbb P}\!\left(
    \frac{u}{1+u}
  \right).
\]
Since \(\kappa_\tau\) is bounded above and away from zero and
\(u\asymp_{\mathbb P}n\lambda_t/S_n\), this proves the survival statement
in (\ref{eq:app-classification-su-cn-upper}).

For contamination, positive semidefiniteness and
(\ref{eq:app-exact-cn}) give
\[
  \mathrm{CN}(\widehat w)^2
  \leq2y^\top Cy+2\mathrm{SU}(\widehat w)^2z^\top Cz.
\]
Conditionally on \(C\), both quadratic forms have expectation
\(\operatorname{tr}(C)\).  Conditional Markov and
\(\mathrm{SU}(\widehat w)=O_{\mathbb P}(1)\) therefore yield
\(\mathrm{CN}^2=O_{\mathbb P}(\operatorname{tr}C)\), proving the second
upper bound in view of (\ref{eq:app-bartlett-classification-traces}).

It remains to prove the lower bound without imposing relative concentration
on \(C\).  Put
\[
  a^\circ:=\kappa_\tau\frac{u}{1+u},
  \qquad Q:=(y-a^\circ z)^\top C(y-a^\circ z).
\]
Conditionally on the nuisance variables, \(a^\circ\) is deterministic and
the coordinates of \(y-a^\circ z\) are independent and centered.  Hence
\begin{equation}
  \mathbb E[Q\mid B,C]
  =\bigl(1-2\kappa_\tau a^\circ+(a^\circ)^2\bigr)
   \operatorname{tr}(C)
  \geq(1-\kappa_\tau^2)\operatorname{tr}(C).
  \label{eq:app-cn-Q-first-moment}
\end{equation}
Their fourth moments are uniformly bounded.  If
\(w=y-a^\circ z\), independence and centering give the standard
quadratic-form expansion
\[
  \mathbb E[(w^\top Cw)^2\mid B,C]
  \leq C_0\bigl\{(\operatorname{tr}C)^2+\|C\|_{\mathrm F}^2\bigr\}
  \leq 2C_0(\operatorname{tr}C)^2,
\]
where the last inequality uses \(C\succeq0\).  Applying the
Paley--Zygmund inequality \citep[Lemma~19]{paley1932some} conditionally,
with threshold parameter \(1/2\), supplies constants \(c_0,p_0>0\) for
which
\[
  \mathbb P\{Q\geq c_0\operatorname{tr}(C)\mid B,C\}\geq p_0.
\]
Finally, (\ref{eq:app-su-concentration}) implies
\(\delta:=\mathrm{SU}(\widehat w)-a^\circ=o_{\mathbb P}(1)\).
Conditional Markov gives
\[
  z^\top Cz=O_{\mathbb P}(\operatorname{tr}C),
  \qquad
  (y-a^\circ z)^\top C(y-a^\circ z)
  =O_{\mathbb P}(\operatorname{tr}C),
\]
and Cauchy--Schwarz in the \(C\)-seminorm bounds the cross term by the
geometric mean of these two quantities.  Since
\[
  (y-\mathrm{SU}(\widehat w)z)^\top
  C(y-\mathrm{SU}(\widehat w)z)
  =Q-2\delta z^\top C(y-a^\circ z)+\delta^2z^\top Cz,
\]
(\ref{eq:app-exact-cn}) therefore yields
\[
  \mathrm{CN}(\widehat w)^2
  =Q+o_{\mathbb P}(\operatorname{tr}C).
\]
Combining the last two displays with the lower relation in
(\ref{eq:app-bartlett-classification-traces}) proves
(\ref{eq:app-classification-cn-lower}).
\end{proof}

\medskip
\noindent\textbf{Proof of
Theorem~\ref{thm:classification-main-characterization}.}\par\smallskip
\begin{proof}
Let \(\widehat w\) be the minimum-norm interpolator.  In the regular regime,
Lemma~\ref{lem:classification-su-cn-scales} gives
\[
  \mathrm{SU}(\widehat w)\asymp_{\mathbb P}s_n,
  \qquad
  \mathrm{CN}(\widehat w)^2=O_{\mathbb P}(V_n),
\]
and constants \(c,p>0\) such that
\[
  \mathbb P\!\left\{
    \mathrm{CN}(\widehat w)^2\geq cV_n
  \right\}\geq p-o(1).
\]
The upper bound uses only conditional Markov. The fixed-positive-probability
lower bound uses the Paley--Zygmund inequality
\citep{paley1932some} and does not assume relative concentration of the
contamination matrix.

\noindent\emph{Sufficiency.}
On the preceding positive-survival event, \(\Gamma\) is defined and the
upper bounds imply
\[
  \Gamma(\widehat w)^2
  =O_{\mathbb P}\!\left(\frac{V_n}{s_n^2}\right)
  =O_{\mathbb P}(\Delta_n).
\]
Thus \(\Delta_n\to0\) gives
\(\Gamma(\widehat w)\to_{\mathbb P}0\), and
(\ref{eq:classification-risk-gamma}) yields classification benign
overfitting.  This implication is global: because
\(\Delta_n\geq k_{-t}^\star/n\), the premise \(\Delta_n\to0\) itself puts
the sequence in the regular regime eventually.  Moreover,
\(\Delta_n\geq n/R_n\) gives \(R_n/n\to\infty\). Since \(R_n\) is at most
the number of positive nuisance eigenvalues beyond the cutoff, the nuisance
rank exceeds \(n\) eventually. The Gaussian design therefore has full row
rank almost surely, so the interpolator used above is feasible.

\noindent\emph{Necessity in the regular regime.}
The proof of the survival estimate gives positive constants \(c_s,C_s\)
such that
\[
  \mathbb P\!\left\{
    c_ss_n\leq\mathrm{SU}(\widehat w)\leq C_ss_n
  \right\}\longrightarrow1.
\]
Intersecting this event with the CN lower event gives, with probability at
least \(p-o(1)\), a positive survival and
\[
  \Gamma(\widehat w)^2
  \geq \frac{c}{C_s^2}\frac{V_n}{s_n^2}
  =\frac{c}{C_s^2}\Delta_n.
\]
If \(\Delta_n\not\to0\), some subsequence has
\(\Delta_n\geq\varepsilon>0\), and along that subsequence \(\Gamma\) stays
away from zero with fixed positive probability.  Since
\(\rho_{\tau_n}\geq(1+\bar\tau^2)^{-1/2}>0\), the excess-risk map in
(\ref{eq:classification-risk-gamma}) is uniformly strictly increasing in
\(\Gamma\).  The excess risk therefore cannot converge to zero in
probability.  This proves the converse.
\end{proof}

\subsection{Survival-contamination transfer identity}
\label{app:classification-fusion-mixing}

Proposition~\ref{prop:classification-fusion-balance} records an algebraic
identity for \(\Delta\).  We use it only to separate the two mechanisms by
which fusion changes classification.  When the signal-deleted cutoffs of a
marginal input \(i\in\{1,2\}\) and the joint input are finite, and
\(V_{i,n}>0\),
\begin{equation}
  \frac{\Delta_{J,n}}{\Delta_{i,n}}
  =
  \underbrace{\frac{V_{J,n}}{V_{i,n}}}_{
    \text{contamination transfer}}
  \times
  \underbrace{\frac{s_{i,n}^2}{s_{J,n}^2}}_{
    \text{survival transfer}}
  \label{eq:app-classification-mixing-decomposition}
\end{equation}
The first factor records how concatenation changes the signal-deleted
nuisance spectrum through \(k_{-t}^\star\) and \(R_{k_{-t}^\star}\).
The second records how the total nuisance mass and signal eigenvalue change
the surviving signal.  Neither factor is monotone: a broader joint tail can
reduce contamination, while a larger nuisance mass can simultaneously
reduce survival.

\medskip
\noindent\textbf{Proof of
Proposition~\ref{prop:classification-fusion-balance}.}\par\smallskip
\begin{proof}
Since \(\Delta_{r,n}=V_{r,n}/s_{r,n}^2\),
(\ref{eq:app-classification-mixing-decomposition}) is exact and equals
(\ref{eq:classification-fusion-balance-identity}).  Applying
Lemma~\ref{lem:classification-su-cn-scales} input by
input gives
\[
  \mathrm{SU}(\widehat w_r)\asymp_{\mathbb P}s_{r,n},\qquad
  \mathrm{CN}(\widehat w_r)^2=O_{\mathbb P}(V_{r,n}),
  \qquad r\in\{i,J\},
\]
together with the matching fixed-positive-probability CN lower bound.
Theorem~\ref{thm:classification-main-characterization} then shows, in the
regular regime, that the fused classifier is benign exactly when
\(\Delta_{J,n}\to0\), equivalently when the product on the right-hand side
of (\ref{eq:app-classification-mixing-decomposition}) tends to zero.

The factors in (\ref{eq:app-classification-mixing-decomposition}) should be
read as deterministic spectral-scale transfers, not as a claim that the
random ratio \(\mathrm{CN}_J^2/\mathrm{CN}_i^2\) concentrates.  The
classification criterion needs only the CN upper bound for sufficiency and
the fixed-positive-probability lower bound for necessity.
\end{proof}

\subsection{Six classification interaction patterns}
Here is the proof of Theorem~\ref{thm:classification-six-patterns}. We will deviate from the notation in the main text to simplify the proof. Let \(X_1\) and \(X_2\) be two inputs, and let \(Y\) be the label. We will construct a joint distribution of \((X_1,X_2,Y)\) such that the marginal distributions of \((X_1,Y)\) and \((X_2,Y)\) have the desired properties, and the joint distribution of \((X_1,X_2,Y)\) has the desired interaction pattern.
\label{app:classification-transition-witnesses}

\begin{proof}
\medskip
\noindent\textbf{Step 1: common dense nuisance coupling.}\par\smallskip
For the five rows other than \(\Ncls+\Ncls\to\Bcls\), let
\(d_1,\ldots,d_P\) and \(q_1,\ldots,q_P\) be the two target nuisance
spectra in Table~\ref{tab:classification-witness-scales}, with the following
explicit common-dimension convention. In the \(\Bcls+\Bcls\to\Ncls\) row take
\(P=n^4\) and append \(n^{-10}\) to the second list until it has length
\(P\); in the \(\Bcls+\Ncls\to\Bcls\) row take \(P=n^2\) and use the same padding for
the second list.  The other three lists already have equal length.  For the
first padded row, the padding contributes respectively
\(O(n^{-6}),O(n^{-16}),O(n^{-4})\) to
\(\sum q_j,\sum q_j^2,\sum q_j/d_j\); for the second it contributes
\(O(n^{-8}),O(n^{-18}),O(n^{-7})\). Each is negligible relative to the
corresponding unpadded quantity, so the padding affects neither the displayed
spectral scales nor the coupling normalization below.
Define
\begin{equation}
  \chi_n:=\left(\sum_{j=1}^P\frac{q_j}{d_j}\right)^{-1/2},
  \qquad
  a_j:=\chi_n\sqrt{\frac{q_j}{d_j}},
  \label{eq:classification-dense-coupling-weights}
\end{equation}
and set
\begin{equation}
  \begin{aligned}
    D&=\operatorname{diag}(1,d_1,\ldots,d_P),
    &A&=\operatorname{diag}(1,a_1,\ldots,a_P),\\
    \Sigma_U&=\operatorname{diag}
    \bigl(1,(1-\chi_n^2)q_1,\ldots,(1-\chi_n^2)q_P\bigr).
  \end{aligned}
  \label{eq:classification-dense-shared-map}
\end{equation}
All \(a_j\) are strictly positive, \(\sum_j a_j^2=1\), and the second
marginal nuisance variances are exactly \(q_j\).  The \(j\)-th joint
nuisance block is
\begin{equation}
  M_j=
  \begin{pmatrix}
    d_j&\chi_n\sqrt{d_jq_j}\\
    \chi_n\sqrt{d_jq_j}&q_j
  \end{pmatrix},
  \qquad
  (1-\chi_n)(d_j\oplus q_j)
  \preceq M_j\preceq
  (1+\chi_n)(d_j\oplus q_j).
  \label{eq:classification-dense-block-comparison}
\end{equation}
For the five rows in table order excluding \(\Ncls+\Ncls\to\Bcls\), the sums defining
\(\chi_n^{-2}\) are bounded below, up to constants, by
\(n^2,n,n^4,n,n^{9/2}\),
respectively.  Thus \(\chi_n\to0\), so the ordered joint nuisance spectrum
is uniformly \(1+o(1)\) comparable to the multiset union of the two target
spectra.

\medskip
\noindent\textbf{Step 2: common label and signal eigendirections.}\par\smallskip
Write \(X_{1,\mathrm{sig}}=G\) and \(U_{\mathrm{sig}}=H\), where
\(G,H\sim N(0,1)\) are independent.  Then
\(X_{2,\mathrm{sig}}=G+H\), and the joint signal block equals
\[
  M=\begin{pmatrix}1&1\\1&2\end{pmatrix}.
\]
Let \((\lambda_+,v_+)\) be its leading normalized eigenpair, with
\(v_+=(v_1,v_2)^\top\), and choose
\(\gamma=v_2\) and \(\beta=v_1+v_2\).  Use the common label
\begin{equation}
  Y=\operatorname{sign}(\beta G+\gamma H+\sigma G_0),
  \qquad G_0\sim N(0,1),\quad \sigma\geq0\ \text{fixed}.
  \label{eq:classification-shared-signal-template}
\end{equation}
The joint coefficient is \((\beta-\gamma,\gamma)^\top=v_+\).  Thus it is a
joint signal eigenvector, and both marginal coefficients lie in their scalar
signal coordinates.  After input-specific standardization, all
three reductions have the normalized form
(\ref{eq:classification-one-sparse-model}), with bounded probit noise levels.
Hence all three inputs satisfy the one-sparse assumption.  Deleting
the joint signal leaves only
\(\lambda_-=(3-\sqrt5)/2\) as one additional constant-order nuisance
direction.

\medskip
\noindent\textbf{Step 3: five substitutions under the common template.}
\par\smallskip
(\ref{eq:classification-dense-coupling-weights}) and
(\ref{eq:classification-dense-shared-map}) give
\[
  a_j^2d_j=\chi_n^2q_j,
  \qquad
  \operatorname{Var}(X_{2,j})
  =a_j^2d_j+(1-\chi_n^2)q_j=q_j,
  \qquad
  \sum_ja_j^2=1.
\]
Furthermore, conjugating each nuisance block by
\(\operatorname{diag}(d_j^{-1/2},q_j^{-1/2})\) produces
\(\bigl(\begin{smallmatrix}1&\chi_n\\\chi_n&1\end{smallmatrix}\bigr)\).
This re-verifies (\ref{eq:classification-dense-block-comparison}).  For a
flat nuisance block \(a^{[m]}\), its mass and effective rank are \(ma\) and
\(m\), respectively.

\medskip
\noindent\textbf{Case 1: \(\Bcls+\Bcls\to\Bcls\).}\par\smallskip
\noindent\emph{Calculation.}
Each marginal has \(S_0=n\), \(R_0=n^2\), and cutoff zero.  The joint
cutoff is at most one, and after deleting the possible
\(\lambda_-\) direction its tail has \(S\asymp n\) and
\(R\asymp n^2\).

\noindent\emph{Conclusion.}
\(\Delta_1,\Delta_2,\Delta_J=\Theta(n^{-1})\), so all three
inputs are benign.

\medskip
\noindent\textbf{Case 2: \(\Bcls+\Ncls\to\Bcls\).}\par\smallskip
\noindent\emph{Marginal calculation.}
The first marginal has the preceding flat block, while the second has
\((n^{-1})^{[p_n]}\).  The latter has cutoff zero,
\(S_0=p_n/n\), \(R_0=p_n\), and
\(\Delta_2\to1/\rho\), whereas \(\Delta_1=\Theta(n^{-1})\).

\noindent\emph{Joint calculation.}
The joint cutoff is at most one; its diffuse first tail gives
\(S\asymp n\), \(R\asymp n^2\), and
\(\Delta_J=\Theta(n^{-1})\).

\noindent\emph{Conclusion.}
The first source's wide tail rescues the fused classifier.

\medskip
\noindent\textbf{Case 3: \(\Bcls+\Ncls\to\Ncls\).}\par\smallskip
\noindent\emph{Marginal calculation.}
The first nuisance block
\((n^{-2})^{[n^2]}\) has \(S_0=1\), \(R_0=n^2\), and
\(\Delta_1=\Theta(n^{-1})\).  The second block \(1^{[n^2]}\) has
\(S_0=R_0=n^2\), giving \(\Delta_2\asymp n\).

\noindent\emph{Joint calculation.}
The \(n^2\) unit directions remain in the joint nuisance spectrum, so
\(\Delta_J\asymp n\).

\noindent\emph{Conclusion.}
The second marginal and the fused classifier are both non-benign.

\medskip
\noindent\textbf{Case 4: \(\Ncls+\Ncls\to\Ncls\).}\par\smallskip
\noindent\emph{Marginal calculation.}
Both marginal blocks are
\((n^{-1})^{[p_n]}\), hence \(\Delta_i\to1/\rho\).  After the extra
constant direction is removed, the joint tail has mass \(2\rho+o(1)\),
effective rank \(2\rho n(1+o(1))\), and cutoff one.

\noindent\emph{Conclusion.}
\(\Delta_J\to1/(2\rho)>0\), so non-benignity persists after fusion.

\medskip
\noindent\textbf{Case 5: \(\Bcls+\Bcls\to\Ncls\).}\par\smallskip
\noindent\emph{First marginal.}
Let \(q=\lfloor\sqrt n\rfloor\).  The first block
\((n^{-2})^{[n^4]}\) has \(S_0=n^2\), \(R_0=n^4\), and
\(\Delta_1\asymp n^{-1}\).

\noindent\emph{Second marginal.}
The second nuisance spectrum is
\((n^2)^{[q]}\uplus(n^{-2})^{[n^2]}\).  Before all \(q\) heads are
deleted, \(r_k=O(q)=o(n)\); afterward the flat tail has mass one and
effective rank \(n^2\).  Therefore its cutoff is \(q\) and
\(\Delta_2\asymp q/n+n/n^2\asymp n^{-1/2}\).

\noindent\emph{Joint calculation.}
In the joint nuisance spectrum the same \(q\) heads precede the constant
direction \(\lambda_-\) and the first source's diffuse tail.  At
\(k=q\), the remaining mass is
\(S_q\asymp n^2\), its squared mass is \(O(1)\), and hence
\(R_q\asymp n^4\).  Consequently
\[
  \Delta_J
  \asymp(1+n)^2\left(n^{-1/2}+n^{-3}\right)
  \asymp n^{3/2}.
\]

\noindent\emph{Conclusion.}
Both marginal interaction quantities vanish, while the joint quantity
diverges.  All cutoffs in Cases 1--5 are \(0\), \(1\), or \(q=o(n)\).

\medskip
\noindent\textbf{Step 4: scale-dependent rescue.}\par\smallskip
\noindent\textbf{Case 6: \(\Ncls+\Ncls\to\Bcls\).}\par\smallskip
\noindent\emph{Parameters and covariance construction.}
Fix
\[
  0<r<\tfrac12,\qquad
  0<\ell<\tfrac12-r,\qquad
  \rho>b_{\mathrm C},\qquad M>2,
\]
let \(m=n^2\), \(p=p_n\), and
\(L=n^\ell\).  Put \(d_{\mathrm{tail}}=Ln^{-1/2}\) and
\(c_n=n^{-2}\), and take
\begin{align*}
  D&=\operatorname{diag}(n^{-r},d_{\mathrm{tail}}I_m),\\
  A&=\operatorname{diag}
  (n^{1/4+r/2},c_n I_m),\\
  \Sigma_U&=\operatorname{diag}
  \bigl(n^{-M},I_p,n^{-M}I_{m-p}\bigr).
\end{align*}
This is exactly the scale-dependent builder used in the experiments.  Every
entry of \(A\) is nonzero, while its nuisance restriction satisfies
\(\|A_{-t}\|_{\mathrm F}^2=mc_n^2\sim n^{-2}\).  If
\(e_n:=c_n^2d_{\mathrm{tail}}=n^{\ell-9/2}\), then source 2 has \(p\)
nuisance eigenvalues \(1+e_n\) and \(m-p\) nuisance eigenvalues
\(n^{-M}+e_n\).  These are the two groups displayed in
Table~\ref{tab:classification-witness-scales}.

\noindent\emph{Signal checks.}
Write \(d=n^{-r}\), \(a=n^{1/4+r/2}\), and \(u=n^{-M}\).  The joint
signal block is
\[
  G=\begin{pmatrix}d&ad\\ad&a^2d+u\end{pmatrix}.
\]
Let \((\lambda_{t,J},v_+)\) be its leading normalized eigenpair, with
\(v_+=(v_1,v_2)^\top\), and choose
\[
  \gamma_1=v_2,
  \qquad
  \beta_1=v_1+av_2,
\]
with all other entries zero.  Use the common label
\(Y=\operatorname{sign}(\beta^\top X_1+\gamma^\top U+\xi)\), where
\(\xi\sim N(0,\sigma^2)\) is independent and \(\sigma<\infty\) is fixed.
Then the coefficient induced by the common
score in the joint input is
\[
  b_J=(\beta-A^\top\gamma,\gamma)^\top
  =v_+
\]
inside the signal block.  It is therefore a joint eigendirection; dividing by
\(\sqrt{\lambda_{t,J}}\) gives the normalized coefficient used in
(\ref{eq:classification-one-sparse-model}).  The marginal coefficients are
also supported on their first coordinate because all matrices are diagonal.
Moreover,
  \(\lambda_{t,1}=n^{-r}\),
  \(\lambda_{t,2}=n^{1/2}+n^{-M}\), and
  \(\lambda_{t,J}\sim n^{1/2}\).  More explicitly,
  \[
    \beta_1^2d\sim\lambda_{t,J},
    \qquad
    \frac{\lambda_{t,J}^2v_2^2}{a^2d+u}
      \sim\lambda_{t,J},
    \qquad
    b_J^\top\Sigma_Jb_J=\lambda_{t,J}.
  \]
  These are the signal variances explained by sources \(1\), \(2\), and
  the joint input, respectively, and all are of order \(n^{1/2}\).
  The corresponding residual-to-signal variance ratios are uniformly bounded
  (and in the first and joint inputs they vanish), while the
  additional variance \(\sigma^2\) is fixed.  Hence all three standardized
  reductions satisfy the one-sparse assumptions with bounded \(\tau_n\).

\noindent\emph{Marginal interaction quantities.}
After deleting the respective signal directions, source 1 has
\(m\) eigenvalues \(Ln^{-1/2}\).  Thus its cutoff is zero,
\(S_0\sim Ln^{3/2}\), \(R_0\sim n^2\), and
\(\Delta_1\asymp n^{2(r+\ell)}\).  For source 2,
\(me_n=n^{\ell-5/2}=o(1)\) and \(mn^{-M}=o(1)\).  Its nuisance mass and
squared mass are therefore both \(\rho n+o(n)\), so its cutoff is zero,
\(R_0\sim\rho n\), and \(\Delta_2\to1/\rho\).

\noindent\emph{Joint interaction quantity.}
For a nuisance innovation variance \(u_j\in\{1,n^{-M}\}\), the corresponding
joint block is
\[
  K_j=\begin{pmatrix}
    d_{\mathrm{tail}}&c_nd_{\mathrm{tail}}\\
    c_nd_{\mathrm{tail}}&c_n^2d_{\mathrm{tail}}+u_j
  \end{pmatrix}.
\]
Summing \(\operatorname{tr}(K_j)\) over all blocks, and including the
negligible smaller eigenvalue of \(G\), gives
\(S_0\sim Ln^{3/2}\).  Since the sum of the squared eigenvalues equals the
sum of \(\operatorname{tr}(K_j^2)\), direct expansion gives
\[
  \sum_{j=1}^m\operatorname{tr}(K_j^2)
  =\Theta\!\left(md_{\mathrm{tail}}^2+p\right)
  =\Theta(n^{1+2\ell}).
\]
The largest nuisance eigenvalue is \(\Theta(1)\), so
\(r_0\asymp n^{3/2+\ell}\gg n\) and the joint cutoff is zero.  Consequently,
\[
  S_0\sim Ln^{3/2},\qquad
  R_0\sim n^2,\qquad
  \Delta_J\asymp n^{2\ell-1}\longrightarrow0.
\]

\noindent\emph{Conclusion.}
Because \(r+\ell>0\), both marginal interaction quantities stay away from
zero, whereas the joint quantity vanishes.
Theorem~\ref{thm:classification-main-characterization} proves
\(\Ncls+\Ncls\to\Bcls\), completing all six rows.
\end{proof}

\section{Proofs for the Cross-Task Separation Results}
\label{app:cross-task-separation-proofs}

\subsection{Fusion rescue from two non-benign sources}
\label{app:nonbenign-fusion-rescue}

\begin{proof}[Proof of
Theorem~\ref{thm:cross-task-separation}, part 1]
Fix constants
\[
  0<r<\frac12,\qquad
  0<\ell<\frac12-r,\qquad
  \rho>\max\{b_{\mathrm C},b_{\mathrm B}\},\qquad M>2,
\]
and put \(m=\lfloor n^2\rfloor\), \(p=\lfloor\rho n\rfloor\), and
\(L=n^\ell\).  We use exactly the scale-dependent construction in
Case 6 of Appendix~\ref{app:classification-transition-witnesses}.  Namely,
\[
  D=\operatorname{diag}(n^{-r},Ln^{-1/2}I_m),
\]
the signal entry of \(A\) is \(n^{1/4+r/2}\), every nuisance entry is
\(n^{-2}\), and \(\Sigma_U\) has signal variance \(n^{-M}\), followed by
\(p\) unit nuisance variances and \(m-p\) nuisance variances \(n^{-M}\).
Thus \(\|A_{-t}\|_{\mathrm F}^2\sim n^{-2}\), and the second marginal
nuisance variances are \(1+n^{\ell-9/2}\) and
\(n^{-M}+n^{\ell-9/2}\) in the two respective groups.  For reference, the
block calculation in
Appendix~\ref{app:classification-transition-witnesses} gives explicitly
\[
  S_{J,-t}\asymp Ln^{3/2},\qquad
  \sum_j\mu_{J,j}^2=\Theta(n^{1+2\ell}),\qquad
  R_{J,-t}\asymp n^2,\qquad k^\star_{-t,J}=0.
\]

Let \((\lambda_{t,J},v_+)\) be the leading normalized eigenpair of the
two-dimensional joint signal block.  For this cross-task result, rescale the
experimental classification direction by \(\lambda_{t,J}^{-1/2}\): if
\(v_+=(v_1,v_2)^\top\), take
\(\gamma_1=v_2/\sqrt{\lambda_{t,J}}\) and
\(\beta_1=(v_1+n^{1/4+r/2}v_2)/\sqrt{\lambda_{t,J}}\), with all remaining
entries zero.  Also take a fixed \(\sigma^2>0\).  The common score is
\[
  L_n=\beta_n^\top X_1+\gamma_n^\top U+\xi,
  \qquad \xi\sim N(0,\sigma^2).
\]
The resulting signal variances explained by \(X_1\), \(X_2\), and the
joint input converge to positive constants, and their normalized
probit noise levels remain bounded.  In particular, the same score supplies
a one-sparse eigendirection in all three inputs.  This scalar
normalization changes neither the covariance construction nor any
\(\Delta\)-order used below.

\paragraph{Classification states.}
The signal-deleted calculations in Case 6 of
Appendix~\ref{app:classification-transition-witnesses} give
\[
  \Delta_{1,n}\asymp n^{2(r+\ell)},\qquad
  \Delta_{2,n}\longrightarrow\frac1\rho,\qquad
  \Delta_{J,n}\asymp n^{2\ell-1}.
\]
All three cutoffs equal zero.  Since \(r+\ell>0\) and
\(\ell<1/2\), Theorem~\ref{thm:classification-main-characterization}
therefore yields
\[
  X_1:\Ncls,\qquad
  X_2:\Ncls,\qquad
  (X_1,X_2):\Bcls.
\]

\paragraph{Regression states.}
For completeness, consider a one-sparse regression reduction with signal
eigenvalue \(\lambda_t\), standardized population amplitude \(\alpha\),
nuisance inverse Gram matrix \(B\), and independent residual vector
\(\varepsilon\).  The Sherman--Morrison formula
\citep[Section~0.7.4]{hornjohnson2013matrix} gives the exact fitted signal amplitude
\begin{equation}
  \widehat\alpha
  =
  \frac{\alpha u+\lambda_tz^\top B\varepsilon}{1+u},
  \qquad
  u:=\lambda_tz^\top Bz.
  \label{eq:cross-task-regression-survival}
\end{equation}
When the nuisance tail has mass \(S\), cutoff zero, and effective rank much
larger than \(n\), the weighted Gaussian Gram bound
\citep[Lemma~10]{bartlett2020benign} gives
\(cSI_n\preceq A_-\preceq CSI_n\) with probability tending to one.
Together with \(\|z\|_2^2/n\to_{\mathbb P}1\), this yields
\[
  u\asymp_{\mathbb P}\frac{n\lambda_t}{S},
\]
Moreover, conditionally on \(B,z\), if the residual variance is bounded,
\[
  \mathbb E\!\left[
    (\lambda_tz^\top B\varepsilon)^2\mid B,z
  \right]
  \leq C\lambda_t^2\|B\|_{\mathrm{op}}^2\|z\|_2^2
  =O_{\mathbb P}\!\left(\frac{\lambda_t^2n}{S^2}\right).
\]
Conditional Chebyshev therefore gives
\[
  \lambda_tz^\top B\varepsilon
  =O_{\mathbb P}\!\left(\frac{\lambda_t\sqrt n}{S}\right).
\]

For \(X_1\), \(\lambda_{t,1}=n^{-r}\) and
\(S_1\asymp Ln^{3/2}\).  Hence
\[
  u_1\asymp_{\mathbb P}n^{-1/2-r-\ell}\to0,
  \qquad
  \lambda_{t,1}z^\top B_1\varepsilon_1=o_{\mathbb P}(1).
\]
Because \(\alpha_1^2\to1\),
(\ref{eq:cross-task-regression-survival}) shows that the signal bias stays
bounded away from zero.  Thus \(X_1\) is regression non-benign.

For \(X_2\), deleting its leading signal eigenvalue leaves \(p\) unit
eigenvalues and negligible remainder.  Its full-spectrum Bartlett cutoff is
one and
\[
  \frac{n}{R_1(\Sigma_2)}\longrightarrow\frac1\rho.
\]
The input-specific residual variance is bounded below by the fixed
\(\sigma^2\).  Lemma~\ref{lem:fixed-probability-variance-lower} therefore
implies \(X_2:\Nreg\).

For the joint input,
\(\lambda_{t,J}\asymp n^{1/2}\), whereas its nuisance mass satisfies
\(S_J\asymp Ln^{3/2}\).  The direct nuisance-block calculation in Case 6
gives the same Gram scale, and hence
\[
  u_J\asymp_{\mathbb P}
  \frac{n\lambda_{t,J}}{S_J}
  \asymp n^{-\ell}\to0.
\]
The residual term in (\ref{eq:cross-task-regression-survival}) is
\(o_{\mathbb P}(1)\), while \(\alpha_J=1\).  The joint regression signal
therefore has asymptotically unit bias, proving
\((X_1,X_2):\Nreg\). This establishes the first table in
Theorem~\ref{thm:cross-task-separation}.
\end{proof}

\subsection{Task separation from two benign sources}
\label{app:benign-marginal-task-separation}

\begin{proof}[Proof of
Theorem~\ref{thm:cross-task-separation}, part 2]
Let \(h=\sqrt{\log n}\), \(p=n^3\), and define
\[
  \varepsilon=h^{-2},\qquad
  \eta=\frac{h}{n^2},\qquad
  \rho=n^{-10}.
\]
Take
\begin{equation}
  D=\operatorname{diag}(1,\eta I_p),\qquad
  A=I_{p+1},\qquad
  \Sigma_U=\operatorname{diag}(\varepsilon,\rho I_p).
  \label{eq:cross-task-benign-marginal-covariance}
\end{equation}
The first joint coordinate block is
\[
  M_\varepsilon=
  \begin{pmatrix}1&1\\1&1+\varepsilon\end{pmatrix}.
\]
Let \((\lambda_-,v_-)\) be its smaller normalized eigenpair.  Write
\(v_-=((v_-)_1,(v_-)_2)^\top\), and choose
\begin{equation}
  \gamma_1=\frac{(v_-)_2}{\sqrt{\lambda_-}},
  \qquad
  \beta_1=\frac{(v_-)_1+(v_-)_2}{\sqrt{\lambda_-}},
  \label{eq:cross-task-low-signal-score}
\end{equation}
with all other entries zero and with \(\xi=0\).  Then
\[
  b_J=(\beta-\gamma,\gamma)^\top
  =v_-/\sqrt{\lambda_-},
  \qquad
  b_J^\top\Sigma_Jb_J=1.
\]
Thus the common score \(L_n=\beta^\top X_1+\gamma^\top U\) is exactly the
standardized low-eigenvalue joint signal.  Since
\[
  \lambda_-=\frac{\varepsilon}{2}+O(\varepsilon^2),\qquad
  (v_-)_1+(v_-)_2=\lambda_-(v_-)_1,
\]
both marginal predictable signal variances are \(\Theta(\varepsilon)\),
whereas their residual variances converge to one.

\paragraph{Marginal regression is benign.} Both marginal covariances have
\[
  \frac{r_0(\Sigma_r)}{n}=h(1+o(1)),\qquad
  k_{b_{\mathrm B}}(\Sigma_r)=0,\qquad
  R_0(\Sigma_r)=n^2h^2(1+o(1)),
  \quad r\in\{1,2\}.
\]
Their population coefficient norms satisfy
\(\|b_r\|_2^2\|\Sigma_r\|_{\mathrm{op}}=\Theta(\varepsilon)\).
Apply Theorem~\ref{thm:bartlett-regression-bounds} with
\(\delta_n=n^{-2}\).  The bias and variance upper bounds are
\[
  O(\varepsilon h)=O(h^{-1}),
  \qquad
  O\!\left((\log n)\frac{n}{R_0(\Sigma_r)}\right)=O(n^{-1}),
\]
respectively.  Therefore
\(X_1,X_2:\Breg\).

\paragraph{Joint regression is non-benign.}
After deleting the exact signal direction \(v_-\), the joint nuisance
spectrum consists of \(\lambda_+=2+O(\varepsilon)\), \(p\) eigenvalues
\(\mu_+=2\eta(1+o(1))\), and \(p\) eigenvalues
\(\mu_-=\rho/2(1+o(1))\).  Its nuisance Gram matrix is therefore bounded
below by \(cnhI_n\) with probability tending to one: apply the rectangular
Gaussian singular-value bound
\citep[Theorem~4.6.1]{vershynin2018highdimensional} to the \(p\) columns
with variance \(\mu_+\), noting that \(p/n\to\infty\) and
\(p\mu_+\asymp nh\).  With
\(B\) denoting its inverse,
\[
  u_J:=\lambda_-z^\top Bz
  =O_{\mathbb P}\!\left(\frac{\varepsilon}{h}\right)
  =O_{\mathbb P}(h^{-3}).
\]
There is no joint residual noise, so
(\ref{eq:cross-task-regression-survival}) gives
\(\widehat\alpha_J=u_J/(1+u_J)\to_{\mathbb P}0\), while the true
standardized amplitude equals one.  Consequently
\[
  \mathcal E_{J,n}^{\mathrm{reg}}
  \geq(1-\widehat\alpha_J)^2
  \xrightarrow{\mathbb P}1,
\]
and the fused regression problem is non-benign.

\paragraph{Marginal classification is benign.}
For either marginal, normalize its scalar predictable score to write
\[
  Y^{\mathrm{cls}}=\operatorname{sign}(Z+\tau_rG),
  \qquad \tau_r=\Theta(h).
\]
Although \(\tau_r\) is not uniformly bounded, benignity follows directly.
Let \(B_r\) be the inverse nuisance Gram matrix. Standard Gaussian sample
covariance concentration \citep[Chapter~5]{vershynin2018highdimensional} for
the \(p\) flat nuisance directions gives
\[
  \operatorname{tr}(B_r)\asymp_{\mathbb P}h^{-1},
  \qquad
  \|B_r\|_{\mathrm F}
  =O_{\mathbb P}\!\left((\sqrt n\,h)^{-1}\right).
\]
Moreover
\(\kappa_{\tau_r}=\sqrt{2/\pi}/\sqrt{1+\tau_r^2}
=\Theta(h^{-1})\).  The polarization calculation in
(\ref{eq:app-su-concentration}) therefore yields
\[
  z^\top B_ry
  =\kappa_{\tau_r}\operatorname{tr}(B_r)
   +O_{\mathbb P}(\|B_r\|_{\mathrm F})
  =\Theta_{\mathbb P}(h^{-2})
   +O_{\mathbb P}((\sqrt n\,h)^{-1})>0
\]
with probability tending to one; call this event \(E_{r,n}\). On
\(E_{r,n}\), the fitted classifier has positive survival.  The Gaussian
sign identity (\ref{eq:classification-risk-gamma}) then puts its correlation
with the label in \([0,1]\), and hence its risk is at most \(1/2\).
The marginal Bayes
correlation is \(\Theta(\sqrt\varepsilon)=\Theta(h^{-1})\), so
\[
  0\leq\mathcal E_{r,n}^{\mathrm{cls}}
  \leq\frac12-\mathcal R_{r,n}^{\mathrm{cls},\star}=O(h^{-1})
  \qquad\text{on }E_{r,n}.
\]
On the complement, the excess risk is always at most one. Hence, for every
fixed \(\delta>0\),
\[
  \mathbb P\!\left\{
    \mathcal E_{r,n}^{\mathrm{cls}}>\delta
  \right\}
  \leq \mathbb P(E_{r,n}^{\mathrm c})
    +\mathbf 1\{Ch^{-1}>\delta\}\longrightarrow0.
\]
Thus \(X_1,X_2:\Bcls\). Moreover,
\(\mathcal R_{r,n}^{\mathrm{cls},\star}\to1/2\) and, because the excess risk
vanishes,
\(\mathcal R_{r,n}^{\mathrm{cls}}(\widehat w_n)\to_{\mathbb P}1/2\): this marginal benignity is
driven by an asymptotically uninformative task.

\paragraph{Joint classification is benign.} The joint classification model is exactly one-sparse with signal
\(\lambda_{t,J}=\lambda_-\sim(2h^2)^{-1}\) and zero residual noise.
After deleting this direction, the nuisance spectrum above satisfies
\[
  S_0\asymp nh,\qquad
  \sum_j\mu_j^2=\Theta(1),\qquad
  R_0\asymp n^2h^2,\qquad k_{-t,J}^\star=0.
\]
Consequently
\[
  \Delta_{J,n}
  \asymp
  \left(1+h^3\right)^2\frac1{nh^2}
  \asymp\frac{h^4}{n}
  \longrightarrow0.
\]
Theorem~\ref{thm:classification-main-characterization} proves
\((X_1,X_2):\Bcls\), completing the second table in
Theorem~\ref{thm:cross-task-separation}.
\end{proof}

\clearpage
\bibliography{main}
\bibliographystyle{plainnat}

\end{document}